\documentclass[12pts]{article}
\usepackage[utf8]{inputenc}
\usepackage[round]{natbib}
\usepackage[margin=1in]{geometry}
\usepackage{graphicx} 
\usepackage{authblk}
\title{Generalization Bounds for Flow-matching Generative Models for Intrinsically Low-dimensional Data}
\author[1]{Saptarshi Chakraborty\thanks{email: \texttt{saptarsc@umich.edu}}}
\author[4]{Quentin Berthet \thanks{email: \texttt{qerthet@google.com}}}
 \author[2,3,4]{Peter L.~Bartlett\thanks{email: \texttt{peter@berkeley.edu}}}
 \affil[1]{Department of Statistics, University of Michigan}
 \affil[2]{Department of Statistics, University of California, Berkeley}
  \affil[3]{Department of Electrical Engineering and Computer Sciences, UC Berkeley}
 \affil[4]{Google DeepMind}
\date{\vspace{-5ex}}
\date{\today}
\usepackage{titletoc}
\usepackage{amsfonts}
\usepackage{amsmath}
\usepackage{amsthm}
\usepackage{amssymb}
\usepackage{bbm}
\usepackage{enumerate}
\usepackage{thm-restate}
\usepackage{graphicx} 
\usepackage{caption} 
\usepackage{subcaption} 

\newcommand{\pd}{P_{\text{data}}}
\newcommand{\pemp}{P^{\text{emp}}_n}
\newcommand{\hvmc}{\hv^{\text{mc}}_{n,m}}
\newcommand{\da}{\dot{\alpha}}
\newcommand{\db}{\dot{\beta}}
\newtheorem{defn}{Definition}
\newtheorem{theorem}[defn]{Theorem}
\newtheorem{lemma}[defn]{Lemma}
\newtheorem{cor}[defn]{Corollary}

\newtheorem{assumption}{Assumption}
\newtheorem{design}{D}

\newtheorem{remark}{Remark}
\usepackage[mathscr]{euscript}
\usepackage[italicdiff]{physics}
\usepackage{natbib}
\usepackage{xcolor}
\usepackage{hyperref}
\hypersetup{
    colorlinks=true,
    citecolor = blue,
    linkcolor=blue,
    filecolor=magenta,      
    urlcolor=teal,
    pdftitle={Overleaf Example},
    pdfpagemode=FullScreen,
    }

\def\Real{\mathop{\mathbb{R}}\nolimits}

\def\argmin{\mathop{\rm argmin}}

\newcommand{\diff}{\mathrm{d}}
\newcommand{\prob}{\mathbb{P}}
\newcommand{\one}{\mathbbm{1}}

\newcommand{\bP}{\boldsymbol{P}}
\newcommand{\bQ}{\boldsymbol{Q}}

\newcommand{\cB}{ \mathcal{B}}
\newcommand{\cC}{ \mathcal{C}}

\newcommand{\cF}{ \mathcal{F}}

\newcommand{\cK}{ \mathcal{K}}
\newcommand{\cL}{ \mathcal{L}}

\newcommand{\cN}{ \mathcal{N}}
\newcommand{\cO}{ \mathcal{O}}

\newcommand{\cR}{ \mathcal{R}}
\newcommand{\cS}{ \mathcal{S}}

\newcommand{\cV}{ \mathcal{V}}
\newcommand{\cW}{ \mathcal{W}}

\newcommand{\sA}{ \mathscr{A}}

\newcommand{\sN}{ \mathscr{N}}

\newcommand{\sT}{ \mathscr{T}}

\newcommand{\hP}{ \hat{P}}

\newcommand{\hX}{ \widehat{X}}

\newcommand{\E}{\mathbb{E}}

\newcommand{\relu}{\text{ReLU}}

\newcommand{\fH}{\mathbb{H}}

\newcommand{\fL}{\mathbb{L}}

\newcommand{\fN}{\mathbb{N}}

\newcommand{\fW}{\mathbb{W}}

\newcommand{\hv}{\hat{v}}

\date{\vspace{-5ex}}

\begin{document}
\maketitle

\begin{abstract}
 Despite its remarkable empirical performance, the theoretical properties of flow-matching in terms of generalization accuracy remain largely unexplored. The current state-of-the-art generalization analyses not only impose stringent assumptions on the estimated velocity field, but also scale poorly with the ambient dimension---precisely the high-dimensional setting where flow matching is most commonly applied. In particular, the data distributions targeted by flow matching models, such as natural images, molecular geometries, etc., are commonly hypothesized to concentrate near a low-dimensional structures embedded in the high-dimensional ambient data space. However, existing analyses fail to reflect this geometry in their derived convergence rates. In this paper, we bridge the gap between the theory and practice of flow matching by establishing generalization bounds on the Wasserstein-$p$ distance between the learned distribution and the target measure, for any $p \ge 1$. Specifically, given $n$ independent and identically distributed samples from a target distribution $P_{\mathrm{data}}$ and for every $d > d^\ast_p(\pd)$, we show that for appropriate network and hyperparameter choices, \(\mathbb{W}_p\!\left(\widehat{P}^{\mathrm{FM}},\, P_{\mathrm{data}}\right)
    \lesssim
    n^{-1/d}
    + n^{-\frac{1}{2p}}\!\left(\log(1/\xi)\right)^{\frac{1}{2p}}\)
with probability at least $1 - \xi$, where $d_p^\ast(P_{\mathrm{data}})$ denotes the Wasserstein-$p$ dimension of the target measure. Our results demonstrate that flow matching naturally adapts to the intrinsic geometry of the data, effectively mitigating the curse of dimensionality: the convergence exponent depends only on the intrinsic dimension, $d_p^\ast(P_{\mathrm{data}})$, rather than the ambient dimension of the feature space, yielding convergence rates that remain meaningful in high-dimensional regimes and offering a principled theoretical explanation for the strong empirical performance of flow matching on structured data distributions with significantly relaxed assumptions, compared to the existing literature. 
\end{abstract}

\section{Introduction}
Learning high-dimensional probability distributions from a handful of samples is a foundational problem in machine learning, with applications spanning image synthesis \citep{9878449,saharia2022photorealistic}, scientific modeling \citep{watson2023novo,geffner2025proteina,corso2023diffdock}, and generative artificial intelligence. 
The dominant paradigm for this task has shifted decisively from learning a single generator/decoder \citep{kingma2014autoencoding,goodfellow2020generative} toward continuous-time generative models such as diffusion models and score-based methods that learn to reverse a fixed noising process \citep{sohl2015deep,NEURIPS2020_4c5bcfec,song2019generative,song2021scorebased}, while optimal transport-based approaches construct more direct paths between distributions \citep{korotin2023neural,bunne2022supervised}. Flow matching \citep{lipman2023flow,albergo2023building,liu2023flow,JMLR:v26:23-1605,NEURIPS2018_69386f6b} has emerged as a particularly compelling framework within this landscape, offering simulation-free training, straight probability paths, and strong empirical performance.  The deterministic formulation of flow matching endows the resulting models with several attractive computational advantages, including stable optimization, flexible control over discretization during sampling, and seamless integration with modern Ordinary Differential Equation (ODE) solvers. In practice, flow-matching methods have demonstrated state-of-the-art performance across a range of challenging high-dimensional applications, including text-to-image generation, video synthesis, and molecular design \citep{ma2024sit,bose2024sestochastic,esser2024scaling}.

However, a rigorous understanding of when and why flow-matching generative models succeed remains incomplete, as prior theory has focused on sampling guarantees. Early work on score-based models established such guarantees under stochastic sampling \citep{block2020generative,debortoli2022convergence,lee2022convergence}, with later results showing that an $\epsilon$-approximate oracle score yields KL error $\cO(\operatorname{poly}(D, \log(1/\delta))\epsilon^{-2})$ under mild regularity \citep{pmlr-v202-chen23q,pmlr-v201-lee23a,albergo2023building,JMLR:v26:23-1605,pmlr-v162-lu22f,pmlr-v202-chen23o}, where $D$ is the ambient dimension and $\delta$ the early stopping time. Stochastic interpolants bridge diffusion and flow matching \citep{JMLR:v26:23-1605,albergo2023building}: \citet{albergo2023building} bound the $2$-Wasserstein distance between observed and estimated flow ODE endpoints via the $\fL_2$-error of the estimated velocity field, assuming Lipschitzness in time, while \citet{JMLR:v26:23-1605} analyze the KL error of flow matching under oracle approximations. For the deterministic sampler, \citet{pmlr-v202-chen23e} derive polynomial discretization bounds that do not extend to learned scores, and \citet{chen2023the} obtain polynomial guarantees with learned scores via interleaved predictor-corrector updates. Most recently, \citet{chen2026highaccuracy} established the first $\operatorname{polylog}(1/\epsilon)$-sampling guarantee, with $\widetilde{\cO}( \overline{\text{dim}}_M(\pd) \operatorname{polylog}(1/\epsilon))$ steps, where $\overline{\text{dim}}_M(\pd)$ is the Minkowski dimension (Section~\ref{sec_intrin}) of the target $\pd$. All of these works, however, assume oracle access to an $\epsilon$-accurate velocity or score estimate, without establishing its learnability.

Despite this progress, existing guarantees for flow matching suffer from the curse of dimensionality: their rates depend exponentially on the ambient dimension $D$, and thus, become uninformative precisely in the high-dimensional regimes where flow matching is most used. For example, \citet{gao2024convergence} derived $\fW_2$ convergence guarantees for flow matching under the assumption that the target distribution is $\log$-concave and admits a bounded density, obtaining $\E \fW_2(\hP^{\text{FM}}, \pd) = \widetilde{\cO}\left(n^{-1/(D+5)}\right)$, when the estimated vector velocity field is uniformly Lipschitz. \citet{zhou2025an} removed this Lipschitzness assumption on the learned velocity field at the cost of a slower rate, establishing $ \E \fW_2(\hP^{\text{FM}}, \pd) = \widetilde{\cO}\left(n^{-\frac{1}{3(D+5)}}\right)$.  Under additional Besov-smoothness assumptions on the target density, \citet{fukumizu2025flow} showed that for any $\delta > 0$, $\E \fW_2(\hP^{\text{FM}}, \pd) = \cO\left(n^{-\frac{s + (2\kappa)^{-1} - \delta}{2s + D}}\right)$, where $s$ is the smoothness parameter of the density and $\kappa \ge 1/2$ is a tunable model parameter. This sits in sharp contrast with the widely held belief that high-dimensional data concentrate near low-dimensional structures \citep{pope2020intrinsic}, suggesting the $D$-dependence is an artifact of analysis rather than intrinsic to the problem---as in classical nonparametric statistics, where kernel estimators and Gaussian process regression achieve minimax rates depending only on intrinsic dimensionality \citep{bickel2007local, kim2019uniform}. Dimension-adaptive, minimax-optimal guarantees have since been established for GANs \citep{dahal2022deep,JMLR:v23:21-0732,JMLR:v26:24-0054}, Wasserstein Autoencoders \citep{chakraborty2024a}, and diffusion models \citep{pmlr-v202-chen23e,oko2023diffusion,pmlr-v238-tang24a,chakraborty2026diffusion}, but whether analogous results hold for flow matching remains open.

While recent studies have begun exploring flow matching on manifolds \citep{benton2024error,pi2026learning} and diffusion models for multi-modal low-dimensional data \citep{wu2026diffusion}, existing guarantees either rely on restrictive manifold reach and density lower bounds or analyze in-expectation $\fW_1$/$\fW_2$ metrics. \citet{kumar2026flowmatchingadaptivemanifold} partially address this, establishing $\fW_2$ rates under oracle access to the ODE sampler for data on a smooth, lower-bounded-reach manifold with H\"older-smooth density; however, their analysis relies on strong, unverifiable geometric conditions and excludes even simple distributions (Gaussian, uniform, degenerate). In contrast, our framework provides the sharpest known non-asymptotic, high-probability concentration guarantees for arbitrary $p \ge 1$ without requiring manifold reach or global Lipschitz regularity on the learned vector field.

\paragraph{Summary of Contributions}
The main contributions of this paper are as follows.
\begingroup
\allowdisplaybreaks
\begin{itemize}
\item \textbf{High-probability Wasserstein guarantees for every $p \ge 1$.}
We develop a general theoretical framework for analyzing the statistical performance of flow-matching generative models. For arbitrary $p \ge 1$, we derive non-asymptotic, high-probability bounds on the Wasserstein-$p$ error
\(\mathbb{W}_p\!\left(\widehat{P}^{\mathrm{FM}},P_{\mathrm{data}}\right).
\)
Specifically, when one has access to i.i.d.\ samples $x_1,\ldots,x_n \sim P_{\mathrm{data}}$, we show that, for any $d > d^\ast_p(\pd)$, with probability at least $1-\xi$,
\(\mathbb{W}_p\!\left(\widehat{P}^{\mathrm{FM}},P_{\mathrm{data}}\right) \lesssim n^{-1/d_p^\ast(P_{\mathrm{data}})} + n^{-1/(2p)} \bigl(\log(1/\xi)\bigr)^{1/(2p)},\)
under appropriate choices of the velocity-field network class, discretization scheme, and early-stopping rule. In contrast to existing analyses \citep{benton2024error,gao2024convergence,zhou2025an,kumar2026flowmatchingadaptivemanifold}, which are largely expectation-based and often restricted to $p\in \{1,2\}$, our results hold for every $ p \ge 1$ and provide guarantees at arbitrarily high confidence. We also give sufficient conditions for selecting the relevant hyperparameters as functions of the sample size and the population distribution.
\item \textbf{Minimal distributional and model assumptions.}
Our guarantees require no density, smoothness, or Lipschitz assumptions on either the data-generating distribution or the velocity-field class. Despite this generality, the resulting rates adapt to the intrinsic dimension of the target distribution, providing a rigorous statistical explanation for the effectiveness of flow matching on structured high-dimensional data, including images, videos, and molecular geometries that may concentrate near low-dimensional subsets of the ambient space.
\item \textbf{Minimax rates.} For targets supported on $\tilde d$-dimensional regular sets, such as
compact $\tilde d$-dimensional manifolds, compact convex sets of dimension $\tilde d$, and
self-similar sets of similarity dimension $\tilde d$, our bound gives $\cO(n^{-1/d})$ for every
$d > \tilde d$ and every $1 \le p < \tilde d/2$, matching the minimax rate for distribution
estimation on such sets up to an arbitrarily small slack in the
exponent.
\end{itemize}
\endgroup
\section{Background and Problem Setting}
\label{background}
\subsection{Notations and Definitions}
$B_\varrho(x,r)$ denotes the open ball of radius $r$ around $x$ w.r.t.\ the metric $\varrho$. For a probability measure $\gamma$, $\text{supp}(\gamma) = \{x: \gamma(B_\varrho(x,r))>0, \text{ for all }r >0\}$. $\one(\cdot)$ denotes the indicator function. We write $A_n \lesssim B_n$ (equivalently $A_n = \cO(B_n)$) if there exists $C>0$, independent of $n$, such that $A_n \le C B_n$; $A_n \precsim B_n$ (equivalently $A_n = \widetilde{\cO}(B_n)$) if $A_n \le B_n \cdot \log^C(en)$ for some $C>0$, i.e.\ up to poly-log factors in $n$; and $A_n \asymp B_n$ if both $A_n \lesssim B_n$ and $B_n \lesssim A_n$ hold. For $k \in \mathbb{N}$, $[k] = \{1, \dots, k\}$. For a function of $t$, $\xi_t$, we write $\dot{\xi}_t = \diff \xi_t/\diff t$. For a distribution $\mu$, $\sT_M(\mu) = \operatorname{Law}(Z \one\{\|Z\| \le M\})$ where $Z \sim \mu$. $\operatorname{div}_x(\cdot)$ denotes the divergence operator w.r.t.\ $x$. We also recall the following standard definitions.

\begin{defn}[Covering Number]
    \normalfont
    For a metric space $(\cS,\varrho)$, the $\epsilon$-covering number with respect to (w.r.t.)  $\varrho$ is
    \(\cN(\epsilon; \cS, \varrho) = \inf\{n \in \mathbb{N}: \exists \, x_1, \dots, x_n \text{ such that } \cup_{i=1}^nB_\varrho(x_i, \epsilon) \supseteq \cS\}.\)
\end{defn}
\begin{defn}[Neural networks]\normalfont
 Let $L \in \mathbb{N}$ and $\{N_i\}_{i \in [L]} \subset \mathbb{N}$. An $L$-layer neural network $f: \Real^D \to \Real^{N_L}$ is 
\begin{equation}
\label{ee1}
f(x) = A_L \circ \sigma_{L-1} \circ A_{L-1} \circ \dots \circ \sigma_1 \circ A_1 (x),
\end{equation}
where $A_i(y) = W_i y + b_i$, with $W_i \in \Real^{N_{i} \times N_{i-1}}$, $b_i \in \Real^{N_i}$, $N_0 = D$, and each $\sigma_j$ applied component-wise. 
Here $\{W_i\}$ are weights, $\{b_i\}$ biases, and $\{\sigma_i\}_{1 \le i \le L-1}$ activation functions. Define: depth $\cL(f) := L$; number of weights $\cW(f) = \sum_{i=1}^L N_i N_{i-1}$; maximum weight $\cB(f) = \max_{1 \le j \le L} (\|b_j\|_\infty \vee \|W_j\|_{\infty})$; and maximum output $\cK(f) = \sup_{x \in \Real^D}\|f(x)\|_\infty$. Let
\begin{align*}
    \cN \cN_{\{\sigma_i\}_{1 \le i \le L-1}} (L, W, B, K) = \big\{  f \text{ of the form \eqref{ee1}}: \cL(f) \le L ,  \cW(f) \le W, \cB(f) \le B, \cK(f) \le K \big\}.
\end{align*}
 When $\sigma_j(x) = x \vee 0$ for all $j \in [L-1]$, we write $\cR \cN (L, W, B, K)$ for $\cN \cN_{\{\sigma_i\}_{1 \le i \le L-1}} (L, W, B, K)$.
 \end{defn}
  \begin{defn}[Wasserstein $p$-distance] \normalfont
  For probability measures $\mu, \nu$ on $\Real^D$, the Wasserstein-$p$ distance (under $\ell_2$) is
      \(\fW_p(\mu, \nu) = \left(\inf_{\gamma \in \Gamma(\mu, \nu)} \E_{(X,Y) \sim \gamma} \|X-Y\|_2^{p}\right)^{1/p},\)  where $\Gamma(\mu, \nu)$ is the set of couplings of $\mu$ and $\nu$.
\end{defn}
\subsection{Flow-matching}
The fundamental problem in generative modeling is to learn an unknown distribution $\pd$ on a data space $\mathcal{X} \subseteq \mathbb{R}^D$ from $n$ i.i.d.\ observations $x_1, \dots, x_n \sim \pd$, where $D$ can be extremely large---e.g., $D = 49{,}152$ for $128\times128$ color ImageNet images. Flow matching \citep{lipman2023flow} constructs a continuous probability path $\{P_t\}_{t \in [0,1]}$ from an easy-to-sample reference $P_0$ (typically standard Gaussian) to the target $P_1 \equiv \pd$, governed by a time-dependent vector field $v(\cdot, \cdot) : \Real^D \times [0,1] \to \Real^D$ whose induced flow transports $P_0$ to $\pd$: starting from $X_0 \sim P_0$,
\begin{equation}\label{flow_eq}
    \dot{X}_t = v(X_t, t), \quad X_0 \sim P_0.
\end{equation}
Since many velocity fields transport $P_0$ to $P_1$, a canonical choice arises from the linear interpolant $I(x_0, x_1, t) = \alpha_t x_1 + \beta_t x_0$, with  $\alpha_t, \beta_t \in [0,1]$ satisfying $\alpha_0 = \beta_1 = 0$, $\alpha_1 = \beta_0 = 1$; this gives $I(x_0,x_1,0)=x_0$ and $I(x_0,x_1,1)=x_1$. The process $\{X_t\}_{0 \le t \le 1}$ constructed as $X_t = I(X_0, X_1, t)$ for independent $X_0 \sim P_0$, $X_1 \sim P_1 \equiv \pd$, with marginal law $P_t = \operatorname{Law}(X_t)$. Since $\pd$ is unknown, it is replaced in practice by $\pemp = \frac{1}{n}\sum_{i=1}^n \delta_{x_i}$, so $X_1 \sim \pemp$. Among velocity fields consistent with $\{P_t\}$, flow matching targets the conditional expectation
   \( v_\star(x, t) = \E\big(\dot{X}_t \mid X_t = x\big),\)
satisfying the continuity equation $\partial_t \varrho_t + \operatorname{div}_x (v_\star \varrho_t) = 0$, $t \in [0,1)$, where $\varrho_t$ is the density of $P_t$. Since $v_\star(\cdot, t)$ may explode as $t \to 1$, the flow is instead run until $t = 1-\delta$. Theorem~\ref{existence} shows the resulting ODE
\begin{equation}
    \label{flow_ode_2}
    \dot{X}_t = v_\star(X_t, t), \quad X_0 \sim P_0, \quad t\in [0,1-\delta]
\end{equation}
is well-defined: for any initialization the solution exists and is unique, and $X_t$ has density $\varrho_t$.
\begin{restatable}{theorem}{thmexist}
    \label{existence}
   For all $t \in [0, 1-\delta]$, $P_t$ admits a density, $\varrho_t$, w.r.t.\ the Lebesgue measure and satisfies the continuity equation (in the weak sense), $\partial_t \varrho_t + \operatorname{div}_x (\varrho_t v_\star(\cdot, t)) = 0,$
    where $v_\star(x,t) = \beta_t^{-1}\db_t x + \beta_t^{-1}(\da_t \beta_t - \db_t \alpha_t) \left(\sum_{i=1}^n \exp\left(-\frac{\|x-\alpha_tx_i\|^2}{2\beta_t^2}\right)\right)^{-1}\sum_{i=1}^n x_i  \exp\left(-\frac{\|x-\alpha_t x_i\|^2}{2\beta_t^2}\right)$. Let 
    \begin{equation}\label{ode_2}
       y_x^\prime(t) = v_\star(y_x(t),t), \text{ with } y_x(0) = x,
    \end{equation}
be the Initial Value Problem (IVP) with initial point $x$. Then \eqref{ode_2} has a unique solution for every $x \in \Real^D$, and $P_t = (Y_t)_\sharp P_0$, where $Y_t: x \mapsto y_x(t)$.
\end{restatable}

Flow matching trains a neural network $v : \mathbb{R}^D \times [0,1] \to \mathbb{R}^D$ to approximate $v_\star$ by minimizing 
   \( \mathcal{L}(v) = \int_0^1 \mathbb{E}\|v(X_t, t) - v_\star(X_t,t)\|^2 \diff t.\)
A key observation of \citet{lipman2023flow} is that minimizing $\mathcal{L}(v)$ is equivalent (Lemma~\ref{cftofm}) to minimizing the conditional flow matching objective
\begin{equation}
    \mathcal{L}_{\mathrm{CFM}}(v) = \int_0^1 \mathbb{E}\|v(X_t, t) - \dot{X_t}\|^2 \diff t = \int_0^1 \mathbb{E}\left\|v(X_t, t) - (\da_t X_1 + \db_t X_0)\right\|^2 \diff t,
\end{equation}
which can be estimated from data without knowledge of $v_\star$. Substituting $\pemp$ for $\pd$ gives the empirical objective, $\mathcal{L}_\text{CNF}(v) = \frac{1}{n}\sum_{i=1}^n \int_0^1 \mathbb{E}\left[\|v(X_{i,t}, t) - (\da_t x_i + \db_t X_0)\|^2\right] \diff t,$
where $X_{i,t} = \alpha_t x_i + \beta_t X_0$, $X_0 \sim P_0$. Since this integral is not computed exactly in practice, we study its numerical approximation 
\begin{equation}\label{obj_3}
    \hat{\cL}_n(v) = \frac{1}{n}\sum_{i=1}^n   \sum_{k = 0}^{N-1} (t_{k+1} - t_k)\mathbb{E}\|v(X_{i, t_k}, t_k) - (\da_{t_k} x_i + \db_{t_k} X_0)\|^2,
\end{equation}
where $0 = t_0 < t_1 < \dots < t_N = 1-\delta$ partitions $[0,1-\delta]$; we use a uniform step size, $t_{k+1} = t_k + h$ for $k \in [N-1]$. The expectation in \eqref{obj_3} is typically further approximated via Monte Carlo samples
   \( X_0^{(j)} \sim \cN(0,I_D), \, X_1^{(j)} \sim \operatorname{Unif}\left(\{X_1, \dots X_n\}\right), \, X^{(j)}_t = \alpha_t X_1^{(j)} + \beta_t X_0^{(j)},\)
giving the Monte Carlo estimate 
\begin{equation}\label{obj_4}
    \hat{\cL}^{\text{mc}}_{n,m}(v) = \frac{1}{m}\sum_{j=1}^m   \sum_{k = 0}^{N-1} (t_{k+1} - t_k) \left\|v(X_{t_k}^{(j)}, t_k) - (\da_{t_k}X_1^{(j)} + \db_{t_k} X_0^{(j)})\right\|^2.
\end{equation}
 The resulting velocity field estimates are 
\begin{equation}
    \label{estimates}
    \hv_n = \argmin_{v \in \cV} \hat{\cL}_n(v) \quad \text{and} \quad \hv_{n,m}^{\text{mc}} = \argmin_{v \in \cV} \hat{\cL}_{n,m}^{\text{mc}}(v).
\end{equation}
\paragraph{Stochastic Sampling}
After the velocity field is learned, one can generate new samples by following the ODE,
\(\diff X_t = \hv(X_t,t) \diff t, \; X_0 \sim \cN(0,I_D).\)
However, analyzing the above ODE is rather difficult, as stability of the flow relies critically on the regularity of the learned velocity field. In particular, standard perturbation arguments require $\hv(\cdot,t)$ to be uniformly globally Lipschitz, and the error between trajectories generated by the learned and optimal velocity fields typically grows according to Gr\"onwall's inequality, leading to an exponential dependence on the Lipschitz constant. Since this Lipschitz constant may potentially blow up as $t \to 1$, these bounds become increasingly loose near the terminal time, making it difficult to obtain meaningful control of the Wasserstein distance between the generated and target distributions. The existing literature artificially imposes a boundedness on the Lipschitz constant of the learned velocity field \citep{kumar2026flowmatchingadaptivemanifold,gao2024convergence}. Alternatively, one can perform the generation through an SDE. This is because the Fokker-Plank equation corresponding to \eqref{flow_ode_2},
\begin{align*}
    \partial_t \varrho_t = & -\operatorname{div}_x (\varrho_t v_\star) = - \operatorname{div}_x \left(\varrho_t \left( v_\star + \frac{1}{2} \sigma^2 \nabla \log \varrho_t \right) \right) + \frac{1}{2}\sigma^2 \Delta \varrho_t,\vspace{-5pt}
\end{align*}
which is satisfied by the SDE,
\(\diff X_t = \left(v_\star(X_t, t) + \frac{\sigma}{2} \nabla \log \varrho_t(X_t) \right) \diff t + \sigma \diff W_t.\)
A formal statement that $\varrho_t$ satisfies the above SDE appears in Appendix~\ref{sec:ode-sde-equivalence}. Here, $\sigma>0$. Since, \(v_\star(x,t)   =  \frac{\da_t}{ \alpha_t} x + \beta_t^2 \left(\frac{\da_t}{\alpha_t} - \frac{\db_t}{\beta_t}\right) \nabla \log \varrho_t(x),\)
One can generate samples by following the SDE,
\begin{equation}
    \diff X_t =  \left(a_{t} v_\star(X_{t}, {t}) - b_t X_t \right) \diff t+ \sigma \ \diff W_t, \quad 0 \le t \le 1-\delta, \vspace{-2pt} \label{eq:sde}
\end{equation}
where, let $a_t = 1+ \left(2\beta_t(\dot\alpha_t\beta_t-\dot\beta_t\alpha_t)\right)^{-1}\sigma^2\alpha_t$
and $b_t = \left(2\beta_t(\dot\alpha_t\beta_t-\dot\beta_t\alpha_t)\right)^{-1}\sigma^2\dot\alpha_t$.
After the velocity field is learned, one generates new samples through the Euler--Maruyama discretization of the SDE as, 
\begin{align}
    \diff \hX_t = \left(a_{t_k} \hv(\hX_{t_k}, {t_k}) - b_t \hX_t \right) \diff t + \sigma \ \diff W_t, \quad t_k \le t < t_{k+1}, 
    \label{exponential_integrator}
\end{align}
where $\hv$ is either $\hv_n$ or $\hv_{n,m}^{\text{mc}}$. The distribution of $\hX_{1-\delta}$ is denoted as $\hP_{1-\delta}(\hv) := \operatorname{Law}\left( \hX_{1-\delta}\right)$. Letting $B(s,t) = \int_s^t b_r \diff r$, the solution to \eqref{exponential_integrator} satisfies
\begin{align*}
    \hX_{t_{k+1}} = e^{-B(t_k, t_{k+1})} \hX_{t_k} + a_{t_k} \hv(\hX_{t_k}, {t_k})\int_{t_k}^{t_{k+1}} e^{-B(s,t_{k+1})} \diff s  + \sigma \ \xi_k \sqrt{\int_{t_k}^{t_{k+1}} e^{-2B(s,t_{k+1})} \diff s},
\end{align*}
where $\xi_k \sim \cN(0,I_D)$. Such a data generation procedure has been used by practitioners such as \citet{singh2024stochastic,ma2024sit}, who report empirical superiority compared to deterministic sampling in flow-matching. The goal of this paper is to analyze the distance of $\hP_{1-\delta}(\hv) = \operatorname{Law}\left(\hX_{1-\delta}\right)$, for both the estimators defined in \eqref{estimates}.
\subsection{Intrinsic Data Dimension}
\label{sec_intrin}
Several notions of the \emph{effective dimension} of a probability measure formalize the intuition of intrinsic dimension, most characterizing the growth rate of covering numbers on the logarithmic scale. Let $(\cS, \varrho)$ be a Polish space endowed with the $\ell_\infty$ metric and let $\mu$ be a probability measure on $\cS$. Recall the $(\epsilon, \tau)$-covering number of $\mu$ \citep{posner1967epsilon}, \(\sN_\epsilon(\mu, \tau) = \inf\{ \cN(\epsilon; S, \varrho) : S \subseteq \cS,\, \mu(S) \geq 1 - \tau \},\) the minimal number of $\epsilon$-balls needed to cover a set carrying at least $1 - \tau$ of the mass of $\mu$.

A natural measure of dimensionality is the \emph{upper Minkowski dimension} \[\overline{\mathrm{dim}}_M(\mu) = \limsup_{\epsilon \downarrow 0} \left(\log(1/\epsilon)\right)^{-1}\log \cN(\epsilon;\, \mathrm{supp}(\mu),\, \ell_\infty),\] which depends only on the metric entropy of the support and so requires no smoothness or manifold structure, accommodating irregular geometries such as fractals. Its statistical implications are well established: covering numbers of smooth function classes are governed by the Minkowski dimension of their domain \citep{kolmogorov1961}, and deep networks, WGANs and Wasserstein Autoencoders all attain rates reflecting it rather than the ambient dimension \citep{JMLR:v21:20-002,JMLR:v23:21-0732,chakraborty2024a}. Being a property of the support alone, however, it ignores how mass is distributed and can be as large as $D$ even when $\mu$ concentrates near a low-dimensional region \citep{JMLR:v26:24-0054}. \citet{weed2019sharp} therefore refined Dudley's \citep{dudley1969speed} entropic dimension into a notion sensitive to the mass distribution of $\mu$, and showed that it governs the convergence rate of the empirical Wasserstein distance.

\begin{defn}[Wasserstein Dimension~\citep{weed2019sharp}]\label{def1} \normalfont
For $p > 0$, the $p$-upper Wasserstein dimension of $\mu$ is
\(d^\ast_p(\mu) = \inf \left\{ s > 2p : \limsup_{\epsilon \downarrow 0} \left(\log(1/\epsilon)\right)^{-1}\log \sN_{\epsilon}\!\left(\mu,\, \epsilon^{\frac{sp}{s-2p}}\right) \leq s \right\},\) and the lower Wasserstein dimension is \(d_\ast(\mu) = \lim_{\tau \downarrow 0}\, \liminf_{\epsilon \downarrow 0} \frac{\log \sN_\epsilon(\mu, \tau)}{\log(1/\epsilon)}.\)
\end{defn}

\citet{weed2019sharp} established that these jointly control the convergence rate of the empirical Wasserstein distance, \(n^{-1/d_\ast(\mu)} \lesssim \E\, \fW_p(\hat{\mu}_n, \mu) \lesssim n^{-1/d^\ast_p(\mu)}\). The upper Wasserstein dimension has recently been used to bound error rates for GANs \citep{JMLR:v26:24-0054} and score-matching diffusion models \citep{chakraborty2026diffusion}. 
 \section{A Proof of Concept Experiment}
\begin{figure}[t]
    \centering
    \begin{subfigure}[t]{0.3\textwidth}
        \includegraphics[width=\textwidth]{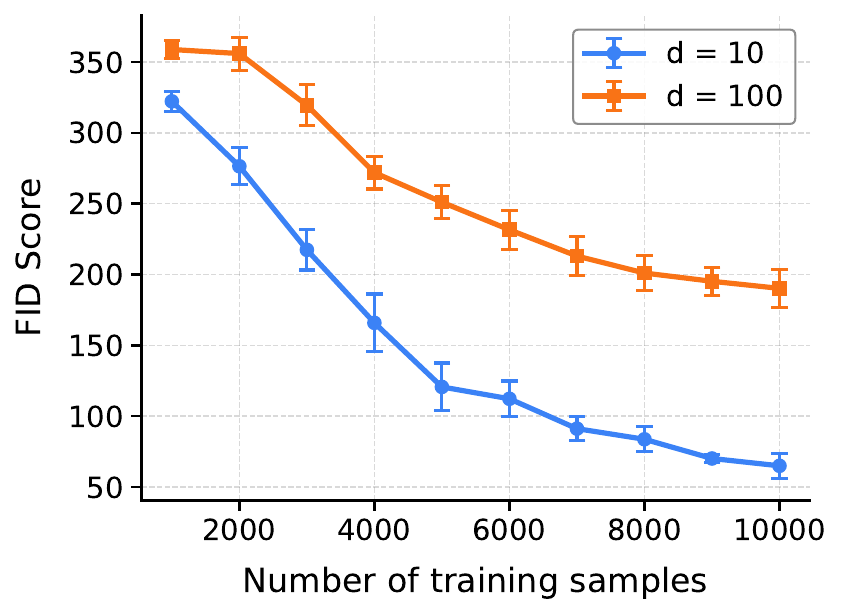}
        \caption{Results on the basenji class}
        \label{fig:sub1}
    \end{subfigure}
    \hfill 
    \begin{subfigure}[t]{0.3\textwidth}
        \includegraphics[width=\textwidth]{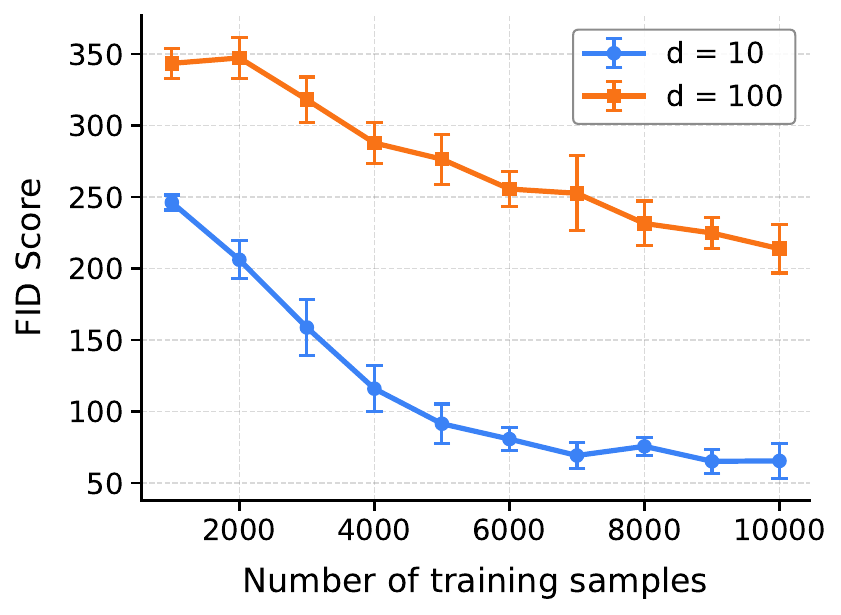}
        \caption{Results on the bubble class}
        \label{fig:sub2}
    \end{subfigure}
    \hfill 
    \begin{subfigure}[t]{0.3\textwidth}
        \includegraphics[width=\textwidth]{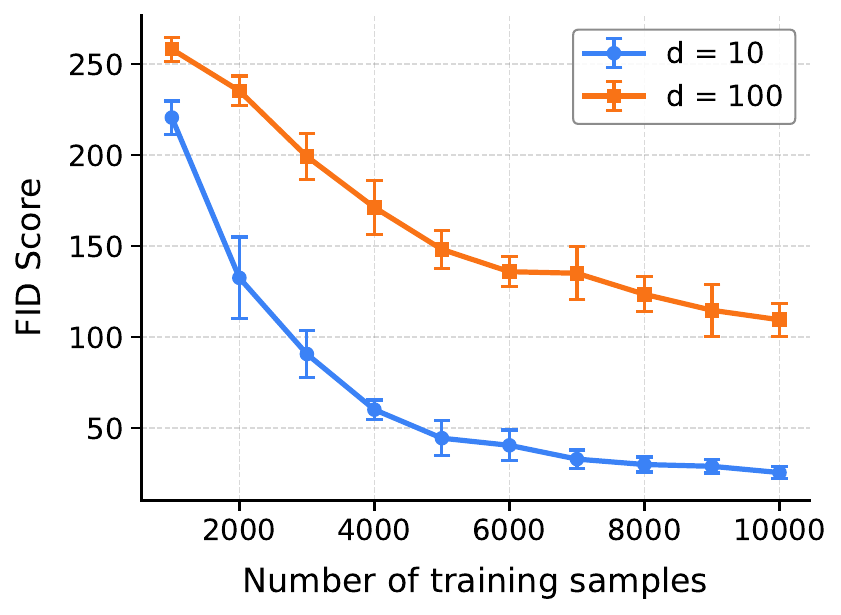}
        \caption{Results on the goldfish class}
        \label{fig:sub3}
    \end{subfigure}
    \caption{Average generalization error (in terms of FID scores) for different values of $n$ for flow-matching. The error bars denote the standard deviation out of $10$ replications. 
    }
    \label{fig:main}
\end{figure}
Before presenting our theoretical analysis, we give an experiment illustrating that Flow Matching (FM) error rates depend primarily on the \emph{intrinsic} dimension of the data, not the ambient dimension. We establish finite-sample generalization bounds using the Fokker-Planck equivalent SDE, bypassing artificial Lipschitz assumptions on the velocity field. Empirically, probability-flow ODE simulations confirm that our predicted intrinsic-dimensional error decay holds under standard deterministic sampling. Since intrinsic dimensionality is difficult to control in real image data, we follow the synthetic construction of \citet{pope2020intrinsic} and \citet{chakraborty2024a}: we use a pre-trained BigGAN \citep{brock2018large} (128-dim latent space, $128 \times 128 \times 3$ output, trained on ImageNet \citep{deng_imagenet}) to generate $11{,}000$ images each for the \texttt{basenji}, \texttt{bubble}, and \texttt{goldfish} classes, fixing all but $d \in \{10, 100\}$ latent coordinates to zero so that generated data lie on a $d$-dimensional manifold. Images are downsampled to $28 \times 28$ via bilinear interpolation. For each $d$, we train a standard UNet FM model with the usual objective $\mathbb{E}\| v_t - v_\theta(x_t, t \cdot 1000)\|_2^2$, varying the training set size over $\{1000, 2000, \ldots, 10000\}$ (Adam, lr $10^{-4}$, batch size $64$, $50$ epochs; NVIDIA RTX 6000 Blackwell GPUs). At evaluation, we simulate the learned probability-flow ODE ($100$ steps) to draw $1{,}000$ samples and compute FID \citep{heusel2017gans} against a held-out set of equal size (indices $10{,}000$--$11{,}000$), averaged over $10$ repetitions.  
 The empirical results, presented in Fig.~\ref{fig:main}, clearly validate our hypothesis. As the number of training samples increases from $1000$ to $10000$, the FM model trained on data with intrinsic dimension $d = 10$ consistently achieves substantially lower FID scores than the model trained on data with intrinsic dimension $d = 100$. The gap persists across large sample sizes, and in fact widens slightly as more data become available, indicating that lower intrinsic-dimensional structure yields faster error decay under flow-based generative learning. This behavior aligns precisely with our theoretical predictions: the intrinsic dimension, rather than the ambient pixel dimension, governs the sample complexity and achievable error rates of flow-matching generative models. All code pertaining to the experiment is publicly available through \url{https://github.com/saptarshic27/fmintrinsic}.
\section{Theoretical Analyses}
\subsection{Main Result and Inference}
To begin our theoretical analyses, we make the following standard assumption \citep{JMLR:v23:21-0732,oko2023diffusion,JMLR:v26:24-0054} that the data is independent and identically distributed and supported on the unit hypercube, $[0,1]^D$.
\begin{assumption}\label{a1}
    We assume that $x_1, \dots, x_n $ are independent and identically distributed according to the distribution $\pd$, such that $\pd\left([0,1]^D\right) =1$.
\end{assumption}
Our approach relaxes several restrictive assumptions found in recent literature. Unlike \citet{kumar2026flowmatchingadaptivemanifold}, we do not require the existence of densities, smoothness, or a support characterized as a smooth manifold with bounded reach. Additionally, we avoid the uniform score-Lipschitzness and log-concavity of the target measure imposed by \citet{gao2024convergence} and \citet{zhou2025an}. Further, unlike \citet{fukumizu2025flow}, we do not require that  $\pd$ to have a lower-bounded smooth density.

\paragraph{Design Choices} To ensure that the interpolating path is regular enough, we impose the following design choices on $\alpha_t$, and  $\beta_t$. Design choices D~\ref{a2}(a)--(c) are satisfied by the commonly used interpolation schedules, $(\alpha_t = t$, $\beta_t = 1-t)$, $(\alpha_t = t$, $\beta_t = \sqrt{1-t^2})$.
\begin{design} \label{a2}We assume that for any $t, s \in [0, 1-\delta]$, the following hold:
    \begin{itemize}
    \item[(a)] $ \left|\frac{\db_t}{\beta_t}\right|, \ \left|\frac{\da_t \beta_t - \db_t \alpha_t}{\beta_t}\right|, \ a_t^2, \ |b_t|,  \ \beta_t^{-1}$, $\left|\frac{\ddot \beta_t}{\beta_t} - \left(\frac{\dot{\beta}_t}{\beta_t}\right)^2\right| \; \left| \beta_t^{-4}\left(3\alpha_t\beta_t\dot\alpha_t\dot\beta_t \;-\; \beta_t^2\dot\alpha_t^2 \;-\; 2\alpha_t^2\dot\beta_t^2\right)\right| $,\\ $\left| \ddot\alpha_t - \alpha_t\ddot\beta_t\,\beta_t^{-1} - \dot\alpha_t\dot\beta_t\,\beta_t^{-1} + \alpha_t\dot\beta_t^2\,\beta_t^{-2} \right|$, $\left| \alpha_t\big(\dot\alpha_t\beta_t-\dot\beta_t\alpha_t\big)^2\,\beta_t^{-4}\right| \lesssim \delta^{-\kappa_1}$.
        \item[(b)] $|\alpha_t - \alpha_s|, \ |\beta_t - \beta_s|, \,  |a_t - a_s| \lesssim \delta^{-\kappa_2} |t-s| $. 
        \item[(c)] $1-\alpha_{1-\delta}, \ \beta_{1-\delta} \lesssim \delta^{\kappa_3}$.
        \end{itemize}
\end{design}
Note that for the linear interpolation, $(\alpha_t = t$, $\beta_t = 1-t)$, it suffices the above hold with $\kappa_1 = 4$, $\kappa_2 = 2$ and $\kappa_3 = 1$, where as for the Variance Preserving (VP) choice $(\alpha_t = t$, $\beta_t = \sqrt{1-t^2})$, the above hold with  $\kappa_1 = 3$, $\kappa_2 = 0.5$ and $\kappa_3 = 0.5$.

Under Assumption~\ref{a1}, we establish high-probability upper bounds on the Wasserstein-$p$ distance between the flow-matching estimators $\hP_{1-\delta}(\hv_n)$ and $\hP_{1-\delta}(\hv^{\text{mc}}_{n,m})$ and the target distribution $\pd$. These bounds depend on the training sample size through a rate exponent governed solely by the upper $p$-Wasserstein dimension of $\pd$, rather than the ambient dimension $D$. Consequently, with data-driven choices of the discretization step-size $h$, early stopping time $\delta$, and velocity-field function class $\cV$, minimizing the flow-matching MSE loss yields an estimator that automatically adapts to the low-dimensional structure of $\pd$. Theorem~\ref{main_hp} derives a high-probability bound in terms of the Wasserstein-$p$ distance. 
\begin{restatable}[Main Theorem]{theorem}{mainthm} 
    \label{main_hp}
    Suppose that $d > d^\ast_p\left(\pd\right)$ and Assumption~\ref{a1} holds. Then under Design Choices \ref{a2}, for any $p \ge 1$, there exists a constant $n_0 \in \fN$, such that if $n \ge n_0$, and $\cV \equiv \cR \cN\left(W, L, B, K\right)$ with $\delta \asymp n^{-\frac{1}{d \kappa_3}}$, $h \asymp n^{-\frac{2p(2p+1)}{d} - \frac{10\kappa_1+2\kappa_2}{d\kappa_3}}$, $M \asymp n^{\frac{2p}{d}}$, then, with probability at least $1-\xi$, 
    \begin{align}
        \fW_p\left(\pd, \sT_M\left( \hP_{1-\delta}(\hv_n)\right)\right) \lesssim n^{-1/d} + n^{-\frac{1}{2p}} \log^{\frac{1}{2p}}(1/\xi). \label{f1}
    \end{align}
    Here, $W \asymp n^{\frac{p(2p+1)(D+2)}{d} \,+\, \frac{(20+7D)\kappa_1+4\kappa_2}{2d\kappa_3}}\log^{D/2+1} n$, $L \asymp \log n$, $B \asymp n^{\left(\frac{2p(2p+1)}{d}+\frac{20\kappa_1+4\kappa_2}{2d\kappa_3}\right)\,\vee\,\frac{(3\kappa_1)\vee 1}{d\kappa_3}}\sqrt{\log n}$, and $K \asymp n^{\frac{(3\kappa_1) \vee 1}{d\kappa_3}} \sqrt{\log n}$. Further, if
    \(m \asymp n^{\frac{p(2p+1)(D+6)}{d} + \frac{(24+7D)\kappa_1+4\kappa_2+4\left((6\kappa_1)\vee 2\right)}{2d\kappa_3}}\log^2 n\left( \log^{D/2+3} n + \log(1/\xi)\right),
    \)
    with probability at least $1-\xi$,
    \begin{align}
           \fW_p\left(\pd, \sT_M \left(\hP_{1-\delta}(\hv^{\text{mc}}_{n,m})\right)\right) \lesssim n^{-1/d} + n^{-\frac{1}{2p}}\log^{\frac{1}{2p}}(1/\xi).
    \end{align}
\end{restatable}

\paragraph{Inference for Data Supported on a Manifold}
Recall that a set $\sA$ is $\tilde d$-regular w.r.t.\ the $\tilde d$-dimensional Hausdorff measure $\fH^{\tilde d}$ if $\fH^{\tilde d}(B_\varrho(x,r)) \asymp r^{\tilde d}$ for all $x \in \sA$ \citep[Definition~6]{weed2019sharp}. By \citet[Proposition~8]{weed2019sharp}, if $\operatorname{supp}(\pd)$ is $\tilde d$-regular and $\pd \ll \fH^{\tilde d}$, then $d_p^\ast(\pd) = \tilde d$ for every $p < \tilde d/2$; in particular this holds for every compact $\tilde d$-dimensional differentiable manifold, nonempty compact convex set of dimension $\tilde d$, its relative boundary, and self-similar sets of similarity dimension $\tilde d$. For such supports and any $1 \le p < \tilde d/2$, Theorem~\ref{main_hp} applies with every $d > \tilde d$, giving $\cO\big(n^{-1/d} + n^{-1/2p}\log^{1/2p}(1/\xi)\big)$; as $d \downarrow \tilde d$ this matches the minimax rate for distribution estimation on $\tilde d$-dimensional manifolds up to the slack $d - \tilde d$ \citep{chakraborty2025minimax}, and improves on \citet{kumar2026flowmatchingadaptivemanifold}, whose comparable guarantee is not minimax optimal and additionally requires bounded reach and H\"older-smooth densities.

\paragraph{Comparison with Existing Bounds}
Unlike \citet{gao2024convergence,zhou2025an}, whose $\cO(n^{-1/(D+5)})$-type rates depend on the log-concavity or Lipschitzness of the score/velocity field and scale with the ambient dimension $D$, and unlike \citet{fukumizu2025flow}, who require Besov smoothness of the density, our bound in Theorem~\ref{main_hp} holds for every $p \ge 1$, is non-asymptotic and high-probability rather than in-expectation, and imposes no density, smoothness, or Lipschitzness condition on $\pd$ or the velocity class. The claim concerns the exponent of $n$; the prescribed architecture still depends on $D$, and more than polynomially (see below). This places flow matching on the same footing as the dimension-adaptive guarantees recently established for GANs \citep{JMLR:v26:24-0054} and diffusion models \citep{chakraborty2026diffusion}, while relaxing the compact-support and manifold-structure assumptions common to both lines of work.

\paragraph{Choice of Early Stopping and Discretization}
Theorem~\ref{main_hp} sets the early-stopping offset $\delta \asymp n^{-1/(d\kappa_3)}$, which by D~\ref{a2}(c) keeps the truncation error $\fW_p(P_{1-\delta}, \pd) = \cO(\delta^{\kappa_3})$ on the same order as the sampling error, while preventing $v_\star$ from blowing up as $t \to 1^-$. Because the sampler is an SDE with a nondegenerate diffusion coefficient, its increments over a step are of mean square order $h$ rather than $h^2$, and the cumulative Euler--Maruyama error is $\cO(\delta^{-(10\kappa_1+2\kappa_2)}h)$ (Lemma~\ref{lemdis}); matching this to the sampling error gives the step size $h \asymp n^{-\frac{2p(2p+1)}{d} - \frac{10\kappa_1+2\kappa_2}{d\kappa_3}}$ above, and a number of partitions $N \asymp h^{-1}$ that still grows only polynomially in $n$. The proof also fixes internally the approximation accuracy $\epsilon$ and the excursion probability $\eta$, both polynomial in $n$; the latter sets the radius $R \asymp \sqrt{D\log(1/\eta)} \asymp \sqrt{D\log n}$ of the box on which $v_\star$ is approximated.

\paragraph{Network and Monte Carlo Scaling}
The velocity network width $W \asymp n^{\frac{p(2p+1)(D+2)}{d} + \frac{(20+7D)\kappa_1+4\kappa_2}{2d\kappa_3}} \log^{D/2+1} n$ is polynomial in $n$ for each fixed $D$, with an exponent that grows linearly in the ambient dimension $D$ and shrinks as the intrinsic dimension $d$ increases, while depth $L \asymp \log n$ suffices; the caps $B$ and $K$ carry an extra $D\sqrt{\log n}$, the magnitude $v_\star$ attains on the truncation box. The polylogarithmic factors reflect the cost of covering the box $[-R,R]^D$ on which $v_\star$ is approximated, and do not affect the exponent of $n$, which is what governs the rate. The Monte Carlo sample size \(m \asymp n^{\frac{p(2p+1)(D+6)}{d} + \frac{(24+7D)\kappa_1+4\kappa_2+4\left((6\kappa_1)\vee 2\right)}{2d\kappa_3}}\log^2 n\big(\log^{D/2+3} n + \log(1/\xi)\big)\) inherits this $(D,d)$-dependence, inflated by the covering number of $\cV$ and by $K^4$ from the sub-exponential concentration in Lemma~\ref{mc_1}; larger $D/d$ ratios thus demand proportionally more generated samples, mirroring the scaling seen for GANs \citep{JMLR:v26:24-0054,JMLR:v23:21-0732} and diffusion models \citep{chakraborty2026diffusion}.

\subsection{Proof Sketch}
\label{sec_pf_sketch}
 \paragraph{Error Decomposition} To derive an upper bound on the excess risk, we first decompose the KL estimation error into a statistical error due to velocity field estimation and a numerical error resulting from the forward Euler-Maruyama discretization of the SDE. This decomposotion is shown in Lemma~\ref{lem_oracle} with proof appearing in Appendix~\ref{error_decomposition}. Through out the remainder of the paper, all expectations are taken conditioned on the data i.e., $\E (\cdot \mid x_1, \dots, x_n)$.
\begin{restatable}{lemma}{lemoracle}\label{lem_oracle}
Under Assumption~\ref{a1} and Design Choices D~\ref{a2}, for any estimator $\hv \in \cV$ of $v_\star$,
\begin{align}
    & \operatorname{KL}\left(P_{1-\delta} \,\|\, \hP_{1-\delta}(\hv)\right)
    \le  \; \frac{2}{\sigma^2}\sum_{k=0}^{N-1}\int_{t_k}^{t_{k+1}} a_t^2\,
      \E\big\|v_\star(X_t,t) - v_\star(X_{t_k},t_k)\big\|^2 \diff t \nonumber\\
    & + \frac{2}{\sigma^2}\sum_{k=0}^{N-1}\int_{t_k}^{t_{k+1}} (a_t - a_{t_k})^2\,
      \E\big\|v_\star(X_{t_k},t_k)\big\|^2 \diff t  + \frac{1}{\sigma^2}\sum_{k=0}^{N-1} h_k\, a_{t_k}^2\,
      \E\big\|v_\star(X_{t_k},t_k) - \hv(X_{t_k},t_k)\big\|^2,\nonumber
\end{align}
where every expectation is taken under the law of the process $\{X_t\}_{t \in [0,1-\delta]}$ solving \eqref{eq:sde}.
\end{restatable}
 Using Lemma~\ref{lem_oracle}, we aim to bound the terms individually in the subsequent paragraphs. 
 \paragraph{Discretization Error}
The first in Lemma~\ref{lem_oracle},
\(\sum_{k=0}^{N-1} \int_{t_k}^{t_{k+1}}  \E  \|a_t v_\star(X_t,t) -  a_{t_k}v_\star(X_{t_k}, t_k)\|^2 \diff t\)
captures the \emph{discretization error} incurred when the continuous-time SDE generated by the velocity field $v_\star$ is approximated using the forward Euler-Maruyama scheme. This term measures the cumulative discrepancy between the velocity evaluated at the left endpoint of each interval and the true velocity along the continuous trajectory. Consequently, it quantifies the bias introduced by replacing the exact SDE dynamics with piecewise-constant updates over the partition $\{t_j\}_{j=0}^{N}$. The following lemma shows that this error decreases as the uniform discretization step size $h$ tends to zero.
\begin{restatable}{lemma}{lemdis}
    \label{lemdis}
Under Assumption~\ref{a1} and Design Choices D~\ref{a2}, with uniform step size $h \le 1$, the
following hold, all expectations being taken under the law of the process $\{X_t\}$ solving
\eqref{eq:sde}:
\begin{align}
   \quad & \sum_{k=0}^{N-1}\int_{t_k}^{t_{k+1}} a_t^2\,
    \E\big\|v_\star(X_t,t) - v_\star(X_{t_k},t_k)\big\|^2 \diff t
    \;\lesssim\; D^3\big(\sigma^2 \vee 1\big)\,\delta^{-10\kappa_1}\, h, \label{eq:dis1}\\[2pt]
  \quad & \sum_{k=0}^{N-1}\int_{t_k}^{t_{k+1}} (a_t - a_{t_k})^2\,
    \E\big\|v_\star(X_{t_k},t_k)\big\|^2 \diff t
    \;\lesssim\; D\,\delta^{-2(\kappa_1+\kappa_2)}\, h^2 . \label{eq:dis2}
\end{align}
\end{restatable}
In particular, writing $c_1 := 10\kappa_1 + 2\kappa_2$, the sum of the two is
$\lesssim D^3(\sigma^2\vee 1)\,\delta^{-c_1} h$.

\paragraph{Approximation Error}
A key component in controlling the first term in Lemma~\ref{lem_oracle} is the approximation error of the function class $\cV$ used to approximate $v_\star$. Neural network approximation theory has been studied extensively: \cite{cybenko1989approximation} and \cite{hornik1991approximation} established universal approximation for sigmoid-type single-hidden-layer networks on compact domains, and more recent work has characterized the expressive power and approximation rates of deep networks \citep{yarotsky2017error,petersen2018optimal,shen2022optimal,JMLR:v21:20-002,schmidt2020nonparametric}.
\begin{restatable}{lemma}{approxlem} \label{approxlem}
For any $\epsilon \in (0,1)$ and any $R \ge 1$, there exists a network $v_0 \in \cR(W,L,B,K)$ with
$\sup_{x \in [-R,R]^D}\|v_0(x,t_i) - v_\star(x,t_i)\| \le \epsilon$ for all $i = 0,\dots,N$, and
\(\cW(v_0) \lesssim  N R^D \delta^{-3 \kappa_1 D} \epsilon^{-D}\big(\log(1/\delta) + \log(1/\epsilon) + \log R\big), \) 
  \(\cL(v_0) \lesssim  \log(1/\delta) + \log(1/\epsilon) + \log R\), \(\cB(v_0) \lesssim \; \delta^{-(3 \kappa_1) \vee 1} R\sqrt{D} + N, \) and \(\cK(v_0) \lesssim \delta^{-(3 \kappa_1) \vee 1} R\sqrt{D}.\)
Further, $v_0(x,t) = 0$ if $x \not\in [-(R+1),R+1]^D$, and $\operatorname{Lip}(v_0(\cdot, t)) \lesssim \delta^{-3 \kappa_1} R$ for all $t \in [0,1-\delta]$, where the implied constants depend only on $D$.
\end{restatable}
The proof (detailed in Appendix~\ref{mainpf}) then proceeds to control the first term in Lemma~\ref{lem_oracle}. For the estimator~$\hv_n$, this term can be upper bounded by a function of the approximation error,  
as shown in Lemma~\ref{approxlem}. For the estimator $\hv^{\text{mc}}_{n,m}$, controlling this term becomes more nuanced, as this involves controlling both an approximation error and a generalization error that accounts for estimating the expectation via Monte-Carlo in Equation~\eqref{obj_4} as shown in Lemma~\ref{mc_1}. Finally, we use the bound on $\operatorname{KL}\big(P_{1-\delta}, \hP_{1-\delta}(\hv) \big)$ to bound $\fW_p\big(P_{1-\delta}, \hP_{1-\delta}(\hv) \big)$ via truncation as shown in Lemma~\ref{w2}. 

\section{Conclusion and Discussions}
In this work, we developed a theoretical understanding of flow matching for data distributions whose complexity is substantially lower than their ambient dimension. We established the first intrinsic-dimension-adaptive, high-probability generalization bound for flow matching in Wasserstein-$p$ distance, for any $p \ge 1$: with appropriate choices of network architecture, early stopping, discretization scheme and truncation level, given $n$ i.i.d.\ samples from $\pd$, and for every $d > d^\ast_p(\pd)$, with probability at least $1-\xi$,
\(\fW_p\big(\pd, \sT_M(\hP_{1-\delta}(\hv_n))\big) \;\lesssim\; n^{-1/d} + n^{-\frac{1}{2p}} \log^{\frac{1}{2p}}(1/\xi),\)
where $d^\ast_p(\cdot)$ is the upper Wasserstein-$p$ dimension \citep{weed2019sharp} and $\sT_M$ truncates at level $M \asymp n^{2p/d}$. Since $d$ may be taken arbitrarily close to $d^\ast_p(\pd)$, the exponent is governed by the intrinsic dimension up to an arbitrarily small slack. Our analysis dispenses with the uniform Lipschitz assumption on the \emph{learned} velocity field standard in state-of-the-art analyses---the regularity of $v_\star$ itself is established rather than assumed---and prescribes network sizes polynomial in $n$ for each fixed $D$. The resulting rates avoid the ambient-dimensional curse of dimensionality and, for intrinsically low-dimensional distributions---including those on compact smooth manifolds---approach the corresponding minimax rates, giving a theoretical basis for the empirically observed effectiveness of flow matching in high dimensions with low-dimensional structure.

We also address several algorithmic considerations---\emph{discretization}, \emph{early stopping}, and \emph{truncation of the generated law}---with theoretically motivated prescriptions. Our bounds imply an early stopping time $\delta = \Theta(n^{-\frac{1}{d \kappa_3}})$, ensuring the learned velocity field stays close to the ground truth while avoiding blow-up as $t \to 1^-$, and a truncation level $M = \Theta(n^{2p/d})$, balancing the tail of $P_{1-\delta}$ against the transport cost it induces. For numerical integration on $[0,1-\delta]$, we recommend a uniform step size $h = \Theta\left(n^{-\frac{2p(2p+1)}{d} - \frac{10\kappa_1+2\kappa_2}{d\kappa_3}}\right)$. This is dictated by the Brownian component of the sampler: over a step of length $h$ the increments of the generating SDE have mean square of order $h$ rather than $h^2$, so the accumulated Euler--Maruyama error is $\cO(\delta^{-(10\kappa_1+2\kappa_2)}h)$, and the step size above is what places it on the same order as the velocity estimation error. For the velocity network, ReLU feed-forward architectures suffice, with depth $L \asymp \log n$ and $W \asymp n^{\,\frac{p(2p+1)(D+2)}{d} \,+\, \frac{(20+7D)\kappa_1+4\kappa_2}{2d\kappa_3}}\log^{D/2+1} n$ weights. The exponent of $n$ here is linear in $D$ and decreasing in $d$.

While our results give detailed statistical guarantees, the \emph{optimization error} from training nonconvex, time-dependent velocity networks remains hard to control; our bounds are optimization-agnostic and can be combined with future optimization-theoretic advances. Two further directions remain open. First, our approximation analysis uses only the Lipschitz regularity of $v_\star$; exploiting its Gaussian-mixture structure, or additional smoothness, should yield smaller networks. Second, we assume a ReLU velocity network---standard in the statistical learning literature and compatible with existing approximation results for H\"older and Besov-type classes---whereas practical models use structured architectures such as U-Nets or transformers. These often exhibit strong empirical smoothness and hierarchical approximation properties, but rigorous guarantees for them remain unavailable.

\section*{Acknowledgements}
We gratefully acknowledge the support of the NSF and the Simons Foundation for the Collaboration on the Theoretical Foundations of Deep Learning through awards DMS-2031883 and \#814639, the NSF's support of FODSI through grant DMS-2023505, and the ONR's support through MURI award N000142112431. 


\clearpage
\appendix
\section*{Appendices}
\startcontents[sections]
\printcontents[sections]{l}{1}{\setcounter{tocdepth}{2}}
\section{Omitted Proofs from Section~\ref{background}}
\subsection{Existence and Uniqueness of the True Vector Velocity Field (Proof of Theorem~\ref{existence})}
\thmexist*
\begin{proof}
    The characteristic function of $X_t$ is given by,
\begin{align*}
    \varphi_t(\kappa) = \E \exp(\iota \kappa^\top X_t) = \int \exp(\iota \kappa^\top (\alpha_t x_1 + \beta_t x_0 )) \diff P_0(x_0) \diff \pemp(x_1).
\end{align*}
For any $t \in [0,1-\delta]$, $\int |\varphi_t(\kappa)| \diff \kappa \lesssim \delta^{-2 \kappa_1}(2\pi)^{D/2} < \infty$, by Lemma~\ref{lem_cond_1}. Hence, applying the Fourier inversion theorem, $X_t$ has a density and is given by,
\begin{align*}
    \varrho_t(x) = (2\pi)^{-D} \int \exp(-\iota\kappa^\top x) \varphi_t(\kappa) \diff \kappa.
\end{align*}
By Lemma~\ref{lem_cond_1}, since,
\begin{align}
    & \int \left\|\int (x_1 - x_0) \exp\left(\iota \kappa^\top (\alpha_t x_1 + \beta_t x_0 )\right) \diff P_0(x_0) \diff \pemp(x_1)\right\| \diff \kappa \nonumber \\
    \lesssim & \frac{\delta^{\kappa_1}\sqrt{2}
   \Gamma\left(\frac{D+1}{2}\right)}{\Gamma\left(\frac{D}{2}\right)} \times \frac{\delta^{\kappa_1}(2\pi)^{D/2}}{\delta^{\kappa_1(D+1)}} + \frac{\sqrt{D}(2\pi)^{D/2}}{\delta^{\kappa_1 D}} < \infty, \label{eq6}
\end{align}
Applying the Fubini-Tonelli Theorem \citep[Corollary 1.7.23]{tao2011measure},
\begingroup
\allowdisplaybreaks
\begin{align}
     \partial_t \varrho_t(x) 
    = & (2\pi)^{-D} \int \exp(-\iota\kappa^\top x) \varphi_t^\prime(\kappa) \diff \kappa \nonumber\\
    = & (2\pi)^{-D} \int \iota \kappa^\top (\da_t x_1 + \db_t
    x_0) \exp(-\iota\kappa^\top x) \exp\left(\iota \kappa^\top (\alpha_t x_1 + \beta_t x_0)\right) \diff P_0(x_0) \diff \pemp(x_1) \diff \kappa \nonumber\\
    = & - \operatorname{div}_x \left( (2\pi)^{-D} \int  (\da_t x_1 + \db_t x_0) \exp(-\iota\kappa^\top x) \exp\left(\iota \kappa^\top (\alpha_t x_1 + \beta_t x_0)\right) \diff P_0(x_0) \diff \pemp(x_1) \diff \kappa \right) \nonumber\\
    = & - \operatorname{div}_x \bigg(  \int   (\da_t x_1 + \db_t x_0) \left( (2\pi)^{-D}\int \exp(-\iota\kappa^\top x) \exp\left(\iota \kappa^\top (\alpha_tx_1 + \beta_t x_0)\right) \diff \kappa \right) \nonumber\\
    & \hspace{2cm} \diff P_0(x_0) \diff \pemp(x_1) \bigg). \label{eq3}
\end{align}
\endgroup
We note that,
\begingroup
\allowdisplaybreaks
\begin{align}
 & \int \exp(\iota \kappa^\top x) v_\star(x,t) \varrho_t(x) \diff x \nonumber\\
= & \E \left(\exp(\iota \kappa^\top X_t) v_\star(X_t,t)\right) \nonumber\\
= & \E \left(\exp(\iota \kappa^\top X_t) \E\left( \dot{X}_t\mid X_t\right)\right) \nonumber\\
 = & \E \E\left(\exp(\iota \kappa^\top X_t) \dot{X}_t\mid X_t\right) \nonumber  \\
 = & \E \left(\exp(\iota \kappa^\top X_t) \dot{X}_t\right) \nonumber\\
 = & \E \left(\exp(\iota \kappa^\top (\alpha_t X_1 + \beta_t X_0)) (\da_t X_1 + \db_t X_0)\right) \nonumber\\
 = & \int \exp(\iota \kappa^\top (\alpha_t x_1 + \beta_t x_0)) (\da_t x_1 + \db_t x_0) \diff P_0(x_0) \diff \pemp(x_1) \label{eq4}
\end{align}
\endgroup
Taking an inverse Fourier transform on both sides, we get,
\begin{align}
   & v_\star(x,t) \varrho_t(x) \nonumber\\
   = & (2 \pi)^{-D}\int \exp(-\iota \kappa^\top x) \exp(\iota \kappa^\top (\alpha_tx_1 + \beta_t x_0)) (\da_t x_1 + \db_t x_0) \diff P_0(x_0) \diff \pemp(x_1) \diff \kappa. \label{eq5}
\end{align}
Note that we can apply the inverse Fourier transform since \eqref{eq6} is satisfied. Hence, from \eqref{eq3} and \eqref{eq5}, we get,
\begin{align}\label{fpk}
    \partial_t \varrho_t(x) +  \operatorname{div}_x  (v_\star(x,t) \varrho_t (x)) =0.
\end{align}
Note that since $v_\star(\cdot, t)$ is uniformly (in $t \in [0,1-\delta]$) $D\delta^{-3\kappa_1}$-Lipschitz, by Corollary~\ref{cor14}, The Picard–Lindel\"of theorem \citep[Theorem 1.3.1]{coddington1955theory} implies that \eqref{ode_2} has a unique solution for all initialization $x \in \Real^D$. From \citet[Theorem~4.4]{Santambrogio2015}, $(Y_t)_\sharp P_0$ satisfies the continuity equation \eqref{fpk}, which admits a unique solution.
\end{proof}
\begin{lemma}\label{lem_cond_1}
For any $0 \le t < 1 $, the following holds,
    \begin{align}
       & \int \left|\int \exp\left(\iota \kappa^\top (\alpha_t x_1 + \beta_t x_0 )\right) \diff P_0(x_0) \diff \pemp(x_1)\right| \diff \kappa \le \frac{(2\pi)^{D/2}}{\beta_t^D}, \label{eq1}\\
       &   \int \left\|\int (x_1 - z) \exp\left(\iota \kappa^\top (\alpha_t x_1 + \beta_t x_0 )\right) \diff P_0(x_0) \diff \pemp(x_1)\right\| \diff \kappa \nonumber\\
       \le & \frac{\db_t\sqrt{2}
   \Gamma\left(\frac{D+1}{2}\right)}{\Gamma\left(\frac{D}{2}\right)} \times \frac{(2\pi)^{D/2}}{\beta_t^{D+1}} + \frac{\da_t\sqrt{D}(2\pi)^{D/2}}{\beta_t^D} \label{eq2}.
    \end{align}
\end{lemma}
\begin{proof}
\begingroup
\allowdisplaybreaks
    \begin{align*}
   & \int \left|\int \exp\left(\iota \kappa^\top (\alpha_t x_1 + \beta_t x_0 )\right) \diff P_0(x_0) \diff \pemp(x_1)\right| \diff \kappa\\
   \le & \frac{1}{n} \sum_{i=1}^n \int \left|\int \exp\left(\iota \kappa^\top (\alpha_t x_i + \beta_t x_0 )\right) \diff P_0(x_0) \right| \diff \kappa\\
   = & \frac{1}{n} \sum_{i=1}^n \int \left|\exp(\iota \alpha_t \kappa^\top x_i) \exp\left(-\frac{1}{2}\beta_t^2 \|\kappa\|^2\right) \right| \diff \kappa\\
   \le & \int \exp\left(-\frac{1}{2}\beta_t^2 \|\kappa\|^2\right)  \diff \kappa\\
   = & \left(\frac{2\pi}{\beta_t^2}\right)^{D/2}.
\end{align*}
\endgroup
\begingroup
\allowdisplaybreaks
\begin{align*}
  & \int \left\|\int (\da_t x_1 + \db_t x_0) \exp\left(\iota \kappa^\top (\alpha_t x_1 + \beta_t x_0 )\right) \diff P_0(x_0) \diff \pemp(x_1)\right\| \diff \kappa\\
   \le & \frac{1}{n}\sum_{i=1}^n \int \left\|\int (\da_t x_i + \db_t x_0) \exp\left(\iota \kappa^\top (\alpha_t x_i + \beta_t x_0)\right) \diff P_0(x_0) \right\| \diff \kappa\\
   \le & \frac{\db_t}{n}\sum_{i=1}^n \int \left\|\int  x_0 \exp\left(\iota \kappa^\top (\alpha_t x_i + \beta_t x_0)\right) \diff P_0(x_0) \right\| \diff \kappa \\
   & + \frac{\da_t}{n}\sum_{i=1}^n \int \left\|\int x_i  \exp\left(\iota \kappa^\top (\alpha_t x_i + \beta_t x_0)\right) \diff P_0(x_0) \right\| \diff \kappa\\
   \le & \frac{\db_t}{n}\sum_{i=1}^n \int \left\|\int  x_0 \exp\left(\iota \kappa^\top (\alpha_t x_i + \beta_t x_0)\right) \diff P_0(x_0) \right\| \diff \kappa \\
   & + \frac{\da_t}{n}\sum_{i=1}^n \int \left\|x_i \exp\left(\iota \alpha_t \kappa^\top  x_i \right) \int   \exp\left(\iota \kappa^\top  \beta_t x_0 \right) \diff P_0(x_0) \right\| \diff \kappa\\
   = & \frac{\db_t}{n}\sum_{i=1}^n \int \left\| \exp(\iota \alpha_t \kappa^\top x_i) \cdot \iota \beta_t \kappa \exp\left( -\frac{\beta_t^2 \|\kappa\|^2}{2}
   \right)\right\| \diff \kappa \\
   & + \frac{\da_t}{n}\sum_{i=1}^n \int \left\| x_i \exp(\iota \alpha_t \kappa^\top x_i)  \exp\left( -\frac{\beta_t^2 \|\kappa\|^2}{2}
   \right)\right\| \diff \kappa\\
   \le & \db_t \int \| \kappa\| \exp\left( -\frac{\beta_t^2 \|\kappa\|^2}{2}
   \right) \diff \kappa + \frac{\da_t}{n}\sum_{i=1}^n \|x_i\| \int  \exp\left( -\frac{\beta_t^2 \|\kappa\|^2}{2}
   \right) \diff \kappa\\
   = & \frac{\sqrt{2}
   \Gamma\left(\frac{D+1}{2}\right)}{\Gamma\left(\frac{D}{2}\right)} \times \frac{(2\pi)^{D/2}}{\beta_t^{D+1}} + \frac{\da_t}{n}\sum_{i=1}^n \|x_i\| \times \frac{(2\pi)^{D/2}}{\beta_t^D}\\
   \le & \frac{\db_t \sqrt{2}
   \Gamma\left(\frac{D+1}{2}\right)}{\Gamma\left(\frac{D}{2}\right)} \times \frac{(2\pi)^{D/2}}{\beta_t^{D+1}} + \frac{\da_t\sqrt{D}(2\pi)^{D/2}}{\beta_t^D}.
\end{align*}
\endgroup
\end{proof}
\subsection{Equivalence to the Conditional Flow-matching Objective}
\begin{lemma}\label{cftofm}
    $\mathcal{L}_{\mathrm{CFM}}(v) = \cL(v) + \int_0^1 \mathbb{E} \|v_\star(X_t,t) - \dot{X_t}\|^2 \diff t $.
\end{lemma}
\begin{proof}
    Note that,
    \begingroup
    \allowdisplaybreaks
\begin{align*}
    & \mathcal{L}_{\mathrm{CFM}}(v) \\
    = & \int_0^1 \mathbb{E}\|v(X_t, t) - \dot{X_t}\|^2 \diff t\\
    = & \int_0^1 \mathbb{E}\left(\|v(X_t, t) - v_\star(X_t,t)\|^2 + \|v_\star(X_t,t) - \dot{X_t}\|^2 + 2 \langle v(X_t, t) - v_\star(X_t,t), v_\star(X_t,t) - \dot{X_t}\rangle \right)\diff t\\
    = & \int_0^1 \mathbb{E}\|v(X_t, t) - v_\star(X_t,t)\|^2 \diff t + \int_0^1 \mathbb{E} \|v_\star(X_t,t) - \dot{X_t}\|^2 \diff t \\
    & + 2 \int_0^1 \mathbb{E} \E \left(\langle v(X_t, t) - v_\star(X_t,t), v_\star(X_t,t) - \dot{X_t}\rangle\mid X_t\right) \diff t\\
    = & \cL(v) + \int_0^1 \mathbb{E} \|v_\star(X_t,t) - \dot{X_t}\|^2 \diff t .
\end{align*}
\endgroup
\end{proof}

\subsection{Equivalence of the probability-flow ODE and the SDE family}
\label{sec:ode-sde-equivalence}

Let $\{\varrho_t\}_{t\in[0,1]}$ be the marginal law of the stochastic interpolant
$X_t = \alpha_t X_1 + \beta_t X_0$, absolutely continuous with density $\varrho_t$
for $t\in(0,1)$, and let $v_\star(x,t)$ denote the interpolant velocity field. For a
constant $\sigma>0$ define
\begin{equation}
    \mathfrak{b}_t(x) \;:=\; v_\star(x,t) + \frac{\sigma^2}{2}\nabla \log \varrho_t(x).
    \label{eq:drift-def}
\end{equation}

By Theorem~\ref{existence}, $v_\star$ solves, in the sense of distributions on
$(0,1)\times\mathbb{R}^D$,
\begin{equation}
    \partial_t \varrho_t + \nabla\cdot(v_\star \varrho_t) = 0,
    \qquad \varrho_t \big|_{t\to 0^+} = \varrho_0 = \cN(0,I_D) \text{ weakly}.
    \label{eq:continuity}
\end{equation}

We first record the regularity of the score, which is what makes \eqref{eq:sde-sigma} below
well posed.

\begin{lemma}[Lipschitzness of the score]
\label{lem:score-lip}
Under Assumption~\ref{a1} and Design Choices D~\ref{a2}, for every $t\in[0,1-\delta]$ the map
$x\mapsto\nabla\log\varrho_t(x)$ is differentiable with
\[
  \nabla^2 \log\varrho_t(x)
  \;=\; \frac{\alpha_t^2}{\beta_t^4}\,\mathrm{Var}\big(X_1 \mid X_t = x\big) \;-\; \frac{1}{\beta_t^2}\,I_D ,
  \qquad
  \big\|\nabla^2\log\varrho_t(x)\big\| \;\lesssim\; D\,\delta^{-4\kappa_1},
\]
uniformly in $x\in\mathbb{R}^D$ and $t\in[0,1-\delta]$. In particular $\nabla\log\varrho_t(\cdot)$
is globally Lipschitz with constant $\cO(D\delta^{-4\kappa_1})$, uniformly in $t\in[0,1-\delta]$.
\end{lemma}

\begin{proof}
By Lemma~\ref{lemlip}, $\nabla\log\varrho_t(x) = -\beta_t^{-2}\big(x - \alpha_t m(x,t)\big)$ with
$m(x,t) = \E[X_1\mid X_t=x]$, and $\nabla_x m(x,t) = \alpha_t\beta_t^{-2}\,\mathrm{Var}(X_1\mid X_t=x)$.
Differentiating the first identity in $x$ and substituting the second gives the display.
For the bound, $\mathrm{Var}(X_1\mid X_t=x) \succeq 0$ and, since $X_1\sim\pemp$ is supported in
$[0,1]^D$ by Assumption~\ref{a1}, each conditional coordinate variance is at most $1/4$, so
\[
  \big\|\mathrm{Var}(X_1\mid X_t=x)\big\| \;\le\; \operatorname{trace} \left(\mathrm{Var}(X_1\mid X_t=x) \right) \;\le\; \tfrac{D}{4}.
\]
With $\alpha_t \le 1$ and $\beta_t^{-1}\lesssim \delta^{-\kappa_1}$ by D~\ref{a2}(a),
$\|\nabla^2\log\varrho_t(x)\| \le \tfrac{D}{4}\alpha_t^2\beta_t^{-4} + \beta_t^{-2} \lesssim D\delta^{-4\kappa_1}$.
\end{proof}

\begin{theorem}[Marginal equivalence of the ODE and the SDE family]
\label{thm:equivalence}
Fix $\delta \in (0,1/2)$. Under Assumption~\ref{a1} and Design Choices D~\ref{a2}, there is a
pathwise unique strong solution $\{X_t\}_{t\in[0,1-\delta]}$ of
    \begin{equation}
        dX_t = \mathfrak{b}_t(X_t)\,dt + \sigma\,dW_t, \qquad X_0 \sim\cN(0,I_D),
        \label{eq:sde-sigma}
    \end{equation}
    with $\mathrm{Law}(X_t) = \varrho_t$ for every $t\in[0,1-\delta]$.
\end{theorem}

\begin{proof}
Let $\mathcal{L}_t^\sigma f := \mathfrak{b}_t\cdot\nabla f + \tfrac{\sigma^2}{2}\Delta f$ be
the generator of \eqref{eq:sde-sigma}, with formal adjoint
$(\mathcal{L}_t^\sigma)^*$. Since $\varrho_t$ is a Gaussian smoothing of a compactly supported
measure it is smooth and strictly positive on $\mathbb{R}^D$ for every $t\in(0,1)$, so the
following manipulations are classical. Using the score identity
$\varrho_t \nabla\log\varrho_t = \nabla\varrho_t$,
\begin{align}
    (\mathcal{L}_t^\sigma)^*\varrho_t
    = & -\nabla\cdot(\mathfrak{b}_t \varrho_t) + \frac{\sigma^2}{2}\Delta\varrho_t \notag\\
    = & -\nabla\cdot\!\Big(v_\star \varrho_t + \frac{\sigma^2}{2}\nabla\varrho_t\Big)
       + \frac{\sigma^2}{2}\Delta\varrho_t \notag \\
    = & -\nabla\cdot(v_\star \varrho_t) \notag\\
    = &\; \partial_t \varrho_t.
    \label{eq:fpk-cancel}
\end{align}
In the above calculations, \eqref{eq:fpk-cancel} follows from \eqref{eq:continuity}. This holds \emph{for every} $\sigma\geq 0$ simultaneously: the diffusive term
$\tfrac{\sigma^2}{2}\Delta\varrho_t$ contributed by the noise is exactly
cancelled by the score correction in \eqref{eq:drift-def}, not merely
asymptotically, and \eqref{eq:fpk-cancel} does not depend on $\sigma$ at all.
Hence $t\mapsto\varrho_t\,dx$ is a probability solution of
\begin{equation}
    \partial_t \mu_t = (\mathcal{L}_t^\sigma)^*\mu_t, \qquad t\in[0,1-\delta],
    \qquad \mu_0 = \varrho_0 = \cN(0,I_D),
    \label{eq:fpk-sigma}
\end{equation}
for every admissible $\sigma$, the initial condition being the weak limit in
\eqref{eq:continuity}.

By Corollary~\ref{cor14}, $v_\star(\cdot,t)$ is $\cO(D\delta^{-3\kappa_1})$-Lipschitz uniformly
over $t\in[0,1-\delta]$, and by Lemma~\ref{lem:score-lip}, $\nabla\log\varrho_t(\cdot)$ is
$\cO(D\delta^{-4\kappa_1})$-Lipschitz uniformly over the same range. Hence
\[
    \mathfrak{b}_t(x) = v_\star(x,t) + \tfrac{\sigma^2}{2}\nabla\log\varrho_t(x)
\]
is globally Lipschitz in $x$, uniformly over $t\in[0,1-\delta]$, with Lipschitz constant
$\Lambda_\delta := \cO\big(D(1\vee\sigma^2)\delta^{-4\kappa_1}\big)$, and is of at most linear
growth, since $\|\mathfrak{b}_t(0)\| \lesssim \delta^{-\kappa_1}\sqrt D$ by Corollary~\ref{cor14}
and $\|\nabla\log\varrho_t(0)\| \le \beta_t^{-2}\alpha_t\E\|X_1\| \lesssim \delta^{-2\kappa_1}\sqrt{D}$.
Since $\sigma$ is a fixed positive constant, the classical It\^o
existence--uniqueness theorem for SDEs with globally Lipschitz coefficients of linear growth
\citep[Theorem~2.9, Chapter~5]{karatzas1991brownian} gives a pathwise unique strong solution
$\{X_t\}_{t\in[0,1-\delta]}$ of \eqref{eq:sde-sigma} for any initial law with
finite second moment, in particular for $\mathrm{Law}(X_0)=\cN(0,I_D)$.

The drift $\mathfrak{b}$ is globally Lipschitz in $x$ uniformly in $t\in [0,1-\delta]$ and of linear growth, and
the diffusion coefficient $\sigma I_D$ is constant and nondegenerate, so the Cauchy problem
\eqref{eq:fpk-sigma} admits at most one probability solution on $[0,1-\delta]$ with initial datum
$\varrho_0$ \citep[Ch.~9]{bogachev2015fokker}. Equivalently, and without appeal to that theory:
$t\mapsto\mathrm{Law}(X_t)$ for the pathwise-unique solution of \eqref{eq:sde-sigma} is itself a
probability solution of \eqref{eq:fpk-sigma} with the same initial datum, and two probability
solutions of a well-posed SDE with the same initial law have equal marginals. From the above arguments, the map,
$t\mapsto\varrho_t$ is also a probability solution of \eqref{eq:fpk-sigma} with
$\varrho_t|_{t=0}=\varrho_0=\cN(0,I_D)$. Uniqueness therefore forces
$\mathrm{Law}(X_t)=\varrho_t$ for every $t\in[0,1-\delta]$.
\end{proof}

\section{Omitted Proofs from Section~\ref{sec_pf_sketch}}
\subsection{Error Decomposition}
\label{error_decomposition}
\lemoracle*
\begin{proof}
Write $\Omega = C([0,1-\delta];\Real^D)$ for path space and let $\bP$ denote the law on $\Omega$ of the
solution $\{X_t\}$ of
\[
  \diff X_t = \big(a_t v_\star(X_t,t) - b_t X_t\big)\diff t + \sigma \diff W_t,
  \qquad X_0 \sim \cN(0, I_D),
\]
which exists and is pathwise unique by Theorem~\ref{thm:equivalence}. Let $\bQ$ denote the law on $\Omega$ of the
Euler--Maruyama process \eqref{exponential_integrator}, namely the solution of
\[
  \diff \hX_t = \big(a_{t_k}\hv(\hX_{t_k},t_k) - b_t \hX_t\big)\diff t + \sigma \diff W_t,
  \quad t \in [t_k, t_{k+1}), \qquad \hX_0 \sim \cN(0,I_D).
\]
Both are driven by the same diffusion coefficient $\sigma > 0$ and started from the same law, so
$\bP$ and $\bQ$ differ only through their drifts. Writing these as functionals of the path,
\[
  \mu^{\bP}_t(\omega) = a_t v_\star(\omega_t, t) - b_t\omega_t,
  \qquad
  \mu^{\bQ}_t(\omega) = a_{t_k}\hv(\omega_{t_k}, t_k) - b_t\omega_t
  \quad \text{for } t \in [t_k,t_{k+1}),
\]
both are adapted, and the linear terms $-b_t\omega_t$ coincide, so that
\begin{equation}
  \mu^{\bP}_t(\omega) - \mu^{\bQ}_t(\omega)
  = a_t v_\star(\omega_t,t) - a_{t_k}\hv(\omega_{t_k},t_k),
  \qquad t \in [t_k, t_{k+1}).
  \label{eq:driftdiff}
\end{equation}
By Corollary~\ref{cor14}, $\|v_\star(x,t)\| \lesssim \delta^{-\kappa_1}(\|x\| + \sqrt D)$ uniformly on
$[0,1-\delta]$; by D~\ref{a2}(a), $\sup_{t \le 1-\delta}|a_t| \lesssim \delta^{-\kappa_1/2}$; and
$\|\hv(\cdot,\cdot)\| \le \sqrt D\,\cK(\hv) \le \sqrt D K$. Hence the drift difference
\eqref{eq:driftdiff} satisfies
\[
  \big\|\mu^{\bP}_t(\omega) - \mu^{\bQ}_t(\omega)\big\|
  \;\le\; C_\delta\Big(1 + \sup_{s \le t}\|\omega_s\|\Big),
  \qquad t \in [0,1-\delta],
  \qquad C_\delta \;\lesssim\; \delta^{-\kappa_1/2}\sqrt D\,\big(\delta^{-\kappa_1} + K\big),
\]
which is finite for each fixed $\delta > 0$. Since Bene\v{s}'s condition requires only linear growth
of the drift difference in the sup-norm of the path, with no restriction on the constant, it holds here
\citep[Corollary~3.5.16]{karatzas1991brownian}, and the stochastic exponential below is a true
martingale. Girsanov's theorem \citep[Theorem~8.6.5]{oksendal2003stochastic}
therefore gives $\bQ \sim \bP$ with 
\[
  \frac{\diff \bP}{\diff \bQ}(\omega)
  = \exp\left(\frac{1}{\sigma}\int_0^{1-\delta}\big\langle \mu^{\bP}_t - \mu^{\bQ}_t, \diff W_t\big\rangle
    - \frac{1}{2\sigma^2}\int_0^{1-\delta}\big\|\mu^{\bP}_t - \mu^{\bQ}_t\big\|^2 \diff t\right),
\]
and taking $\E_{\bP}\log(\diff\bP/\diff\bQ)$, the stochastic integral being a true martingale under $\bP$,
\begin{equation}
  \operatorname{KL}(\bP \,\|\, \bQ)
  = \frac{1}{2\sigma^2}\sum_{k=0}^{N-1}\int_{t_k}^{t_{k+1}}
    \E\big\|a_t v_\star(X_t,t) - a_{t_k}\hv(X_{t_k},t_k)\big\|^2 \diff t .
  \label{eq:pathkl}
\end{equation}
The marginal laws at the terminal time are $P_{1-\delta} = \bP \circ \pi_{1-\delta}^{-1}$ and
$\hP_{1-\delta}(\hv) = \bQ \circ \pi_{1-\delta}^{-1}$, where $\pi_{1-\delta}(\omega) = \omega_{1-\delta}$,
the first identification being Theorem~\ref{thm:equivalence}. Since the Kullback--Leibler divergence contracts
under the pushforward by $\pi_{1-\delta}$,
\begin{equation}
  \operatorname{KL}\left(P_{1-\delta}\,\|\,\hP_{1-\delta}(\hv)\right) \le \operatorname{KL}(\bP\,\|\,\bQ).
  \label{eq:dpi}
\end{equation}
It remains to split \eqref{eq:pathkl}. Adding and subtracting $a_{t_k}v_\star(X_{t_k},t_k)$ and using
$\|u+v\|^2 \le 2\|u\|^2 + 2\|v\|^2$,
\begin{align*}
  \operatorname{KL}(\bP\,\|\,\bQ)
  \le & \; \frac{1}{\sigma^2}\sum_{k=0}^{N-1}\int_{t_k}^{t_{k+1}}
    \E\big\|a_t v_\star(X_t,t) - a_{t_k} v_\star(X_{t_k},t_k)\big\|^2 \diff t \\
  & + \frac{1}{\sigma^2}\sum_{k=0}^{N-1}\int_{t_k}^{t_{k+1}} a_{t_k}^2\,
    \E\big\|v_\star(X_{t_k},t_k) - \hv(X_{t_k},t_k)\big\|^2 \diff t .
\end{align*}
The second sum has an integrand constant in $t$ on each $[t_k,t_{k+1})$ and so equals
$\sigma^{-2}\sum_k h_k a_{t_k}^2\,\E\|v_\star(X_{t_k},t_k) - \hv(X_{t_k},t_k)\|^2$. For the first, adding and
subtracting $a_t v_\star(X_{t_k},t_k)$ and splitting once more,
\[
  \big\|a_t v_\star(X_t,t) - a_{t_k}v_\star(X_{t_k},t_k)\big\|^2
  \le 2a_t^2\big\|v_\star(X_t,t) - v_\star(X_{t_k},t_k)\big\|^2
    + 2(a_t - a_{t_k})^2\big\|v_\star(X_{t_k},t_k)\big\|^2 .
\]
Combining the three displays with \eqref{eq:dpi} gives the claim. \qedhere
\end{proof}
 \subsection{Discretization Error}

\lemdis*

\begin{proof}
Throughout, $t \in [t_k, t_{k+1})$ for the relevant $k$, and all expectations are under the law of
$\{X_t\}_{t\in[0,1-\delta]}$ solving \eqref{eq:sde}. By Theorem~\ref{thm:equivalence},
$\mathrm{Law}(X_s) = \varrho_s = \mathrm{Law}(\alpha_s X_1 + \beta_s X_0)$ for each fixed $s$. Since $X_1 \sim \pemp$ is supported in
$[0,1]^D$ by Assumption~\ref{a1} and $X_0 \sim \cN(0,I_D)$,
\begin{equation}
  \E\|X_s\|^2 \le 2\alpha_s^2\,\E\|X_1\|^2 + 2\beta_s^2\,\E\|X_0\|^2 \lesssim D,
  \qquad s \in [0,1-\delta].
  \label{eq:secondmoment}
\end{equation}
Combining \eqref{eq:secondmoment} with Corollary~\ref{cor14},
\begin{equation}
  \E\big\|v_\star(X_s,s)\big\|^2 \lesssim \delta^{-2\kappa_1}\Big(\E\|X_s\|^2 + D\Big)
  \lesssim D\,\delta^{-2\kappa_1}.
  \label{eq:vstarmoment}
\end{equation}

We write $\mu_s(x) := a_s v_\star(x,s) - b_s x$ for the drift
of \eqref{eq:sde}. By \eqref{eq:vstarmoment}, D~\ref{a2}(a) (which gives $a_s^2 \lesssim \delta^{-\kappa_1}$
and $|b_s| \lesssim \delta^{-\kappa_1}$) and \eqref{eq:secondmoment},
\begin{equation}
  \E\|\mu_s(X_s)\|^2 \le 2a_s^2\,\E\|v_\star(X_s,s)\|^2 + 2b_s^2\,\E\|X_s\|^2
  \lesssim D\,\delta^{-3\kappa_1}.
  \label{eq:driftmoment}
\end{equation}
Integrating \eqref{eq:sde} from $t_k$ to $t$,
\[
  X_t - X_{t_k} = \int_{t_k}^{t}\mu_s(X_s)\,\diff s + \sigma\big(W_t - W_{t_k}\big).
\]
By $\|u+v\|^2 \le 2\|u\|^2+2\|v\|^2$, the Cauchy--Schwarz inequality applied to the drift
integral, and $\E\|W_t - W_{t_k}\|^2 = D(t-t_k)$,
\begin{align}
  \E\big\|X_t - X_{t_k}\big\|^2
  &\le 2(t-t_k)\int_{t_k}^{t}\E\|\mu_s(X_s)\|^2\,\diff s + 2\sigma^2 D\,(t-t_k) \nonumber\\
  &\lesssim D\,\delta^{-3\kappa_1}(t-t_k)^2 + \sigma^2 D\,(t-t_k) \nonumber\\
  &\lesssim\; D\big(\sigma^2 + \delta^{-3\kappa_1}h\big)(t-t_k).
  \label{eq:increment}
\end{align}

Note that the $\sigma^2$ appearing here is exactly cancelled by the factor $\sigma^{-2}$ carried by
Lemma~\ref{lem_oracle}, so the resulting contribution to the KL bound is independent of $\sigma$; see
Remark~\ref{rem:sigma} below.

 Splitting in space and in time,
\begin{align}
  & \E\big\|v_\star(X_t,t) - v_\star(X_{t_k},t_k)\big\|^2 \nonumber\\
  \le\; & 2\,\E\big\|v_\star(X_t,t) - v_\star(X_{t_k},t)\big\|^2
    + 2\,\E\big\|v_\star(X_{t_k},t) - v_\star(X_{t_k},t_k)\big\|^2 .
  \label{eq:split}
\end{align}
For the first summand, Corollary~\ref{cor14} gives $\operatorname{Lip}(v_\star(\cdot,t)) \lesssim D\delta^{-3\kappa_1}$
uniformly in $t \in [0,1-\delta]$. Thus, from \eqref{eq:increment},
\begin{equation}
  \E\big\|v_\star(X_t,t) - v_\star(X_{t_k},t)\big\|^2
  \lesssim D^2\delta^{-6\kappa_1}\,\E\|X_t - X_{t_k}\|^2
  \lesssim D^3\delta^{-6\kappa_1}\big(\sigma^2 + \delta^{-3\kappa_1}h\big)(t-t_k).
  \label{eq:spatial}
\end{equation}
For the second summand, the two evaluations share the spatial argument $X_{t_k}$, so only the
one-time marginal of $X_{t_k}$ is involved. By Lemma~\ref{lip_general},
$\|\partial_t v_\star(x,t)\| \lesssim \delta^{-\kappa_1}(1+\|x\|)$ uniformly on $[0,1-\delta]$. Hence,
by the mean value theorem and \eqref{eq:secondmoment},
\begin{equation}
  \E\big\|v_\star(X_{t_k},t) - v_\star(X_{t_k},t_k)\big\|^2
  \lesssim \delta^{-2\kappa_1}\,\E\big(1 + \|X_{t_k}\|^2\big)(t-t_k)^2
  \lesssim D\,\delta^{-2\kappa_1}(t-t_k)^2 .
  \label{eq:temporal}
\end{equation}
 Using $a_t^2 \lesssim \delta^{-\kappa_1}$, then \eqref{eq:split},
\eqref{eq:spatial} and \eqref{eq:temporal}, and finally
$\int_{t_k}^{t_{k+1}}(t-t_k)\diff t = h^2/2$ and $\int_{t_k}^{t_{k+1}}(t-t_k)^2\diff t = h^3/3$,
\begin{align*}
   \sum_{k=0}^{N-1}\int_{t_k}^{t_{k+1}} a_t^2\,\E\big\|v_\star(X_t,t) - v_\star(X_{t_k},t_k)\big\|^2 \diff t 
  \lesssim\; & \delta^{-\kappa_1}\left[
      D^3\delta^{-6\kappa_1}\big(\sigma^2 + \delta^{-3\kappa_1}h\big)\sum_{k=0}^{N-1}\frac{h^2}{2}
      + D\,\delta^{-2\kappa_1}\sum_{k=0}^{N-1}\frac{h^3}{3}\right] \\
  \lesssim\; & D^3\delta^{-7\kappa_1}\big(\sigma^2 + \delta^{-3\kappa_1}h\big)\,h
      + D\,\delta^{-3\kappa_1}h^2 \\
  \lesssim\; & D^3\big(\sigma^2 \vee 1\big)\,\delta^{-10\kappa_1}\,h ,
\end{align*}
where the second line uses $Nh = 1-\delta \le 1$, so that $\sum_k h^2 \le h$ and $\sum_k h^3 \le h^2$,
and the last uses $h \le 1$ to absorb $h^2$ into $h$. This is \eqref{eq:dis1}.

To bound \eqref{eq:dis2}, we note that this term involves $v_\star$ at the single time $t_k$ only, so
\eqref{eq:vstarmoment} applies directly. By D~\ref{a2}(b), $|a_t - a_{t_k}| \lesssim \delta^{-\kappa_2}(t-t_k)$, whence
\[
  \sum_{k=0}^{N-1}\int_{t_k}^{t_{k+1}} (a_t - a_{t_k})^2\,\E\big\|v_\star(X_{t_k},t_k)\big\|^2 \diff t
  \lesssim D\,\delta^{-2\kappa_1 - 2\kappa_2}\sum_{k=0}^{N-1}\int_{t_k}^{t_{k+1}} (t-t_k)^2 \diff t
  \lesssim D\,\delta^{-2(\kappa_1+\kappa_2)}h^2,
\]
using $\sum_k h^3 \le h^2$. This is \eqref{eq:dis2}. Since $h \le 1$ and
$2(\kappa_1+\kappa_2) \le 10\kappa_1 + 2\kappa_2$, the two bounds combine as stated. \qedhere
\end{proof}

\begin{remark}\label{rem:sigma} \normalfont
The diffusion coefficient cannot be tuned to remove the $O(h)$ term. Lemma~\ref{lem_oracle} carries a
factor $\sigma^{-2}$, while the leading term of \eqref{eq:dis1} carries $\sigma^2$ through the Brownian
increment in \eqref{eq:increment}; the two cancel, and the discretization contribution to
$\operatorname{KL}(P_{1-\delta}\|\hP_{1-\delta}(\hv))$ is of order $D^3\delta^{-c_1}h$, uniformly in
$\sigma$.
\end{remark}

\subsection{Approximation Result}
\approxlem*

\begin{proof}
Fix $\epsilon \in (0,1)$ and $R \ge 1$. By Corollary~\ref{cor14}, for every $t \in [0,1-\delta]$,
\begin{equation}
  \operatorname{Lip}\big(v_\star(\cdot,t)\big) =: L_t \;\lesssim\; L_\delta \asymp \delta^{-3\kappa_1},
  \;
  \sup_{x \in [-R,R]^D}\|v_\star(x,t)\| \;\lesssim\; \delta^{-\kappa_1}\big(R\sqrt{D} + \sqrt{D}\big) \;\lesssim\; \delta^{-\kappa_1} R \sqrt{D}.
  \label{eq:vstarbox}
\end{equation}
Let,
\begin{equation}
  K_0 \;:=\; c\,\delta^{-(3\kappa_1)\vee 1} R \sqrt{D}
  \label{eq:K0}
\end{equation}
with $c$ a large enough absolute constant that $K_0 \ge 1 + \sup_{t \in [0,1-\delta]}\sup_{x \in [-R,R]^D}\|v_\star(x,t)\|_\infty$, which is possible by \eqref{eq:vstarbox}.
Fix $i \in \{0,\dots,N\}$ and $\ell \in [D]$.

Define $g_i^{(\ell)} : [0,1]^D \to \Real$ by
\[
  g_i^{(\ell)}(y) \;=\; \frac{1}{2RL_{t_i}}\Big(v_\star^{(\ell)}\big(2Ry - R\mathbf{1},\, t_i\big) - v_\star^{(\ell)}\big(-R\mathbf{1},\, t_i\big)\Big).
\]
Since $x \mapsto v_\star^{(\ell)}(x,t_i)$ is $L_{t_i}$-Lipschitz, $g_i^{(\ell)}$ is $1$-Lipschitz on $[0,1]^D$ and $g_i^{(\ell)}(0) = 0$. Let $h_i^{(\ell)}$ be the $\varepsilon$-approximating ReLU network for $g_i^{(\ell)}$ furnished by Lemma~\ref{approx_master}, and set
\[
  \tilde{v}^{(\ell)}_{t_i}(x) \;=\; 2RL_{t_i}\, h^{(\ell)}_i\!\left(\frac{x + R\mathbf{1}}{2R}\right) + v_\star^{(\ell)}\big(-R\mathbf{1}, t_i\big).
\]
By construction,
\begin{align}
    \sup_{x \in [-R,R]^D} \big|\tilde{v}^{(\ell)}_{t_i}(x) - v_\star^{(\ell)}(x, t_i)\big|
    = & \; 2 R L_{t_i} \sup_{y \in [0,1]^D} \Big|h^{(\ell)}_i(y) - g^{(\ell)}_i(y)\Big| \nonumber\\
    \le & \; 2R L_{t_i}\varepsilon \nonumber\\
    \lesssim & \; R\, L_\delta\, \varepsilon. \label{j3}
\end{align}
Let $\tilde v_{t_i}$ be obtained by stacking $\{\tilde v^{(\ell)}_{t_i}\}_{\ell=1}^D$. By Lemma~\ref{approx_master}, $\|h^{(\ell)}_i\|_\infty \le \sqrt D$, so on all of $\Real^D$
\(\cW(\tilde{v}_{t_i}) \lesssim D\big(\varepsilon^{-D} \log(1/\varepsilon) + D\big), \;
  \cL(\tilde{v}_{t_i}) \lesssim \log(1/\varepsilon) + D, \;
  \cB(\tilde{v}_{t_i}) \lesssim R L_\delta + \delta^{-\kappa_1}R\sqrt D \lesssim K_0,
\)
and, since $h^{(\ell)}_i$ is $10D$-Lipschitz in $\ell_1$ and the inner affine map has gain $1/(2R)$,
\begin{equation}
  \operatorname{Lip}(\tilde v_{t_i}) \;\le\; 2RL_{t_i} \cdot 10 D \cdot \tfrac{1}{2R}\cdot \sqrt D \;\lesssim\; D^{3/2} L_\delta .
  \label{eq:liptilde}
\end{equation}

We define the clipping function,  $\operatorname{cl}_{K_0}(u) = \relu(u + K_0) - \relu(u - K_0) - K_0$, so that $\operatorname{cl}_{K_0}(u) = u$ for $|u| \le K_0$, $|\operatorname{cl}_{K_0}| \le K_0$ everywhere, and $\operatorname{cl}_{K_0}$ is $1$-Lipschitz; write $\operatorname{cl}_{K_0}(\tilde v_{t_i})$ for its coordinatewise application. By \eqref{eq:vstarbox}, \eqref{eq:K0} and \eqref{j3}, clipping does not alter $\tilde v_{t_i}$ on $[-R,R]^D$ once $\varepsilon \le 1$. Next set
\[
  \chi(x) \;=\; \relu\!\left(1 - \sum_{\ell=1}^D \relu\big(\relu(x_\ell) + \relu(-x_\ell) - R\big)\right) ,
\]
a ReLU network of width $O(D)$ and depth $3$ with $\chi \equiv 1$ on $[-R,R]^D$, $\chi \equiv 0$ off $[-(R+1),R+1]^D$, $0 \le \chi \le 1$, and $\operatorname{Lip}(\chi) \lesssim D$. We define $t_{-1} = -1$ and $t_{N+1} = 1$, let $\delta_i = \frac{1}{2}\big( (t_i - t_{i-1}) \wedge (t_{i+1} - t_i)\big)$ and
\[\xi_{a, b}(x) = \relu\left(\frac{x+a}{a-b}\right) - \relu\left(\frac{x+b}{a-b}\right) - \relu\left(\frac{x-b}{a-b}\right) + \relu\left(\frac{x-a}{a-b}\right),\]
so that $\xi_i := \xi_{\delta_i/2,\,\delta_i/4}(\,\cdot\, - t_i)$ satisfies $0 \le \xi_i \le 1$, $\xi_i \equiv 1$ on $|t - t_i| \le \delta_i/4$, and $\operatorname{supp}\xi_i \subseteq \{|t-t_i| \le \delta_i/2\}$; its weights are of order $1/\delta_i \asymp N$. Since $\delta_i \le \frac12|t_i - t_j|$ for $j \ne i$, the supports of $\{\xi_i\}$ are pairwise disjoint and $\xi_j(t_i) = 0$ for $j \ne i$. Let
\[
  \psi_i(x,t) \;=\; \min\big(\chi(x), \xi_i(t)\big) \;=\; \chi(x) - \relu\big(\chi(x) - \xi_i(t)\big),
\]
which is again a ReLU network, takes values in $[0,1]$, equals $1$ on $[-R,R]^D \times \{|t-t_i|\le \delta_i/4\}$, and vanishes whenever $\chi(x) = 0$ or $\xi_i(t) = 0$.
 
Finally, let $\Pi$ denote the ReLU approximate-multiplication network on $[-K_0,K_0]\times[0,1]$ of accuracy $\epsilon/(4\sqrt D)$, of width and depth $\cO\big(\log(K_0\sqrt D/\epsilon)\big)$, and satisfying $\Pi(u,0) = 0$ identically. We define
\[
  v_0^{(\ell)}(x,t) \;=\; \sum_{i=0}^{N} \Pi\Big( \operatorname{cl}_{K_0}\big(\tilde{v}^{(\ell)}_{t_i}(x)\big),\; \psi_i(x,t) \Big),
  \qquad \ell \in [D].
\]

Fix $i \in \{0,\dots,N\}$ and $x \in [-R,R]^D$. Then $\psi_j(x,t_i) = 0$ for $j \ne i$, so those summands vanish exactly, while $\psi_i(x,t_i) = 1$. Hence, by \eqref{j3},
\[
  \big|v_0^{(\ell)}(x,t_i) - v^{(\ell)}_\star(x,t_i)\big| \;\le\; \frac{\epsilon}{4\sqrt D} + 2RL_{t_i}\varepsilon .
\]
Choosing $\varepsilon = \epsilon\big/\big(8\sqrt{D}\,R L_\delta\big)$ makes the second term at most $\epsilon/(4\sqrt D)$, so that the coordinatewise error is at most $\epsilon/(2\sqrt D)$ and
\[
  \sup_{x\in[-R,R]^D}\big\|v_0(x,t_i) - v_\star(x,t_i)\big\|
  \;\le\; \sqrt D \cdot \frac{\epsilon}{2\sqrt D} \;\le\; \epsilon
  \qquad \text{for every } i = 0,\dots,N.
\]
 
With this $\varepsilon$ we have $\varepsilon^{-D} \lesssim R^D \delta^{-3\kappa_1 D}\epsilon^{-D}$ and $\log(1/\varepsilon) \asymp \log(K_0\sqrt D/\epsilon) \asymp \log(1/\delta) + \log(1/\epsilon) + \log R$, whence, summing the $N+1$ parallel branches and recalling that $D$ is fixed,
\begingroup
\allowdisplaybreaks
\begin{align*}
    \cW(v_0) \lesssim & \; N\Big( \varepsilon^{-D}\log(1/\varepsilon) + \log(K_0\sqrt D/\epsilon) + D\Big)
      \;\lesssim\; N R^D \delta^{-3\kappa_1 D}\epsilon^{-D}\big(\log(1/\delta) + \log(1/\epsilon) + \log R\big),\\
    \cL(v_0) \lesssim & \; \log(1/\varepsilon) + \log(K_0\sqrt D/\epsilon) + D \;\lesssim\; \log(1/\delta) + \log(1/\epsilon) + \log R,\\
    \cB(v_0) \lesssim & \; \max\big\{K_0,\, 1/\delta_i,\, 1\big\} \;\lesssim\; \delta^{-(3\kappa_1)\vee 1} R\sqrt{D} + N,\\
    \cK(v_0) \lesssim & \; K_0 \;\lesssim\; \delta^{-(3\kappa_1)\vee 1} R \sqrt{D} .
\end{align*}
\endgroup
Moreover $v_0(x,t) = 0$ for $x \notin [-(R+1),R+1]^D$, since then $\chi(x) = 0$, hence $\psi_i(x,t) = 0$ for every $i$ and $\Pi(\cdot,0) = 0$. Finally, for fixed $t$ the map $x \mapsto \Pi(u(x),\psi(x))$ has $|\partial_u \Pi| \le 1$ and $|\partial_\psi \Pi| \lesssim K_0$, so by \eqref{eq:liptilde} and $\operatorname{Lip}(\chi) \lesssim D$,
\begin{equation}
  \operatorname{Lip}\big(v_0(\cdot,t)\big) \;\lesssim\; D^{3/2}L_\delta + D K_0 \;\lesssim\; D^{5/2}\, \delta^{-3\kappa_1} R ,
  \qquad t \in [0,1-\delta].
  \label{eq:lipfinal}
\end{equation}
\end{proof}
\section{Proof of the Main Result}
\subsection{Proof of Theorem~\ref{main_hp}}
\label{mainpf}
\mainthm*
\begin{proof}
Throughout, $R \asymp \sqrt{D \log(1/\eta)}$ denotes the truncation radius of
Lemma~\ref{approxlem}, and
$K$ denotes the output bound of the approximating network $v_0$ of that lemma, namely
\[
  K \;\lesssim\; \delta^{-(3\kappa_1) \vee 1} R \sqrt{D}
  \;\asymp\; \delta^{-(3\kappa_1) \vee 1} D \sqrt{\log(1/\eta)} .
\]
We abbreviate
\[
  c_1 \;:=\; 10\kappa_1 + 2\kappa_2 ,
\]
the exponent appearing in Lemma~\ref{lemdis}. The ambient dimension $D$ and the diffusion
coefficient $\sigma$ are held fixed, so factors of $D$ and $\sigma$ are absorbed by $\lesssim$ in the
final rate; we nonetheless track those of $D$ explicitly below. By Remark~\ref{rem:sigma}, the
factor $\sigma^2$ carried by Lemma~\ref{lemdis} cancels against the $\sigma^{-2}$ of
Lemma~\ref{lem_oracle}, so the discretization contribution below is in fact uniform in $\sigma$.
 
\medskip
\textbf{Proof for estimator $\hv_n$}:\\
Since conditioning on $X_{t_k}$ eliminates the cross term in
$\E\|v(X_{t_k},t_k) - \dot X_{t_k}\|^2$ exactly as in the proof of the conditional
flow-matching identity (Lemma~\ref{cftofm}), so
$\widehat{\cL}_n(v) = \sum_{k=0}^{N-1} h_k \E\|v(X_{t_k},t_k) - v_\star(X_{t_k},t_k)\|^2 +
\widehat{\cL}_n(v_\star)$, the second term being independent of $v$. Hence $\hv_n$, the
minimizer of $\widehat{\cL}_n$ over $\cV$, also minimizes the left-hand side below, and
\begingroup
\allowdisplaybreaks
    \begin{align}
     \sum_{k=0}^{N-1} h_k  \E \left\|\hv_n(X_{t_k}, t_k) - v_\star(X_{t_k},t_k) \right\|^2 
      = & \inf_{v \in \cV} \sum_{k=0}^{N-1} h_k  \E \left\|v(X_{t_k}, t_k) - v_\star(X_{t_k},t_k) \right\|^2 \nonumber\\
       \le & \sum_{k=0}^{N-1} h_k  \E \left\|v_0(X_{t_k}, t_k) - v_\star(X_{t_k},t_k) \right\|^2 \nonumber\\
       = & \sum_{k=0}^{N-1} h_k \E \left\|v_0(X_{t_k}, t_k) - v_\star(X_{t_k},t_k) \right\|^2 \one\left( \|X_{t_k}\| \le R\right) \nonumber\\
       & + \sum_{k=0}^{N-1} h_k \E \left\|v_0(X_{t_k}, t_k) - v_\star(X_{t_k},t_k) \right\|^2 \one\left( \|X_{t_k}\| > R\right) \nonumber\\
       \lesssim & \sum_{k=0}^{N-1} h_k \epsilon^2 + \sum_{k=0}^{N-1} h_k \left(D K^2 \eta + \delta^{-2\kappa_1} D^{2} \eta^{1/2} \right) \label{j110}\\
       \lesssim & \ \epsilon^2 + D^2 \delta^{-2\left((3 \kappa_1) \vee 1\right)} \eta^{1/2} \label{j5}.
    \end{align}
    \endgroup
Here \eqref{j110} follows from Lemmata~\ref{approxlem} and \ref{lem12}, bounding the second
summand by $2\E\|v_0\|^2 \one(\cdot) + 2\E\|v_\star\|^2 \one(\cdot)$; the indicator confines
$X_{t_k}$ to the Euclidean ball of radius $R$, which is contained in $[-R,R]^D$, so
Lemma~\ref{approxlem} applies on the first event. Inequality \eqref{j5} uses $\sum_k h_k \le 1$
together with $D K^2 \eta \lesssim D^2\delta^{-2((3\kappa_1)\vee 1)}\log(1/\eta)\,\eta$ and
$\log(1/\eta)\,\eta \le \eta^{1/2}$ for $\eta \le \eta_0$, and $2\kappa_1 \le 2\big((3\kappa_1)\vee 1\big)$.
 
From Lemmata~\ref{lem_oracle} and \ref{lemdis}, we observe that,
    \begingroup
    \allowdisplaybreaks
    \begin{align*}
         \operatorname{KL} \left(P_{1-\delta}\| \hP_{1-\delta}(\hv_n)\right) 
       \lesssim &  \sum_{k=0}^{N-1} \int_{t_k}^{t_{k+1}}  a_t^2\,\E  \|v_\star(X_t,t) -  v_\star(X_{t_k}, t_k)\|^2 \diff t +    \sum_{k=0}^{N-1} \int_{t_k}^{t_{k+1}} (a_t - a_{t_k})^2 \E \|v_\star(X_{t_k}, t_k)\|^2 \diff t \nonumber \\
    & +   \sum_{k=0}^{N-1} h_k a_{t_k}^2 \E  \| v_\star(X_{t_k},t_k) - \hv_n(X_{t_k}, t_k)\|^2 \nonumber\\
    \lesssim & \ D^3 \delta^{-c_1} h +  D\,\delta^{-2( \kappa_1+\kappa_2)} h^2 + \delta^{-\kappa_1} \left(\epsilon^2 + D^2 \delta^{-2\left((3 \kappa_1) \vee 1\right)} \eta^{1/2}\right)\\
     \lesssim & \ D^3 \delta^{-c_1} h + \delta^{-\kappa_1} \left(\epsilon^2 + D^2 \delta^{-2\left((3 \kappa_1) \vee 1\right)} \eta^{1/2}\right),
    \end{align*}
\endgroup
where the last step using $h \le 1$ and $2(\kappa_1+\kappa_2) \le c_1$, so that the second summand is
absorbed by the first.

Hence, from Lemma~\ref{w2},
 \begingroup
 \allowdisplaybreaks
\begin{align}
     \fW_p(\pd, \sT_M(\hP_{1-\delta}(\hv_n))) 
    \lesssim & \ \fW_p(\pd, \pemp) + \delta^{\kappa_3} + M^{-\frac{1}{2p}} +  M \left(\operatorname{KL}\left(P_{1-\delta}, \hP_{1-\delta}(\hv_n)\right)\right)^{\frac{1}{2p}} \nonumber \\
     \lesssim & \ \fW_p(\pd, \pemp) + \delta^{\kappa_3} + M^{-\frac{1}{2p}} \nonumber\\
     & +  M \left(D^3\delta^{-c_1} h + \delta^{-\kappa_1} \left(\epsilon^2 + D^2\delta^{-2\left((3 \kappa_1) \vee 1\right)} \eta^{1/2}\right)\right)^{\frac{1}{2p}} \nonumber \\
     \le & \ \fW_p(\pd, \pemp) + \delta^{\kappa_3} + M^{-\frac{1}{2p}} \nonumber\\
     & +  M \left(D^{\frac{3}{2p}}\delta^{-\frac{c_1}{2p}} h^{\frac{1}{2p}} + \delta^{-\frac{\kappa_1}{2p}} \epsilon^{1/p} + D^{\frac{1}{p}}\delta^{-\frac{\kappa_1}{2p}-\frac{(3 \kappa_1) \vee 1}{p}} \eta^{\frac{1}{4p}}\right)\label{i20},
\end{align}
\endgroup
by the sub-additivity of $u \mapsto u^{1/2p}$.

We choose
\(M = n^{\frac{2p}{d}}, \;
\delta = n^{-\frac{1}{d \kappa_3}}, \;
h \asymp n^{-2p\left( \frac{2p+1}{d} + \frac{c_1}{2dp\kappa_3}\right)}
  = n^{-\frac{2p(2p+1)}{d} - \frac{10\kappa_1+2\kappa_2}{d\kappa_3}}, \;
\epsilon = n^{-p\left( \frac{2p+1}{d} + \frac{\kappa_1}{2dp\kappa_3}\right)},\) and 
\(\eta = n^{-4p \left(\frac{2p+1}{d} + \frac{\kappa_1 + 2\left((3\kappa_1)\vee 1\right)}{2dp\kappa_3}\right)}.\)

Each of the remaining terms of \eqref{i20} is then $O(n^{-1/d})$. 
Hence,
\begin{align*}
    \fW_p(\pd, \sT_M(\hP_{1-\delta}(\hv_n)))  \lesssim \fW_p(\pd, \pemp) + n^{-1/d} \lesssim n^{-1/d}  + n^{-\frac{1}{2p}} \log^{\frac{1}{2p}}(1/\xi),
\end{align*}
with probability at least $1-\xi$, by Lemma~\ref{wcontrol}.
 
\medskip
\textbf{Proof for Estimator $\hvmc$}: \\
Write $\epsilon_{\mathrm{c}}$ for the radius of the $\ell_\infty$-cover of $\cV$ appearing in
Lemma~\ref{mc_1}, to distinguish it from the approximation accuracy $\epsilon$ of
Lemma~\ref{approxlem}. From Lemma~\ref{mc_1} we recall that, with probability at least
$1 - \beta$,
\begin{align*}
        & \sum_{k=0}^{N-1} h_k   \E \|\hvmc(X_{t_k}, t_k) - v_\star(X_{t_k}, t_k)\|^2 \nonumber\\
        \lesssim & \sum_{k=0}^{N-1} h_k \E\|v_0(X_{t_k}, t_k) - v_\star(X_{t_k}, t_k)\|^2
          + D K^2 \sqrt{\frac{\Lambda_\beta}{m}}+  D\,\epsilon_{\mathrm{c}} K,
    \end{align*}
where, $ \Lambda_\beta = \log\frac{N \left( 1+ \cN\left(\epsilon_{\mathrm{c}}; \cV, \ell_\infty\right) \right)}{\beta}.$ We take
\(m \asymp \epsilon_{\mathrm{c}}^{-2} K^{2}\, \Lambda_{\xi'},\)
which makes the middle term $O(D\,\epsilon_{\mathrm{c}} K)$. This ensures that with probability at least $1-\xi^\prime$,
\begingroup
\allowdisplaybreaks
    \begin{align}
        \sum_{j=0}^{N-1} h_j  \E \left\|\hvmc(X_{t_j}, t_j) - v_\star(X_{t_j},t_j) \right\|^2 
       \lesssim & \sum_{j=0}^{N-1} h_j  \E \left\|v_0(X_{t_j}, t_j) - v_\star(X_{t_j},t_j) \right\|^2 + D\,\epsilon_{\mathrm{c}} K \nonumber\\
       = & \sum_{j=0}^{N-1} h_j \E \left\|v_0(X_{t_j}, t_j) - v_\star(X_{t_j},t_j) \right\|^2 \one\left(\|X_{t_j}\| \le R\right) \nonumber\\
       & + \sum_{j=0}^{N-1} h_j \E \left\|v_0(X_{t_j}, t_j) - v_\star(X_{t_j},t_j) \right\|^2 \one\left( \|X_{t_j}\| > R\right)  + D\,\epsilon_{\mathrm{c}} K \nonumber\\
       \lesssim & \sum_{k=0}^{N-1} h_k \epsilon^2 + \sum_{k=0}^{N-1} h_k \left(D K^2 \eta + \delta^{-2\kappa_1} D^{2} \eta^{1/2} \right) + D\,\epsilon_{\mathrm{c}} K \label{j10}\\
       \lesssim & \ \epsilon^2 + D\,\epsilon_{\mathrm{c}} K + D^2\delta^{-2\left((3 \kappa_1) \vee 1\right)} \eta^{1/2}. \label{i19}
    \end{align}
    \endgroup
 From Lemmata~\ref{lem_oracle} and \ref{lemdis}, exactly as above,
    \begin{align*}
       \operatorname{KL} \left(P_{1-\delta}\| \hP_{1-\delta}(\hvmc)\right)
    \lesssim & \ D^3\delta^{-c_1} h  + \delta^{-\kappa_1} \left(\epsilon^2 + D\,\epsilon_{\mathrm{c}} K + D^2 \delta^{-2\left((3 \kappa_1) \vee 1\right)} \eta^{1/2}\right).
    \end{align*}
   Hence, from Lemma~\ref{w2},
   \begingroup
   \allowdisplaybreaks
\begin{align}
    & \fW_p(\pd, \sT_M(\hP_{1-\delta}(\hvmc))) \nonumber \\
    \lesssim & \ \fW_p(\pd, \pemp) + \delta^{\kappa_3} + M^{-\frac{1}{2p}} +  M \left(\operatorname{KL}\left(P_{1-\delta}, \hP_{1-\delta}(\hvmc)\right)\right)^{\frac{1}{2p}} \nonumber \\
     \le & \ \fW_p(\pd, \pemp) + \delta^{\kappa_3} + M^{-\frac{1}{2p}} \nonumber\\
     & +  M \left(D^{\frac{3}{2p}}\delta^{-\frac{c_1}{2p}} h^{\frac{1}{2p}} + \delta^{-\frac{\kappa_1}{2p}}\epsilon^{\frac{1}{p}} + D^{\frac{1}{2p}}\delta^{-\frac{\kappa_1}{2p}} \epsilon_{\mathrm{c}}^{\frac{1}{2p}} K^{\frac{1}{2p}} + D^{\frac{1}{p}}\delta^{-\frac{\kappa_1}{2p}-\frac{(3 \kappa_1) \vee 1}{p}} \eta^{\frac{1}{4p}}\right)\label{i22}.
\end{align}
\endgroup
Keeping $M, \delta, h, \epsilon$ and $\eta$ as above and setting
\(\epsilon_{\mathrm{c}} \;=\; \epsilon^2/K,\)
the third term of \eqref{i22} coincides with the second up to the fixed factor $D^{1/2p}$, and every term is $O(n^{-1/d})$ by the computation above. Hence
\begin{align*}
    \fW_p(\pd, \sT_M(\hP_{1-\delta}(\hvmc)))  \lesssim \fW_p(\pd, \pemp) + n^{-1/d} \lesssim n^{-1/d}  + n^{-\frac{1}{2p}} \log^{\frac{1}{2p}}(1/\xi),
\end{align*}
with probability at least $1-\xi - \xi^\prime$, by Lemma~\ref{wcontrol}. Replacing $\xi$ and $\xi^\prime$ with $\xi/2$ gives the desired result.
 
From Lemma~\ref{approxlem}, $W \lesssim N R^D \delta^{-3\kappa_1 D} \epsilon^{-D}\big(\log(1/\delta)+\log(1/\epsilon)+\log R\big)$, with $N = h^{-1}$.
Substituting the choices above gives
\[
  N = n^{\frac{2p(2p+1)}{d}+\frac{20\kappa_1+4\kappa_2}{2d\kappa_3}}, \qquad
  \delta^{-3\kappa_1D} = n^{\frac{3\kappa_1D}{d\kappa_3}} = n^{\frac{6\kappa_1 D}{2d\kappa_3}}, \qquad
  \epsilon^{-D} = n^{\frac{Dp(2p+1)}{d}+\frac{D\kappa_1}{2d\kappa_3}} .
\]

Since $\eta$ is polynomial in $n$, $R \asymp \sqrt{D\log(1/\eta)} \asymp \sqrt{D \log n}$, so $R^D \asymp (D\log n)^{D/2}$ and $\log(1/\delta)+\log(1/\epsilon)+\log R \asymp \log n$. Summing exponents, the $n$-exponents combine as
$\frac{2p(2p+1)}{d} + \frac{Dp(2p+1)}{d} = \frac{p(2p+1)(D+2)}{d}$ and the $\kappa$-exponents as
$\frac{(20\kappa_1+4\kappa_2) + 6\kappa_1 D + D\kappa_1}{2d\kappa_3} = \frac{(20+7D)\kappa_1+4\kappa_2}{2d\kappa_3}$, so
\[
W \;\lesssim\; n^{\frac{p(2p+1)(D+2)}{d} + \frac{(20+7D)\kappa_1+4\kappa_2}{2d\kappa_3}} \left(D\log n\right)^{D/2}\log n .
\]
 
Further, $L \lesssim \log(1/\delta)+\log(1/\epsilon) + \log R \lesssim \log n$.
 
From Lemma~\ref{approxlem}, $B \lesssim \delta^{-(3\kappa_1)\vee 1}R\sqrt{D} + N$, where $\delta^{-(3\kappa_1)\vee 1} = n^{\frac{(3\kappa_1)\vee 1}{d\kappa_3}}$ and $N$ is as above, so
\(
B \lesssim n^{\left(\frac{2p(2p+1)}{d}+\frac{10\kappa_1+2\kappa_2}{d\kappa_3}\right) \vee \frac{(3\kappa_1)\vee 1}{d\kappa_3}} \, D\sqrt{\log n},
\)
and likewise
\(
K \lesssim \delta^{-(3\kappa_1)\vee 1} R\sqrt{D} \lesssim n^{\frac{(3\kappa_1)\vee 1}{d\kappa_3}}\, D \sqrt{\log n},
\)
the latter unchanged.
 
Finally, by Lemma~\ref{lem_nakada}, $\log \cN(\epsilon_{\mathrm{c}}; \cV, \ell_\infty) \le W \log\big(2 L B^L (W+1)^L/\epsilon_{\mathrm{c}}\big) \lesssim W \log^2 n$, since $\log L, \log B, \log W, \log(1/\epsilon_{\mathrm{c}}) \lesssim \log n$ and $L \lesssim \log n$. With $\epsilon_{\mathrm{c}} = \epsilon^2/K$ we have $\epsilon_{\mathrm{c}}^{-2}K^{2} = \epsilon^{-4}K^{4}$, whence
\begin{align*}
    m & \asymp \epsilon^{-4} K^{4}\Big(W\log^{2} n + \log(1/\xi)\Big) \\
    & \asymp n^{\frac{p(2p+1)(D+6)}{d}+\frac{(24+7D)\kappa_1+4\kappa_2+4\left((6\kappa_1)\vee 2\right)}{2d\kappa_3}}\; \log^{2} n \;\Big((D\log n)^{D/2}\log^{3} n + \log(1/\xi)\Big),
\end{align*}
using $\epsilon^{-4} = n^{\frac{4p(2p+1)}{d}+\frac{4\kappa_1}{2d\kappa_3}}$ and
$K^{4} \asymp n^{\frac{8\left((3\kappa_1)\vee 1\right)}{2d\kappa_3}}D^4\log^2 n$, together with
$8\big((3\kappa_1)\vee 1\big) = 4\big((6\kappa_1)\vee 2\big)$ and
$\frac{4p(2p+1)}{d} + \frac{p(2p+1)(D+2)}{d} = \frac{p(2p+1)(D+6)}{d}$.
\end{proof}

\subsection{Supporting Lemmata}

\begin{lemma}\label{mc_1}
    Let $\epsilon_{\mathrm{c}} \in (0,1)$, let $v_0 \in \cV$ be the network of Lemma~\ref{approxlem}, and write $K$ for the output bound of $\cV \equiv \cR\cN(W,L,B,K)$ and
    \(\Lambda_\beta \;:=\; \log\Big( N\big(1 + \cN(\epsilon_{\mathrm{c}}; \cV, \ell_\infty)\big)\big/\beta\Big).
    \)
    If $m \ge \Lambda_\beta$, then with probability at least $1-\beta$ over the Monte Carlo sampling,
    \begin{align}
        \sum_{k=0}^{N-1} h_k \E\|\hv^{\text{mc}}_{n,m} - v_\star\|^2
        \;\lesssim\; \sum_{k=0}^{N-1} h_k \E\|v_0 - v_\star\|^2
        + D K^2 \sqrt{\frac{\Lambda_\beta}{m}} + D\,\epsilon_{\mathrm{c}} K .
    \end{align}
\end{lemma}

\begin{proof}
All expectations are conditional on $x_1,\dots,x_n$. Fix $k \in \{0,1,\dots,N-1\}$ and $v \in \cV$. Since the Monte Carlo sample is i.i.d.,
\begin{align*}
    & \E \left(\frac{1}{m}\sum_{j=1}^m \left\|v(X_{t_k}^{(j)}, t_k) - (\da_{t_k} X^{(j)}_1 + \db_{t_k} X_0^{(j)})\right\|^2 \right)
    =   \E \left\|v(X_{t_k}, t_k) - (\da_{t_k} X_1 + \db_{t_k} X_0)\right\|^2\\
    = & \frac{1}{n} \sum_{i=1}^n \E_{X_0} \|v( X_{i,t_k}, t_k) - (\da_{t_k} x_i + \db_{t_k} X_0)\|^2,
\end{align*}
where $X_{i,t} = \alpha_t x_i + \beta_t X_0$ and $X_0 \sim P_0$. By Bernstein's inequality for sub-exponential variables \citep[Theorem 2.9.1]{Vershynin_2026}, for any fixed $k$ and $v$, with probability at least $1-\beta$,
\begin{align}
    & \left|\frac{1}{m}\sum_{j=1}^m \left\|v(X_{t_k}^{(j)}, t_k) - (\da_{t_k} X^{(j)}_1 + \db_{t_k} X_0^{(j)})\right\|^2  - \E\|v(X_{t_k}, t_k) - (\da_{t_k}X_1 + \db_{t_k} X_0)\|^2 \right| \nonumber \\
    \lesssim & \left(\left\| \|v(X_{t_k}, t_k)\|^2 \right\|_{\psi_1}  + \left\| \|\da_{t_k}X_1 + \db_{t_k} X_0\|^2 \right\|_{\psi_1} \right)\left(\sqrt{\frac{ \log(1/\beta)}{m}} \;\vee\; \frac{\log(1/\beta)}{m}\right) \label{j8} .
\end{align}

Since $\cK(v) \le K$ bounds $\|v(\cdot,\cdot)\|_\infty$ coordinatewise, $\|v(X_{t_k},t_k)\|^2 \le DK^2$ almost surely, so
\[
   \left\| \|v(X_{t_k}, t_k)\|^2 \right\|_{\psi_1} \;\lesssim\; D K^2 .
\]

Moreover, by D\ref{a2}(a), $|\da_t| \vee |\db_t| \lesssim \delta^{-\kappa_1}$ for $t \in [0,1-\delta]$, so $\big\|\|\da_{t_k}X_1 + \db_{t_k}X_0\|\big\|_{\psi_2} \lesssim \delta^{-\kappa_1}\sqrt D$ and hence
\[
  \left\| \|\da_{t_k}X_1 + \db_{t_k} X_0\|^2 \right\|_{\psi_1} \;=\; \big\|\|\da_{t_k}X_1 + \db_{t_k}X_0\|\big\|_{\psi_2}^2 \;\lesssim\; D\delta^{-2\kappa_1} \;\lesssim\; DK^2,
\]
the last step because $K \gtrsim \delta^{-(3\kappa_1)\vee 1} \ge \delta^{-\kappa_1}$.

Hence from \eqref{j8}, for fixed $k$ and $v$, with probability at least $1-\beta$,
\begin{align}
   & \left|\frac{1}{m}\sum_{j=1}^m \left\|v(X_{t_k}^{(j)}, t_k) - (\da_{t_k} X^{(j)}_1 + \db_{t_k} X_0^{(j)})\right\|^2  - \E\|v(X_{t_k}, t_k) - (\da_{t_k} X_1 + \db_{t_k} X_0)\|^2 \right| \nonumber\\
    \lesssim & \; D K^2 \sqrt{\frac{ \log(1/\beta)}{m}}. \label{j9}
\end{align}

Let $\cC\left(\epsilon_{\mathrm{c}}; \cV, \ell_\infty\right)$ be an $\epsilon_{\mathrm{c}}$-cover of $\cV$ in the $\ell_\infty$ norm, of cardinality $\cN(\epsilon_{\mathrm{c}}; \cV, \ell_\infty)$.
Applying \eqref{j9} with $\beta$ replaced by $\beta / \big(N(1 + \cN(\epsilon_{\mathrm{c}};\cV,\ell_\infty))\big)$ and taking a union bound over the $N$ time indices and the $\cN(\epsilon_{\mathrm{c}};\cV,\ell_\infty) + 1$ networks in $\cC(\epsilon_{\mathrm{c}};\cV,\ell_\infty) \cup \{v_0\}$, we obtain: with probability at least $1-\beta$, simultaneously for all $v \in \cC\left(\epsilon_{\mathrm{c}}; \cV, \ell_\infty\right) \cup \{v_0\}$,
\begin{align}
   & \sum_{k=0}^{N-1} h_k \left|\frac{1}{m}\sum_{j=1}^m \left\|v(X_{t_k}^{(j)}, t_k) - (\da_{t_k} X^{(j)}_1 + \db_{t_k} X_0^{(j)})\right\|^2  - \E \|v(X_{t_k}, t_k) - (\da_{t_k} X_1 + \db_{t_k} X_0)\|^2 \right| \nonumber\\
 \lesssim & \; D K^2 \sqrt{\frac{ \Lambda_\beta}{m}}, \label{j9b}
\end{align}
using $\sum_k h_k \le 1$. We work on this event for the remainder of the proof.

Let $\tilde{v} \in \cC(\epsilon_{\mathrm{c}}; \cV, \ell_\infty)$ be such that $\|\hv^{\text{mc}}_{n,m} - \tilde{v}\|_\infty \le \epsilon_{\mathrm{c}}$. Using $\|a\|^2 - \|b\|^2 = \langle a-b, a+b \rangle$, for any $k$ and any $x$,
\begin{align}
& \left| \|\hv^{\text{mc}}_{n,m}(x, t_k) - (\da_{t_k} X_1 + \db_{t_k} X_0)\|^2
- \|\tilde{v}(x, t_k) - (\da_{t_k} X_1 + \db_{t_k} X_0)\|^2 \right| \nonumber\\
\le \; & \sqrt{D}\,\|\hv^{\text{mc}}_{n,m} - \tilde{v}\|_\infty
\left( \|\hv^{\text{mc}}_{n,m}\|_\infty + \|\tilde{v}\|_\infty
+ 2\|\da_{t_k} X_1 + \db_{t_k} X_0\| \right) \nonumber\\
\le \; & 2\sqrt{D}\,\epsilon_{\mathrm{c}} \left( \sqrt{D} K + \|\da_{t_k} X_1 + \db_{t_k} X_0\| \right).
\label{eq:pert}
\end{align}
Taking expectations in \eqref{eq:pert} and using $\E\|\da_{t_k}X_1 + \db_{t_k}X_0\| \lesssim \delta^{-\kappa_1}\sqrt{D} \lesssim \sqrt D K$ gives, for every $k$,
\[
  \Big| \E\|\hv^{\text{mc}}_{n,m}(X_{t_k},t_k) - (\da_{t_k}X_1 + \db_{t_k}X_0)\|^2 - \E\|\tilde v(X_{t_k},t_k) - (\da_{t_k}X_1 + \db_{t_k}X_0)\|^2 \Big| \;\lesssim\; D\,\epsilon_{\mathrm{c}} K,
\]
and the same bound holds with $\E$ replaced by the empirical average $\frac1m\sum_{j=1}^m$, since \eqref{eq:pert} is pointwise in the sample. Chaining these with \eqref{j9b} applied first to $\tilde v$ and then to $v_0$, and using that $\hv^{\text{mc}}_{n,m}$ minimises the Monte Carlo objective over $\cV \ni v_0$,
\begingroup
\allowdisplaybreaks
\begin{align*}
& \sum_{k=0}^{N-1} h_k \E \|\hv^{\text{mc}}_{n,m}(X_{t_k}, t_k) - (\da_{t_k} X_1 + \db_{t_k} X_0)\|^2 \\
\le \; & \sum_{k=0}^{N-1} h_k \E \|\tilde{v}(X_{t_k}, t_k) - (\da_{t_k} X_1 + \db_{t_k} X_0)\|^2 + \cO\big(D\epsilon_{\mathrm{c}} K\big) \\
\le \; & \frac{1}{m}\sum_{j=1}^m \sum_{k=0}^{N-1} h_k \|\tilde{v}(X_{t_k}^{(j)}, t_k) - (\da_{t_k}X_1^{(j)} + \db_{t_k}X_0^{(j)})\|^2 + \cO\left(DK^2\sqrt{\frac{\Lambda_\beta}{m}}\right) + \cO\big(D\epsilon_{\mathrm{c}} K\big) \\
\le \; & \frac{1}{m}\sum_{j=1}^m \sum_{k=0}^{N-1} h_k \|\hv^{\text{mc}}_{n,m}(X_{t_k}^{(j)}, t_k) - (\da_{t_k}X_1^{(j)} + \db_{t_k}X_0^{(j)})\|^2 + \cO\left(DK^2\sqrt{\frac{\Lambda_\beta}{m}}\right) + \cO\big(D\epsilon_{\mathrm{c}} K\big) \\
\le \; & \frac{1}{m}\sum_{j=1}^m \sum_{k=0}^{N-1} h_k \|v_0(X_{t_k}^{(j)}, t_k) - (\da_{t_k}X_1^{(j)} + \db_{t_k}X_0^{(j)})\|^2 + \cO\left(DK^2\sqrt{\frac{\Lambda_\beta}{m}}\right) + \cO\big(D\epsilon_{\mathrm{c}} K\big) \\
\le \; & \sum_{k=0}^{N-1} h_k \E \|v_0(X_{t_k}, t_k) - (\da_{t_k}X_1 + \db_{t_k}X_0)\|^2 + \cO\left(DK^2\sqrt{\frac{\Lambda_\beta}{m}}\right) + \cO\big(D\epsilon_{\mathrm{c}} K\big).
\end{align*}
\endgroup
Finally, conditioning on $X_{t_k}$ makes the cross term vanish, so that for every $v$,
\begin{align*}
    &  \sum_{k=0}^{N-1} h_k \E\|v(X_{t_k},t_k) - (\da_{t_k}X_1 + \db_{t_k}X_0)\|^2 \\
  = & \sum_{k=0}^{N-1} h_k \E\|v(X_{t_k},t_k) - v_\star(X_{t_k},t_k)\|^2
  + \sum_{k=0}^{N-1} h_k \E\|v_\star(X_{t_k},t_k) - (\da_{t_k}X_1 + \db_{t_k}X_0)\|^2,
\end{align*}
this being the discrete analogue of Lemma~\ref{cftofm}. Applying it to $v = \hv^{\text{mc}}_{n,m}$ and to $v = v_0$ and subtracting the common second term, which does not depend on $v$, yields
\[
\sum_{k=0}^{N-1} h_k \E\|\hv^{\text{mc}}_{n,m} - v_\star\|^2
\lesssim \sum_{k=0}^{N-1} h_k \E\|v_0 - v_\star\|^2
+ D K^2\sqrt{\frac{\Lambda_\beta}{m}} + D\epsilon_{\mathrm{c}} K . \qedhere
\]
\end{proof}
\section{Supporting Results}
\subsection{Results for the Velocity Field}
\begin{lemma}[Spatial Derivative of $v_\star$]\label{lemlip} The following holds for any $x \in \Real^D$ and $t>0$,
\begin{align}
    v_\star(x,t) = & \frac{\db_t}{\beta_t} x + \frac{\da_t \beta_t - \db_t \alpha_t}{\beta_t} \E(X_1|X_t = x) =\frac{\da_t}{ \alpha_t} x + \beta_t^2 \left(\frac{\da_t}{\alpha_t} - \frac{\db_t}{\beta_t}\right) \nabla \log \varrho_t(x), \label{i1}\\
    \nabla v_\star(x,t) = & \frac{\alpha_t^2}{\beta_t^2} \left(\frac{\da_t}{\alpha_t} - \frac{\db_t}{\beta_t}\right) \operatorname{Var}(X_1|X_t = x) + \frac{\db_t}{\beta_t} I_D = \frac{1}{2} \partial_t \left(\frac{\alpha_t^2}{\beta_t^2}\right)\operatorname{Var}(X_1|X_t = x) + \frac{\db_t}{\beta_t} I_D \label{i2}.
\end{align}
\end{lemma}
\begin{proof}
Throughout, fix $t\in(0,1)$ with $\alpha_t\neq 0,\ \beta_t>0$, and recall the interpolant is given by,
\(X_t = \alpha_t X_1 + \beta_t X_0, \, X_1\sim\pemp,\, X_0\sim N(0,I_D)\ \text{independent of } X_1.\)
Hence, $\varrho_t$, the law of $X_t$, is the Gaussian smoothing
\[
\varrho_t(x) = \int_{\Real^D} \phi(x;\alpha_t y,\beta_t^2 I_D)\,\pemp(dy),
\qquad
\]
where, $\phi(x;\mu,\Sigma) := (2\pi)^{-D/2}\det(\Sigma)^{-1/2}\exp\!\Big(-\tfrac12(x-\mu)^\top \Sigma^{-1}(x-\mu)\Big).$
By definition, $v_\star(x,t) = \E[\dot X_t \mid X_t = x]$. For notational simplicity, we write
\[
m(x,t):=\E[X_1\mid X_t=x]\in\Real^D,\qquad
\Sigma(x,t):=\operatorname{Var}(X_1\mid X_t=x)\in\Real^{D\times D}.
\]
Since $\beta_t>0$, we may write $X_0 = \beta_t^{-1}(X_t-\alpha_t X_1)$. Hence
\[
v_\star(x,t) = \E[\da_t X_1 + \db_t X_0 \mid X_t=x]
= \da_t\, m(x,t) + \frac{\db_t}{\beta_t}\big(x - \alpha_t m(x,t)\big),
\]
using $\E[X_t\mid X_t=x]=x$. Collecting terms in $m(x,t)$ gives
\begin{equation}
v_\star(x,t) = \frac{\db_t}{\beta_t}\,x + \frac{\da_t\beta_t-\db_t\alpha_t}{\beta_t}\, m(x,t).
\end{equation}

To prove equation \eqref{i2}, we note that, $\nabla_x \phi(x;\alpha_t y,\beta_t^2 I_D) = -\beta_t^{-2}(x-\alpha_t y)\,\phi(x;\alpha_t y,\beta_t^2 I_D)$, differentiation under the integral sign (justified since $\beta_t>0$ and $\pemp$ has finite second moments) gives
\[
\nabla \varrho_t(x) = -\frac{1}{\beta_t^2}\int (x-\alpha_t y)\,\phi(x;\alpha_t y,\beta_t^2 I_D)\,\pemp(dy).
\]
Dividing by $\varrho_t(x)>0$ and using that $\phi(x;\alpha_t y,\beta_t^2I_D)\pemp(dy)/\varrho_t(x)$ is exactly the conditional law of $X_1$ given $X_t=x$,
\begin{equation}
\nabla\log\varrho_t(x) = -\frac{1}{\beta_t^2}\,\E[x-\alpha_t X_1\mid X_t=x] = -\frac{x-\alpha_t m(x,t)}{\beta_t^2}.
\label{i3}
\end{equation}
Solving for $m(x,t)$ (using $\alpha_t\neq0$):
\begin{equation}
m(x,t) = \frac{x+\beta_t^2\nabla\log\varrho_t(x)}{\alpha_t}.
\label{i4}
\end{equation}
Substituting \eqref{i4} into \eqref{i1},
\[
v_\star(x,t) = \frac{\db_t}{\beta_t}x
+ \frac{\da_t\beta_t-\db_t\alpha_t}{\alpha_t\beta_t}\,x
+ \frac{(\da_t\beta_t-\db_t\alpha_t)\beta_t}{\alpha_t}\,\nabla\log\varrho_t(x).
\]
The coefficient of $x$ simplifies as
\(\frac{\db_t}{\beta_t}+\frac{\da_t\beta_t-\db_t\alpha_t}{\alpha_t\beta_t}
=\frac{\db_t\alpha_t+\da_t\beta_t-\db_t\alpha_t}{\alpha_t\beta_t}
=\frac{\da_t}{\alpha_t},
\)
and the coefficient of $\nabla\log\varrho_t(x)$ as
\[
\frac{\beta_t(\da_t\beta_t-\db_t\alpha_t)}{\alpha_t}
=\beta_t^2\cdot\frac{\da_t\beta_t-\db_t\alpha_t}{\alpha_t\beta_t}
=\beta_t^2\left(\frac{\da_t}{\alpha_t}-\frac{\db_t}{\beta_t}\right).
\]
Hence
\begin{equation}
v_\star(x,t) = \frac{\da_t}{\alpha_t}x + \beta_t^2\left(\frac{\da_t}{\alpha_t}-\frac{\db_t}{\beta_t}\right)\nabla\log\varrho_t(x),
\label{i5}
\end{equation}
which together with \eqref{i1} establishes the first display of the lemma.

To complete the proof, we make the following claim,
\begin{equation}
\nabla_x m(x,t) = \frac{\alpha_t}{\beta_t^2}\,\Sigma(x,t).
\label{i6}
\end{equation}
Write $N_i(x):=\int y_i\,\phi(x;\alpha_t y,\beta_t^2I_D)\,\pemp(dy)$, so that $m_i(x,t)=N_i(x)/\varrho_t(x)$ for $i=1,\dots,D$. Differentiating under the integral sign as before,
\begin{align*}
    \partial_{x_j}N_i(x)
= & -\frac{1}{\beta_t^2}\int y_i(x_j-\alpha_t y_j)\,\phi(x;\alpha_ty,\beta_t^2I_D)\,\pemp(dy)\\
= & -\frac{x_j}{\beta_t^2}N_i(x) + \frac{\alpha_t}{\beta_t^2}\,\varrho_t(x)\,\E[X_{1,i}X_{1,j}\mid X_t=x].
\end{align*}
Dividing by $\varrho_t(x)$,
\(\frac{\partial_{x_j}N_i(x)}{\varrho_t(x)} = -\frac{x_j}{\beta_t^2}\,m_i(x,t) + \frac{\alpha_t}{\beta_t^2}\,\E[X_{1,i}X_{1,j}\mid X_t=x].\)
By the quotient rule, $\partial_{x_j}m_i(x,t) = \dfrac{\partial_{x_j}N_i(x)}{\varrho_t(x)} - m_i(x,t)\,\partial_{x_j}\log\varrho_t(x)$. Substituting the previous display and (2) (i.e. $\partial_{x_j}\log\varrho_t(x) = -\beta_t^{-2}(x_j-\alpha_t m_j(x,t))$),
\[
\partial_{x_j}m_i(x,t)
= \frac{\alpha_t}{\beta_t^2}\Big(\E[X_{1,i}X_{1,j}\mid X_t=x] - m_i(x,t)m_j(x,t)\Big)
= \frac{\alpha_t}{\beta_t^2}\,\Sigma_{ij}(x,t),
\]
which is \eqref{i6} in coordinates.

Differentiate \eqref{i1} with respect to $x$ and apply \eqref{i6}:
\[
\nabla v_\star(x,t) = \frac{\db_t}{\beta_t}I_D + \frac{\da_t\beta_t-\db_t\alpha_t}{\beta_t}\,\nabla_x m(x,t)
= \frac{\db_t}{\beta_t}I_D + \frac{\alpha_t(\da_t\beta_t-\db_t\alpha_t)}{\beta_t^3}\,\Sigma(x,t).
\]
Since
\(\frac{\alpha_t(\da_t\beta_t-\db_t\alpha_t)}{\beta_t^3}
= \frac{\alpha_t^2}{\beta_t^2}\cdot\frac{\da_t\beta_t-\db_t\alpha_t}{\alpha_t\beta_t}
= \frac{\alpha_t^2}{\beta_t^2}\left(\frac{\da_t}{\alpha_t}-\frac{\db_t}{\beta_t}\right),\)
we obtain
\begin{equation}
\nabla v_\star(x,t) = \frac{\alpha_t^2}{\beta_t^2}\left(\frac{\da_t}{\alpha_t}-\frac{\db_t}{\beta_t}\right)\operatorname{Var}(X_1\mid X_t=x) + \frac{\db_t}{\beta_t}I_D.
\label{i7}
\end{equation}
By the quotient rule,
\(\partial_t\!\left(\frac{\alpha_t^2}{\beta_t^2}\right)
= \frac{2\alpha_t\da_t\beta_t^2 - 2\alpha_t^2\beta_t\db_t}{\beta_t^4}
= \frac{2\alpha_t\beta_t(\da_t\beta_t-\alpha_t\db_t)}{\beta_t^4}
= 2\,\frac{\alpha_t^2}{\beta_t^2}\left(\frac{\da_t}{\alpha_t}-\frac{\db_t}{\beta_t}\right).
\)
Hence
\(\frac12\,\partial_t\!\left(\frac{\alpha_t^2}{\beta_t^2}\right) = \frac{\alpha_t^2}{\beta_t^2}\left(\frac{\da_t}{\alpha_t}-\frac{\db_t}{\beta_t}\right),
\)
and substituting into \eqref{i7} gives
\[
\nabla v_\star(x,t) = \frac12\,\partial_t\!\left(\frac{\alpha_t^2}{\beta_t^2}\right)\operatorname{Var}(X_1\mid X_t=x) + \frac{\db_t}{\beta_t}I_D.
\]
\end{proof}
\begin{cor}\label{cor14}
    For any $t \in [0, 1-\delta]$,
    \begin{align}
         \|v_\star(x,t)\| \lesssim & \delta^{-\kappa_1} \|x\| + \delta^{-\kappa_1} \sqrt{D}, \label{i8}\\
         \| \nabla v_\star(x,t) \| \lesssim & D \delta^{-3\kappa_1}. \label{i9}
    \end{align}
\end{cor}
\begin{proof}
    By construction, for any $t \in [0, 1-\delta]$,
    \begin{align*}
        \|v_\star(x,t)\| = & \left\|\frac{\db_t}{\beta_t} x + \frac{\da_t \beta_t - \db_t \alpha_t}{\beta_t} \E(X_1|X_t = x)\right\| \\
        \lesssim & \delta^{-\kappa_1} \|x\| + \delta^{-\kappa_1} \sqrt{D}.
    \end{align*}
    Similarly,
    \begin{align*}
       \| \nabla v_\star(x,t) \|= & \left\|\frac{\alpha_t^2}{\beta_t^2} \left(\frac{\da_t}{\alpha_t} - \frac{\db_t}{\beta_t}\right) \operatorname{Var}(X_1|X_t = x) + \frac{\db_t}{\beta_t} I_D \right\|\\
       \lesssim & \frac{D}{\beta_t^2}\left|\frac{\da_t \beta_t - \db_t \alpha_t}{\beta_t}\right| + \left|\frac{\db_t}{\beta_t}\right|\\
       \lesssim & D \delta^{-3\kappa_1}.
    \end{align*}
\end{proof}
\begin{lemma}\label{lem12}
There exists an absolute constant $C$ such that, for every $\eta \le \eta_0$, writing $R = C\sqrt{D\log(1/\eta)}$, we have for every $t \in [0,1-\delta]$
\[
  \prob\big(\|X_t\| > R\big) \;\le\; \eta
  \qquad\text{and}\qquad
  \E \|v_\star(X_t, t)\|^2 \one\big( \|X_t\| > R\big) \;\lesssim\; \delta^{-2\kappa_1} D^{2} \eta^{1/2}.
\]
\end{lemma}
\begin{proof}
Fix $t \in [0,1-\delta]$. By Theorem~\ref{thm:equivalence}, $\mathrm{Law}(X_t) = \varrho_t = \mathrm{Law}(\alpha_t X_1 + \beta_t X_0)$, a statement about the one-time marginal, which is all that is used below. Since $X_1 \sim \pemp$ is supported in $[0,1]^D$ by Assumption~\ref{a1} we have $\|X_1\| \le \sqrt D$, and since $\alpha_t, \beta_t \in [0,1]$,
\begin{equation}
  \|X_t\| \;\le\; \alpha_t\|X_1\| + \beta_t\|X_0\| \;\le\; \sqrt D + \|X_0\|,
  \qquad X_0 \sim \cN(0,I_D).
  \label{eq:xtsg}
\end{equation}
Hence $\big\|\,\|X_t\|\,\big\|_{\psi_2} \le \sqrt D + \big\|\,\|X_0\|\,\big\|_{\psi_2} \lesssim \sqrt D$, and by Proposition~2.6.1 of \citet{Vershynin_2026}, $\E\|X_t\|^4 \lesssim D^2$. Moreover, by \eqref{eq:xtsg} and Gaussian concentration of the norm, $\prob\big(\|X_t\| \ge 2\sqrt D + s\big) \le \prob\big(\|X_0\| \ge \sqrt D + s\big) \le e^{-s^2/2}$ for every $s \ge 0$; taking $s = \sqrt{2\log(1/\eta)}$ gives the first claim, for a suitable absolute constant $C$ and $\eta \le \eta_0$.

For the second claim, by Corollary~\ref{cor14} and then the Cauchy--Schwarz inequality,
\begingroup
\allowdisplaybreaks
\begin{align}
  & \E \|v_\star(X_t, t)\|^2 \one\big( \|X_t\| > R\big) \nonumber\\
  \lesssim\; & \delta^{-2\kappa_1}\,\E \big(\|X_t\|^2 + D\big)\, \one\big( \|X_t\| > R\big) \label{sep1}\\
  \le\; & \delta^{-2\kappa_1} \Big(\E \big(\|X_t\|^2 + D\big)^2\Big)^{1/2} \Big(\prob\big(\|X_t\| > R\big)\Big)^{1/2} \nonumber\\
  \lesssim\; & \delta^{-2\kappa_1} \big(\E \|X_t\|^4 + D^2\big)^{1/2}\, \eta^{1/2} \nonumber\\
  \lesssim\; & \delta^{-2\kappa_1} D\, \eta^{1/2} \;\le\; \delta^{-2\kappa_1} D^{2}\, \eta^{1/2}. \nonumber
\end{align}
\endgroup
In the above calculations, \eqref{sep1} follows from Corollary~\ref{cor14}.
\end{proof}

\begin{lemma}[Temporal velocity of $v_\star$]\label{lip_general}
Fix $t\in(0,1)$ with $\alpha_t \neq 0$, $\beta_t > 0$, and suppose $\alpha_t,\beta_t$ are twice continuously differentiable in a neighborhood of $t$. Define
\[
u_t := \frac{\dot\alpha_t}{\alpha_t}, \qquad w_t := \frac{\dot\beta_t}{\beta_t}, \qquad \kappa_t := u_t - w_t = \frac{\dot\alpha_t}{\alpha_t} - \frac{\dot\beta_t}{\beta_t}.
\]
Then, with $E_{1,t}(x) = \E(X_1 \mid X_t = x)$, $E_{2,t}(x) = \E(\|X_1\|^2 \mid X_t = x)$, $E_{3,t}(x) = \E(\|X_1\|^2 X_1 \mid X_t = x)$, and $V_{1,t}(x) = \operatorname{Var}(X_1 \mid X_t = x)$ as in Lemma~\ref{lemlip},
\begin{align}
\partial_t v_\star(x,t) &= \dot w_t\, x \;-\; \frac{\alpha_t^2 \kappa_t (2w_t - u_t)}{\beta_t^2}\, V_{1,t}(x)\, x \;+\; \alpha_t\big(\dot u_t - \dot w_t + u_t \kappa_t\big)\, E_{1,t}(x) \nonumber\\
&\qquad - \frac{\alpha_t^3 \kappa_t^2}{\beta_t^2}\big(E_{3,t}(x) - E_{1,t}(x) E_{2,t}(x)\big). \label{eq:gen_dt_v}
\end{align}
Further, for any $t \in [0, 1 -\delta]$ and $x \in \Real^D$, $\|\partial_t v_\star(x,t)\| \lesssim \delta^{-\kappa_1} (1 + \|x\|)$.
\end{lemma}

\begin{proof}
 Throughout, write $E_1,E_2,E_3,V_1$ for $E_{1,t}(x),E_{2,t}(x),E_{3,t}(x),V_{1,t}(x)$, and $g(x) := \nabla_x \log \varrho_t(x)$.

From Lemma~\ref{lem:gradEf}, taking $f(y) = y$ recovers $\nabla_x E_1 = \frac{\alpha_t}{\beta_t^2}V_1$; taking $f(y) = \|y\|^2$ gives $\nabla_x E_2 = \frac{\alpha_t}{\beta_t^2}(E_3 - E_2 E_1)$. Since $\operatorname{tr}(V_1) = E_2 - \|E_1\|^2$,
\begin{equation}\label{eq:gradtrV}
\nabla_x \operatorname{tr}\big(V_{1,t}(x)\big) = \nabla_x E_2 - 2 (\nabla_x E_1)^\top E_1 = \frac{\alpha_t}{\beta_t^2}\big(E_3 - E_2 E_1 - 2 V_1 E_1\big).
\end{equation}

The continuity equation $\partial_t \varrho_t + \nabla_x \cdot (\varrho_t v_\star) = 0$ gives, for fixed $x$,
\begin{equation}\label{eq:dtg}
\partial_t g(x) = -\Big[\nabla_x^2 \log \varrho_t(x)\, v_\star(x,t) + \nabla_x v_\star(x,t)\, g(x) + \nabla_x\big(\nabla_x \cdot v_\star(x,t)\big)\Big].
\end{equation}
Differentiating $g(x) = -\frac{x - \alpha_t E_1}{\beta_t^2}$ in $x$ and using $\nabla_x E_1 = \frac{\alpha_t}{\beta_t^2}V_1$,
\begin{equation}\label{eq:hess}
\nabla_x^2 \log \varrho_t(x) = \frac{\alpha_t^2}{\beta_t^4}V_1 - \frac{1}{\beta_t^2}I_D =: H.
\end{equation}

Using \eqref{i5}, \eqref{i2}, and \eqref{eq:hess},
\begingroup
\allowdisplaybreaks
\begin{align*}
H\, v_\star(x,t) &= \left(\frac{\alpha_t^2}{\beta_t^4}V_1 - \frac{1}{\beta_t^2}I_D\right)\big(u_t x + q_t g(x)\big) = \frac{\alpha_t^2 u_t}{\beta_t^4}V_1 x + \frac{\alpha_t^2 q_t}{\beta_t^4}V_1 g(x) - \frac{u_t}{\beta_t^2}x - \frac{q_t}{\beta_t^2}g(x),\\
\nabla_x v_\star(x,t)\, g(x) &= \left(\frac{\alpha_t^2}{\beta_t^2}\kappa_t V_1 + w_t I_D\right) g(x) = \frac{\alpha_t^2 \kappa_t}{\beta_t^2}V_1 g(x) + w_t g(x).
\end{align*}
\endgroup
Since $q_t = \beta_t^2 \kappa_t$, the coefficient of $V_1 g(x)$ in the sum is $\frac{\alpha_t^2 q_t}{\beta_t^4} + \frac{\alpha_t^2\kappa_t}{\beta_t^2} = 2\frac{\alpha_t^2\kappa_t}{\beta_t^2}$, and the coefficient of $g(x)$ is $w_t - \frac{q_t}{\beta_t^2} = w_t - \kappa_t = 2w_t - u_t$. Substituting $g(x) = -\frac{x-\alpha_tE_1}{\beta_t^2}$ throughout and collecting terms by type ($V_1x$, $x$, $V_1E_1$, $E_1$),
\begin{align}
H v_\star(x,t) + \nabla_xv_\star(x,t)g(x) = \frac{\alpha_t^2(2w_t-u_t)}{\beta_t^4}V_1x - \frac{2w_t}{\beta_t^2}x + \frac{2\alpha_t^3\kappa_t}{\beta_t^4}V_1E_1 + \frac{\alpha_t(2w_t-u_t)}{\beta_t^2}E_1. \label{eq:AB}
\end{align}

From \eqref{i2}, $\nabla_x \cdot v_\star(x,t) = \frac{\alpha_t^2}{\beta_t^2}\kappa_t\operatorname{tr}(V_1) + Dw_t$, so by \eqref{eq:gradtrV},
\begin{equation}\label{eq:div}
\nabla_x\big(\nabla_x \cdot v_\star(x,t)\big) = \frac{\alpha_t^3\kappa_t}{\beta_t^4}\big(E_3 - E_2E_1 - 2V_1E_1\big).
\end{equation}
Adding \eqref{eq:div} to \eqref{eq:AB}, the $V_1E_1$ contributions cancel exactly ($\frac{2\alpha_t^3\kappa_t}{\beta_t^4} - \frac{2\alpha_t^3\kappa_t}{\beta_t^4}=0$), leaving
\[
\partial_t g(x) = -\frac{\alpha_t^2(2w_t-u_t)}{\beta_t^4}V_1x + \frac{2w_t}{\beta_t^2}x - \frac{\alpha_t(2w_t-u_t)}{\beta_t^2}E_1 - \frac{\alpha_t^3\kappa_t}{\beta_t^4}(E_3-E_2E_1).
\]

 By the product rule,
\[
\partial_t v_\star(x,t) = \dot u_t\, x + \dot q_t\, g(x) + q_t\, \partial_t g(x), \qquad \dot q_t = \beta_t^2\big(2w_t\kappa_t + \dot\kappa_t\big), \quad \dot\kappa_t = \dot u_t - \dot w_t.
\]
Substituting $g(x) = -\frac{x-\alpha_tE_1}{\beta_t^2}$ and $q_t\partial_tg(x)$, then collecting terms of each type ($x$; $V_1x$; $E_1$; $E_3-E_2E_1$), we get the final result. Finally, 
    \begin{align*}
\|\partial_t v_\star(x,t)\| \le &   |\dot w_t| \| x\| + \left|\frac{\alpha_t^2 \kappa_t (2w_t - u_t)}{\beta_t^2}\right|  \left\|V_{1,t}(x)\right\| \|x\|+ \left|\alpha_t\big(\dot u_t - \dot w_t + u_t \kappa_t\big)\right| \|E_{1,t}(x)\| \nonumber\\
&\qquad + \left| \frac{\alpha_t^3 \kappa_t^2}{\beta_t^2} \right| \|E_{3,t}(x) - E_{1,t}(x) E_{2,t}(x)\| \\
\lesssim & \delta^{-\kappa_1} (1 + \|x\|). 
\end{align*}
\end{proof}

\begin{lemma}\label{lem:gradEf}
Fix $t\in(0,1)$ with $\alpha_t \neq 0$, $\beta_t > 0$. For any measurable $f:\Real^D\to\Real$ with $\E|f(X_1)| < \infty$, define $E_f(x,t) := \E\big(f(X_1)\mid X_t = x\big)$. Then, for every $x\in\Real^D$,
\[
\nabla_x E_f(x,t) = \frac{\alpha_t}{\beta_t^2}\,\operatorname{Cov}\big(f(X_1),X_1 \mid X_t = x\big).
\]
\end{lemma}

\begin{proof}
Write $N_f(x) := \varrho_t(x)\, E_f(x,t) = \int f(y)\, \phi(x;\alpha_t y, \beta_t^2 I_D)\, \pemp(\diff y)$. Differentiating under the integral sign,
\begin{align*}
\nabla_x N_f(x) = & \int f(y) \left(-\frac{x - \alpha_t y}{\beta_t^2}\right) \phi(x;\alpha_t y,\beta_t^2I_D)\,\pemp(\diff y) \\
= &  -\frac{x}{\beta_t^2} N_f(x) + \frac{\alpha_t}{\beta_t^2} \int f(y)\, y\, \phi(x;\alpha_ty,\beta_t^2I_D)\,\pemp(\diff y).
\end{align*}
Dividing by $\varrho_t(x)$ and applying the quotient rule together with $g(x) := \nabla_x\log\varrho_t(x) = -\frac{x - \alpha_t E_1(x,t)}{\beta_t^2}$ (Lemma~\ref{lemlip}),
\begin{align}
\nabla_x E_f(x,t) &= \frac{\nabla_x N_f(x)}{\varrho_t(x)} - E_f(x,t)\, g(x) \nonumber\\
&= -\frac{x}{\beta_t^2}E_f(x,t) + \frac{\alpha_t}{\beta_t^2}\E(f(X_1)X_1 \mid X_t = x) + \frac{x - \alpha_t E_1(x,t)}{\beta_t^2} E_f(x,t) \nonumber\\
&= \frac{\alpha_t}{\beta_t^2}\Big(\E(f(X_1)X_1 \mid X_t=x) - E_f(x,t)\,E_1(x,t)\Big) \nonumber\\
&= \frac{\alpha_t}{\beta_t^2}\operatorname{Cov}\big(f(X_1),X_1 \mid X_t = x\big). \label{eq:gradEf}
\end{align}
\end{proof}

\subsection{Results for the Wasserstein Distance}
\begin{lemma}\label{wcontrol}
    For any $d > d^\ast_p\left(\pd\right)$, with probability at least $1-\xi$,
    \[\fW_p(\pemp,\pd) \le n^{-1/d} + n^{-\frac{1}{2p}} D^{1/2} \log^{\frac{1}{2p}}(1/\xi).\]
\end{lemma}
\begin{proof}
Let $f_c (y) = \sup_{x \in [0,1]^D} \left(f(x) - \|x-y\|^p\right)$ denote the $c$-transform of $f$. It can be shown that $f(x) = \inf_{y \in [0,1]^D}\left(f_c(y) + \|x-y\|^p\right)$. From Kantorovich duality, it is well known that, for any $p \ge 1$,
\begin{align*}
    \fW_p^p(\mu, \nu) = \sup_{f \in \cC_b\left([0,1]^D\right)} \left( \int f \diff \mu - \int f_c \diff \nu\right),
\end{align*}
where $\cC_b$ denotes the set of all bounded continuous functions on $[0,1]^D$. By noting that $(f+a)_c = f_c + a$ for any constant $a$, we can restrict this function class by shifting $f$ such that $\inf_{x \in [0,1]^D} f(x) = 0$. Since $f$ satisfies $f(x) - f(y) \le \|x-y\|^p \le D^{p/2}$ under the optimal transform, this shift guarantees $0 \le f(x) \le D^{p/2}$. 
Consequently, for any $y \in [0,1]^D$, we have
\begin{align*}
    f_c(y) = \sup_{x \in [0,1]^D} \left(f(x) - \|x-y\|^p\right) \ge f(y) - \|y-y\|^p = f(y) \ge 0.
\end{align*}
Hence, we can take the supremum on the set of all bounded continuous functions from $[0,1]^D \to [0,D^{p/2}]$. Thus,
\begin{align*}
    \fW_p^p(\mu, \nu) = \sup_{f \in \cC_b\left([0,1]^D\right) \cap \{0 \le f \le D^{p/2}\}} \left( \int f \diff \mu - \int f_c \diff \nu\right).
\end{align*}
Let 
\begin{align*}
    W(x_1, \dots, x_n) = & \fW_p^p(\pd,\pemp)\\
    = & \sup_{f \in \cC_b\left([0,1]^D\right) \cap \{0 \le f \le D^{p/2}\}} \left( \frac{1}{n} \sum_{i=1}^n f(x_i) - \int f_c \diff \pd\right).
\end{align*}
Hence,
\begin{align*}
    & W(x_1, \dots, x_i, \dots, x_n) - W(x_1, \dots, x_i^\prime, \dots, x_n) \\
   \le & \sup_{f \in \cC_b\left([0,1]^D\right) \cap \{0 \le f \le D^{p/2}\}} \left( \frac{1}{n} \sum_{j=1}^n f(x_j) - \int f_c \diff \pd\right) \\
   &- \sup_{f \in \cC_b\left([0,1]^D\right) \cap \{0 \le f \le D^{p/2}\}} \left( \frac{1}{n} \sum_{j \neq i} f(x_j) + \frac{f(x_i^\prime)}{n} - \int f_c \diff \pd \right)\\
   \le  & \frac{1}{n} \sup_{0 \le f \le D^{p/2}} \left(f(x_i) - f(x_i^\prime)\right) \\
   \le & \frac{D^{p/2}}{n}.
\end{align*}
Applying the bounded difference inequality, we get that with probability at least $1-\xi$,
\begin{align*}
    \fW_p^p(\pemp,\pd) \le & \ \E \fW_p^p(\pemp,\pd) + \sqrt{\frac{D^p \log(1/\xi)}{2n}}\\
    \le & \ n^{-p/d} + \sqrt{\frac{D^p \log(1/\xi)}{2n}},
\end{align*}
where the last inequality follows from \citet[Theorem~10]{chakraborty2026diffusion}. The desired inequality now follows from taking the $p$-th root of both sides and applying the subadditivity inequality $(a+b)^{1/p} \le a^{1/p} + b^{1/p}$ (valid for $a,b \ge 0$ and $p \ge 1$).
\end{proof}
\begin{lemma}
    \label{w2}
    For any $M> 0$ and any estimator $\hv$ of $v_\star$,
    \begin{align*}
        \fW_p(\pd, \sT_M(\hP_{1-\delta}(\hv)))
    \lesssim & \fW_p(\pd, \pemp) + \delta^{\kappa_3} + M^{-\frac{1}{2p}} +  M \left(\operatorname{KL}\left(P_{1-\delta}, \hP_{1-\delta}(\hv)\right)\right)^{\frac{1}{2p}}.
    \end{align*}
\end{lemma}
\begin{proof}
We first observe the following:
\begingroup
\allowdisplaybreaks
    \begin{align*}
     &  \fW_p(\pd, \sT_M(\hP_{1-\delta}(\hv))) \\
     \le & \fW_p(\pd, \sT_M(P_{1-\delta})) + \fW_p(\sT_M(P_{1-\delta}), \sT_M(\hP_{1-\delta}(\hv)))\\
     \le & \fW_p(\pd, \sT_M(P_{1-\delta})) + 2 M \left(\operatorname{TV}\left(\sT_M(P_{1-\delta}), \sT_M(\hP_{1-\delta}(\hv))\right)\right)^{1/p}\\
     \le & \fW_p(\pd, \sT_M(P_{1-\delta})) + 2 M  \left(\operatorname{TV}\left(P_{1-\delta}, \hP_{1-\delta}(\hv)\right)\right)^{1/p}\\
     \lesssim & \fW_p(\pd, \sT_M(P_{1-\delta})) + M \left(\operatorname{KL}\left(P_{1-\delta}, \hP_{1-\delta}(\hv)\right)\right)^{\frac{1}{2p}} \\
     \le & \fW_p(\pd, \pemp) + \fW_p(\pemp, P_{1-\delta}) + \fW_p(P_{1-\delta}, \sT_M(P_{1-\delta})) + M \left(\operatorname{KL}\left(P_{1-\delta}, \hP_{1-\delta}(\hv)\right)\right)^{\frac{1}{2p}}.
   \end{align*}
\endgroup
We note that,
\begin{align*}
    \fW_p^p(\pemp, P_{1-\delta}) \le & \E_{X_1 \sim \pemp, X_0 \sim \cN(0,I_D) } \|X_1 - (\alpha_{1-\delta} X_1+ \beta_{1-\delta} X_0)\|^p \\
    \lesssim & (1-\alpha_{1-\delta})^p \E_{X_1 \sim \pemp } \|X_1\|^p  + \beta_{1-\delta}^p \E_{X_1 \sim \pemp, X_0 \sim \cN(0,I_D) } \| X_0\|^p\\
    \lesssim & \ \delta^{p \kappa_3}.
\end{align*}
Taking the $p$-th root, we get, $\fW_p(\pemp, P_{1-\delta}) \lesssim \delta^{\kappa_3}$.
Further,
\begin{align*}
    \fW_p^p(P_{1-\delta}, \sT_M(P_{1-\delta})) \le & \E\|X_{1-\delta}\|^p \one\{\|X_{1-\delta}\|_2 \ge M \}\\
    \le & \sqrt{\E\|X_{1-\delta}\|^{2p}} \sqrt{\prob \left(\|X_{1-\delta}\|_2 \ge M\right)}\\
    \lesssim & \sqrt{\frac{\E\|X_{1-\delta}\|_2}{M}}\\
    \lesssim & M^{-1/2},
\end{align*}
so that, taking $p$-th roots, $\fW_p(P_{1-\delta}, \sT_M(P_{1-\delta})) \lesssim M^{-1/(2p)}$.
Hence,
\begin{align*}
     \fW_p(\pd, \sT_M(\hP_{1-\delta}(\hv))) 
    \lesssim & \fW_p(\pd, \pemp) + \delta^{\kappa_3} + M^{-\frac{1}{2p}} +  M \left(\operatorname{KL}\left(P_{1-\delta}, \hP_{1-\delta}(\hv)\right)\right)^{\frac{1}{2p}} .
\end{align*}
\end{proof}
\section{Supporting Results from the Literature}
 \begin{lemma}\label{approx_master}
     For any $\varepsilon \in (0,1)$, there exists a ReLU network architecture
\(\cR\cN\!\left(W,L,B,K\right),\)
such that for any target $1$-Lipschitz function $f$ defined on $[0,1]^D$ with $f(0)=0$, the architecture yields an approximation $\hat f$ satisfying
\(\|f-\hat f\|_{\infty} \le \varepsilon.\)

Moreover, the Lipschitz continuity of $\hat f$ is bounded by
\[
\|\hat f(x)-\hat f(y)\| \le 10D\,\|x-y\|_{1}, \qquad \forall\, x,y\in[0,1]^D.
\]

The configuration of the network architecture is
\( K=\sqrt{D}, B=O(1), L = O\!\left(\log (1/\varepsilon) +D\right),\) and \(
 W = O\!\left(
\varepsilon^{-D} \left(
\log (1/\varepsilon)
+ D
\right)
\right).\)
 \end{lemma}
\begin{lemma}[Lemma 21 of \citet{JMLR:v21:20-002}]\label{lem_nakada} 
     Let $\cF = \cR\cN(W, L, B)$ be a space of ReLU networks with the number of weights, the number of layers, and the maximum absolute value of weights bounded by $W$, $L$, and $B$, respectively. Then,
\[
\log\cN(\epsilon; \cF, \ell_\infty) \leq W \log \left( 2LB^L(W + 1)^L \frac{1}{\epsilon} \right).
\]
 \end{lemma}


\begin{thebibliography}{71}
\providecommand{\natexlab}[1]{#1}
\providecommand{\url}[1]{\texttt{#1}}
\expandafter\ifx\csname urlstyle\endcsname\relax
  \providecommand{\doi}[1]{doi: #1}\else
  \providecommand{\doi}{doi: \begingroup \urlstyle{rm}\Url}\fi

\bibitem[Albergo et~al.(2025)Albergo, Boffi, and Vanden-Eijnden]{JMLR:v26:23-1605}
Michael Albergo, Nicholas~M. Boffi, and Eric Vanden-Eijnden.
\newblock Stochastic interpolants: A unifying framework for flows and diffusions.
\newblock \emph{Journal of Machine Learning Research}, 26\penalty0 (209):\penalty0 1--80, 2025.
\newblock URL \url{http://jmlr.org/papers/v26/23-1605.html}.

\bibitem[Albergo and Vanden-Eijnden(2023)]{albergo2023building}
Michael~Samuel Albergo and Eric Vanden-Eijnden.
\newblock Building normalizing flows with stochastic interpolants.
\newblock In \emph{The Eleventh International Conference on Learning Representations}, 2023.
\newblock URL \url{https://openreview.net/forum?id=li7qeBbCR1t}.

\bibitem[Benton et~al.(2024)Benton, Deligiannidis, and Doucet]{benton2024error}
Joe Benton, George Deligiannidis, and Arnaud Doucet.
\newblock Error bounds for flow matching methods.
\newblock \emph{Transactions on Machine Learning Research}, 2024.
\newblock ISSN 2835-8856.
\newblock URL \url{https://openreview.net/forum?id=uqQPyWFDhY}.

\bibitem[Bickel and Li(2007)]{bickel2007local}
Peter~J Bickel and Bo~Li.
\newblock Local polynomial regression on unknown manifolds.
\newblock \emph{Lecture Notes-Monograph Series}, pages 177--186, 2007.

\bibitem[Block et~al.(2020)Block, Mroueh, and Rakhlin]{block2020generative}
Adam Block, Youssef Mroueh, and Alexander Rakhlin.
\newblock Generative modeling with denoising auto-encoders and langevin sampling.
\newblock \emph{arXiv preprint arXiv:2002.00107}, 2020.

\bibitem[Bogachev et~al.(2015)Bogachev, Krylov, R{\"o}ckner, and Shaposhnikov]{bogachev2015fokker}
Vladimir~I. Bogachev, Nikolai~V. Krylov, Michael R{\"o}ckner, and Stanislav~V. Shaposhnikov.
\newblock \emph{Fokker--Planck--Kolmogorov Equations}, volume 207 of \emph{Mathematical Surveys and Monographs}.
\newblock American Mathematical Society, Providence, RI, 2015.

\bibitem[Bortoli(2022)]{debortoli2022convergence}
Valentin~De Bortoli.
\newblock Convergence of denoising diffusion models under the manifold hypothesis.
\newblock \emph{Transactions on Machine Learning Research}, 2022.
\newblock ISSN 2835-8856.
\newblock URL \url{https://openreview.net/forum?id=MhK5aXo3gB}.
\newblock Expert Certification.

\bibitem[Bose et~al.(2024)Bose, Akhound-Sadegh, Huguet, FATRAS, Rector-Brooks, Liu, Nica, Korablyov, Bronstein, and Tong]{bose2024sestochastic}
Joey Bose, Tara Akhound-Sadegh, Guillaume Huguet, Kilian FATRAS, Jarrid Rector-Brooks, Cheng-Hao Liu, Andrei~Cristian Nica, Maksym Korablyov, Michael~M. Bronstein, and Alexander Tong.
\newblock {SE}(3)-stochastic flow matching for protein backbone generation.
\newblock In \emph{The Twelfth International Conference on Learning Representations}, 2024.
\newblock URL \url{https://openreview.net/forum?id=kJFIH23hXb}.

\bibitem[Brock et~al.(2019)Brock, Donahue, and Simonyan]{brock2018large}
Andrew Brock, Jeff Donahue, and Karen Simonyan.
\newblock Large scale {GAN} training for high fidelity natural image synthesis.
\newblock In \emph{International Conference on Learning Representations}, 2019.
\newblock URL \url{https://openreview.net/forum?id=B1xsqj09Fm}.

\bibitem[Bunne et~al.(2022)Bunne, Krause, and marco cuturi]{bunne2022supervised}
Charlotte Bunne, Andreas Krause, and marco cuturi.
\newblock Supervised training of conditional monge maps.
\newblock In Alice~H. Oh, Alekh Agarwal, Danielle Belgrave, and Kyunghyun Cho, editors, \emph{Advances in Neural Information Processing Systems}, 2022.
\newblock URL \url{https://openreview.net/forum?id=sPNtVVUq7wi}.

\bibitem[Chakraborty(2025)]{chakraborty2025minimax}
Saptarshi Chakraborty.
\newblock Minimax lower bounds for estimating distributions on low-dimensional spaces.
\newblock \emph{Transactions on Machine Learning Research}, 2025.
\newblock ISSN 2835-8856.
\newblock URL \url{https://openreview.net/forum?id=wIgRV336hC}.

\bibitem[Chakraborty and Bartlett(2024)]{chakraborty2024a}
Saptarshi Chakraborty and Peter Bartlett.
\newblock A statistical analysis of wasserstein autoencoders for intrinsically low-dimensional data.
\newblock In \emph{The Twelfth International Conference on Learning Representations}, 2024.
\newblock URL \url{https://openreview.net/forum?id=WjRPZsfeBO}.

\bibitem[Chakraborty and Bartlett(2025)]{JMLR:v26:24-0054}
Saptarshi Chakraborty and Peter~L. Bartlett.
\newblock On the statistical properties of generative adversarial models for low intrinsic data dimension.
\newblock \emph{Journal of Machine Learning Research}, 26\penalty0 (111):\penalty0 1--57, 2025.
\newblock URL \url{http://jmlr.org/papers/v26/24-0054.html}.

\bibitem[Chakraborty et~al.(2026)Chakraborty, Berthet, and Bartlett]{chakraborty2026diffusion}
Saptarshi Chakraborty, Quentin Berthet, and Peter~L. Bartlett.
\newblock Generalization properties of score-matching diffusion models for intrinsically low-dimensional data, 2026.
\newblock URL \url{https://arxiv.org/abs/2603.03700}.

\bibitem[Chen et~al.(2026)Chen, Chewi, Daskalakis, and Rakhlin]{chen2026highaccuracy}
Fan Chen, Sinho Chewi, Constantinos~Costis Daskalakis, and Alexander Rakhlin.
\newblock High-accuracy sampling for diffusion models and log-concave distributions.
\newblock In \emph{Forty-third International Conference on Machine Learning}, 2026.
\newblock URL \url{https://openreview.net/forum?id=GW3umRqsZZ}.

\bibitem[Chen et~al.(2023{\natexlab{a}})Chen, Lee, and Lu]{pmlr-v202-chen23q}
Hongrui Chen, Holden Lee, and Jianfeng Lu.
\newblock Improved analysis of score-based generative modeling: User-friendly bounds under minimal smoothness assumptions.
\newblock In Andreas Krause, Emma Brunskill, Kyunghyun Cho, Barbara Engelhardt, Sivan Sabato, and Jonathan Scarlett, editors, \emph{Proceedings of the 40th International Conference on Machine Learning}, volume 202 of \emph{Proceedings of Machine Learning Research}, pages 4735--4763. PMLR, 23--29 Jul 2023{\natexlab{a}}.
\newblock URL \url{https://proceedings.mlr.press/v202/chen23q.html}.

\bibitem[Chen et~al.(2023{\natexlab{b}})Chen, Huang, Zhao, and Wang]{pmlr-v202-chen23o}
Minshuo Chen, Kaixuan Huang, Tuo Zhao, and Mengdi Wang.
\newblock Score approximation, estimation and distribution recovery of diffusion models on low-dimensional data.
\newblock In Andreas Krause, Emma Brunskill, Kyunghyun Cho, Barbara Engelhardt, Sivan Sabato, and Jonathan Scarlett, editors, \emph{Proceedings of the 40th International Conference on Machine Learning}, volume 202 of \emph{Proceedings of Machine Learning Research}, pages 4672--4712. PMLR, 23--29 Jul 2023{\natexlab{b}}.
\newblock URL \url{https://proceedings.mlr.press/v202/chen23o.html}.

\bibitem[Chen et~al.(2018)Chen, Rubanova, Bettencourt, and Duvenaud]{NEURIPS2018_69386f6b}
Ricky T.~Q. Chen, Yulia Rubanova, Jesse Bettencourt, and David~K Duvenaud.
\newblock Neural ordinary differential equations.
\newblock In S.~Bengio, H.~Wallach, H.~Larochelle, K.~Grauman, N.~Cesa-Bianchi, and R.~Garnett, editors, \emph{Advances in Neural Information Processing Systems}, volume~31. Curran Associates, Inc., 2018.
\newblock URL \url{https://proceedings.neurips.cc/paper_files/paper/2018/file/69386f6bb1dfed68692a24c8686939b9-Paper.pdf}.

\bibitem[Chen et~al.(2023{\natexlab{c}})Chen, Chewi, Lee, Li, Lu, and Salim]{chen2023the}
Sitan Chen, Sinho Chewi, Holden Lee, Yuanzhi Li, Jianfeng Lu, and Adil Salim.
\newblock The probability flow {ODE} is provably fast.
\newblock In \emph{Thirty-seventh Conference on Neural Information Processing Systems}, 2023{\natexlab{c}}.
\newblock URL \url{https://openreview.net/forum?id=KD6MFeWSAd}.

\bibitem[Chen et~al.(2023{\natexlab{d}})Chen, Daras, and Dimakis]{pmlr-v202-chen23e}
Sitan Chen, Giannis Daras, and Alex Dimakis.
\newblock Restoration-degradation beyond linear diffusions: A non-asymptotic analysis for {DDIM}-type samplers.
\newblock In Andreas Krause, Emma Brunskill, Kyunghyun Cho, Barbara Engelhardt, Sivan Sabato, and Jonathan Scarlett, editors, \emph{Proceedings of the 40th International Conference on Machine Learning}, volume 202 of \emph{Proceedings of Machine Learning Research}, pages 4462--4484. PMLR, 23--29 Jul 2023{\natexlab{d}}.
\newblock URL \url{https://proceedings.mlr.press/v202/chen23e.html}.

\bibitem[Coddington and Levinson(1955)]{coddington1955theory}
Earl~A. Coddington and Norman Levinson.
\newblock \emph{Theory of Ordinary Differential Equations}.
\newblock McGraw-Hill, New York, 1955.
\newblock ISBN 978-0070115422.

\bibitem[Corso et~al.(2023)Corso, St{\"a}rk, Jing, Barzilay, and Jaakkola]{corso2023diffdock}
Gabriele Corso, Hannes St{\"a}rk, Bowen Jing, Regina Barzilay, and Tommi~S. Jaakkola.
\newblock Diffdock: Diffusion steps, twists, and turns for molecular docking.
\newblock In \emph{The Eleventh International Conference on Learning Representations}, 2023.
\newblock URL \url{https://openreview.net/forum?id=kKF8_K-mBbS}.

\bibitem[Cybenko(1989)]{cybenko1989approximation}
George Cybenko.
\newblock Approximation by superpositions of a sigmoidal function.
\newblock \emph{Mathematics of control, signals and systems}, 2\penalty0 (4):\penalty0 303--314, 1989.

\bibitem[Dahal et~al.(2022)Dahal, Havrilla, Chen, Zhao, and Liao]{dahal2022deep}
Biraj Dahal, Alexander Havrilla, Minshuo Chen, Tuo Zhao, and Wenjing Liao.
\newblock On deep generative models for approximation and estimation of distributions on manifolds.
\newblock \emph{Advances in Neural Information Processing Systems}, 35:\penalty0 10615--10628, 2022.

\bibitem[Deng et~al.(2009)Deng, Dong, Socher, Li, Li, and Fei-Fei]{deng_imagenet}
Jia Deng, Wei Dong, Richard Socher, Li-Jia Li, Kai Li, and Li~Fei-Fei.
\newblock Imagenet: A large-scale hierarchical image database.
\newblock In \emph{2009 IEEE Conference on Computer Vision and Pattern Recognition}, pages 248--255, 2009.
\newblock \doi{10.1109/CVPR.2009.5206848}.

\bibitem[Dudley(1969)]{dudley1969speed}
Richard~Mansfield Dudley.
\newblock The speed of mean glivenko-cantelli convergence.
\newblock \emph{The Annals of Mathematical Statistics}, 40\penalty0 (1):\penalty0 40--50, 1969.

\bibitem[Esser et~al.(2024)Esser, Kulal, Blattmann, Entezari, M{\"u}ller, Saini, Levi, Lorenz, Sauer, Boesel, Podell, Dockhorn, English, and Rombach]{esser2024scaling}
Patrick Esser, Sumith Kulal, Andreas Blattmann, Rahim Entezari, Jonas M{\"u}ller, Harry Saini, Yam Levi, Dominik Lorenz, Axel Sauer, Frederic Boesel, Dustin Podell, Tim Dockhorn, Zion English, and Robin Rombach.
\newblock Scaling rectified flow transformers for high-resolution image synthesis.
\newblock In \emph{Forty-first International Conference on Machine Learning}, 2024.
\newblock URL \url{https://openreview.net/forum?id=FPnUhsQJ5B}.

\bibitem[Fukumizu et~al.(2025)Fukumizu, Suzuki, Isobe, Oko, and Koyama]{fukumizu2025flow}
Kenji Fukumizu, Taiji Suzuki, Noboru Isobe, Kazusato Oko, and Masanori Koyama.
\newblock Flow matching achieves almost minimax optimal convergence.
\newblock In \emph{The Thirteenth International Conference on Learning Representations}, 2025.
\newblock URL \url{https://openreview.net/forum?id=2OMyAFjiJJ}.

\bibitem[Gao et~al.(2024)Gao, Huang, Jiao, and Zheng]{gao2024convergence}
Yuan Gao, Jian Huang, Yuling Jiao, and Shurong Zheng.
\newblock Convergence of continuous normalizing flows for learning probability distributions.
\newblock \emph{arXiv preprint arXiv:2404.00551}, 2024.

\bibitem[Geffner et~al.(2025)Geffner, Didi, Zhang, Reidenbach, Cao, Yim, Geiger, Dallago, Kucukbenli, Vahdat, and Kreis]{geffner2025proteina}
Tomas Geffner, Kieran Didi, Zuobai Zhang, Danny Reidenbach, Zhonglin Cao, Jason Yim, Mario Geiger, Christian Dallago, Emine Kucukbenli, Arash Vahdat, and Karsten Kreis.
\newblock Proteina: Scaling flow-based protein structure generative models.
\newblock In \emph{The Thirteenth International Conference on Learning Representations}, 2025.
\newblock URL \url{https://openreview.net/forum?id=TVQLu34bdw}.

\bibitem[Goodfellow et~al.(2020)Goodfellow, Pouget-Abadie, Mirza, Xu, Warde-Farley, Ozair, Courville, and Bengio]{goodfellow2020generative}
Ian Goodfellow, Jean Pouget-Abadie, Mehdi Mirza, Bing Xu, David Warde-Farley, Sherjil Ozair, Aaron Courville, and Yoshua Bengio.
\newblock Generative adversarial networks.
\newblock \emph{Communications of the ACM}, 63\penalty0 (11):\penalty0 139--144, 2020.

\bibitem[Heusel et~al.(2017)Heusel, Ramsauer, Unterthiner, Nessler, and Hochreiter]{heusel2017gans}
Martin Heusel, Hubert Ramsauer, Thomas Unterthiner, Bernhard Nessler, and Sepp Hochreiter.
\newblock Gans trained by a two time-scale update rule converge to a local nash equilibrium.
\newblock \emph{Advances in Neural Information Processing Systems}, 30, 2017.

\bibitem[Ho et~al.(2020)Ho, Jain, and Abbeel]{NEURIPS2020_4c5bcfec}
Jonathan Ho, Ajay Jain, and Pieter Abbeel.
\newblock Denoising diffusion probabilistic models.
\newblock In H.~Larochelle, M.~Ranzato, R.~Hadsell, M.F. Balcan, and H.~Lin, editors, \emph{Advances in Neural Information Processing Systems}, volume~33, pages 6840--6851. Curran Associates, Inc., 2020.
\newblock URL \url{https://proceedings.neurips.cc/paper_files/paper/2020/file/4c5bcfec8584af0d967f1ab10179ca4b-Paper.pdf}.

\bibitem[Hornik(1991)]{hornik1991approximation}
Kurt Hornik.
\newblock Approximation capabilities of multilayer feedforward networks.
\newblock \emph{Neural networks}, 4\penalty0 (2):\penalty0 251--257, 1991.

\bibitem[Huang et~al.(2022)Huang, Jiao, Li, Liu, Wang, and Yang]{JMLR:v23:21-0732}
Jian Huang, Yuling Jiao, Zhen Li, Shiao Liu, Yang Wang, and Yunfei Yang.
\newblock An error analysis of generative adversarial networks for learning distributions.
\newblock \emph{Journal of Machine Learning Research}, 23\penalty0 (116):\penalty0 1--43, 2022.
\newblock URL \url{http://jmlr.org/papers/v23/21-0732.html}.

\bibitem[Karatzas and Shreve(1991)]{karatzas1991brownian}
Ioannis Karatzas and Steven~E. Shreve.
\newblock \emph{Brownian Motion and Stochastic Calculus}.
\newblock Springer, Berlin, 2nd edition, 1991.

\bibitem[Kim et~al.(2019)Kim, Shin, Rinaldo, and Wasserman]{kim2019uniform}
Jisu Kim, Jaehyeok Shin, Alessandro Rinaldo, and Larry Wasserman.
\newblock Uniform convergence rate of the kernel density estimator adaptive to intrinsic volume dimension.
\newblock In \emph{International Conference on Machine Learning}, pages 3398--3407. PMLR, 2019.

\bibitem[Kingma and Welling(2014)]{kingma2014autoencoding}
Diederik~P Kingma and Max Welling.
\newblock Auto-encoding variational bayes.
\newblock In \emph{International Conference on Learning Representations (ICLR)}, 2014.

\bibitem[Kolmogorov and Tikhomirov(1961)]{kolmogorov1961}
Andrey~N Kolmogorov and Vladimir~Mikha{\i}lovich Tikhomirov.
\newblock $\epsilon$-entropy and $\epsilon$-capacity of sets in function spaces.
\newblock \emph{Translations of the American Mathematical Society}, 17:\penalty0 277--364, 1961.

\bibitem[Korotin et~al.(2023)Korotin, Selikhanovych, and Burnaev]{korotin2023neural}
Alexander Korotin, Daniil Selikhanovych, and Evgeny Burnaev.
\newblock Neural optimal transport.
\newblock In \emph{The Eleventh International Conference on Learning Representations}, 2023.
\newblock URL \url{https://openreview.net/forum?id=d8CBRlWNkqH}.

\bibitem[Kumar et~al.(2026)Kumar, Wang, and Lin]{kumar2026flowmatchingadaptivemanifold}
Shivam Kumar, Yixin Wang, and Lizhen Lin.
\newblock Flow matching is adaptive to manifold structures, 2026.
\newblock URL \url{https://arxiv.org/abs/2602.22486}.

\bibitem[Lee et~al.(2022)Lee, Lu, and Tan]{lee2022convergence}
Holden Lee, Jianfeng Lu, and Yixin Tan.
\newblock Convergence for score-based generative modeling with polynomial complexity.
\newblock In Alice~H. Oh, Alekh Agarwal, Danielle Belgrave, and Kyunghyun Cho, editors, \emph{Advances in Neural Information Processing Systems}, 2022.
\newblock URL \url{https://openreview.net/forum?id=dUSI4vFyMK}.

\bibitem[Lee et~al.(2023)Lee, Lu, and Tan]{pmlr-v201-lee23a}
Holden Lee, Jianfeng Lu, and Yixin Tan.
\newblock Convergence of score-based generative modeling for general data distributions.
\newblock In Shipra Agrawal and Francesco Orabona, editors, \emph{Proceedings of The 34th International Conference on Algorithmic Learning Theory}, volume 201 of \emph{Proceedings of Machine Learning Research}, pages 946--985. PMLR, 20 Feb--23 Feb 2023.
\newblock URL \url{https://proceedings.mlr.press/v201/lee23a.html}.

\bibitem[Lipman et~al.(2023)Lipman, Chen, Ben-Hamu, Nickel, and Le]{lipman2023flow}
Yaron Lipman, Ricky~TQ Chen, Heli Ben-Hamu, Maximilian Nickel, and Matt Le.
\newblock Flow matching for generative modeling.
\newblock In \emph{11th International Conference on Learning Representations, ICLR 2023}, 2023.

\bibitem[Liu et~al.(2023)Liu, Gong, and qiang liu]{liu2023flow}
Xingchao Liu, Chengyue Gong, and qiang liu.
\newblock Flow straight and fast: Learning to generate and transfer data with rectified flow.
\newblock In \emph{The Eleventh International Conference on Learning Representations}, 2023.
\newblock URL \url{https://openreview.net/forum?id=XVjTT1nw5z}.

\bibitem[Lu et~al.(2022)Lu, Zheng, Bao, Chen, Li, and Zhu]{pmlr-v162-lu22f}
Cheng Lu, Kaiwen Zheng, Fan Bao, Jianfei Chen, Chongxuan Li, and Jun Zhu.
\newblock Maximum likelihood training for score-based diffusion {ODE}s by high order denoising score matching.
\newblock In Kamalika Chaudhuri, Stefanie Jegelka, Le~Song, Csaba Szepesvari, Gang Niu, and Sivan Sabato, editors, \emph{Proceedings of the 39th International Conference on Machine Learning}, volume 162 of \emph{Proceedings of Machine Learning Research}, pages 14429--14460. PMLR, 17--23 Jul 2022.
\newblock URL \url{https://proceedings.mlr.press/v162/lu22f.html}.

\bibitem[Ma et~al.(2024)Ma, Goldstein, Albergo, Boffi, Vanden-Eijnden, and Xie]{ma2024sit}
Nanye Ma, Mark Goldstein, Michael~S Albergo, Nicholas~M Boffi, Eric Vanden-Eijnden, and Saining Xie.
\newblock Sit: Exploring flow and diffusion-based generative models with scalable interpolant transformers.
\newblock In \emph{European Conference on Computer Vision}, pages 23--40. Springer, 2024.

\bibitem[Nakada and Imaizumi(2020)]{JMLR:v21:20-002}
Ryumei Nakada and Masaaki Imaizumi.
\newblock Adaptive approximation and generalization of deep neural network with intrinsic dimensionality.
\newblock \emph{Journal of Machine Learning Research}, 21\penalty0 (174):\penalty0 1--38, 2020.
\newblock URL \url{http://jmlr.org/papers/v21/20-002.html}.

\bibitem[Oko et~al.(2023)Oko, Akiyama, and Suzuki]{oko2023diffusion}
Kazusato Oko, Shunta Akiyama, and Taiji Suzuki.
\newblock Diffusion models are minimax optimal distribution estimators.
\newblock In \emph{International Conference on Machine Learning}, pages 26517--26582. PMLR, 2023.

\bibitem[{\O}ksendal(2003)]{oksendal2003stochastic}
Bernt {\O}ksendal.
\newblock Stochastic differential equations.
\newblock In \emph{Stochastic differential equations: an introduction with applications}, pages 38--50. Springer, 2003.

\bibitem[Petersen and Voigtlaender(2018)]{petersen2018optimal}
Philipp Petersen and Felix Voigtlaender.
\newblock Optimal approximation of piecewise smooth functions using deep relu neural networks.
\newblock \emph{Neural Networks}, 108:\penalty0 296--330, 2018.

\bibitem[Pi et~al.(2026)Pi, Lu, Su, Wu, Hu, and Liu]{pi2026learning}
Sophia Pi, Mingcheng Lu, Maojiang Su, Weimin Wu, Jerry Yao-Chieh Hu, and Han Liu.
\newblock Learning manifold data with flow matching.
\newblock In \emph{Forty-third International Conference on Machine Learning}, 2026.
\newblock URL \url{https://openreview.net/forum?id=9mA5d705Gd}.

\bibitem[Pope et~al.(2020)Pope, Zhu, Abdelkader, Goldblum, and Goldstein]{pope2020intrinsic}
Phil Pope, Chen Zhu, Ahmed Abdelkader, Micah Goldblum, and Tom Goldstein.
\newblock The intrinsic dimension of images and its impact on learning.
\newblock In \emph{International Conference on Learning Representations}, 2020.

\bibitem[Posner et~al.(1967)Posner, Rodemich, and Rumsey~Jr]{posner1967epsilon}
Edward~C Posner, Eugene~R Rodemich, and Howard Rumsey~Jr.
\newblock Epsilon entropy of stochastic processes.
\newblock \emph{The Annals of Mathematical Statistics}, pages 1000--1020, 1967.

\bibitem[Rombach et~al.(2022)Rombach, Blattmann, Lorenz, Esser, and Ommer]{9878449}
Robin Rombach, Andreas Blattmann, Dominik Lorenz, Patrick Esser, and Bjorn Ommer.
\newblock { High-Resolution Image Synthesis with Latent Diffusion Models }.
\newblock In \emph{2022 IEEE/CVF Conference on Computer Vision and Pattern Recognition (CVPR)}, pages 10674--10685, Los Alamitos, CA, USA, June 2022. IEEE Computer Society.
\newblock \doi{10.1109/CVPR52688.2022.01042}.
\newblock URL \url{https://doi.ieeecomputersociety.org/10.1109/CVPR52688.2022.01042}.

\bibitem[Saharia et~al.(2022)Saharia, Chan, Saxena, Li, Whang, Denton, Ghasemipour, Gontijo-Lopes, Ayan, Salimans, Ho, Fleet, and Norouzi]{saharia2022photorealistic}
Chitwan Saharia, William Chan, Saurabh Saxena, Lala Li, Jay Whang, Emily Denton, Seyed Kamyar~Seyed Ghasemipour, Raphael Gontijo-Lopes, Burcu~Karagol Ayan, Tim Salimans, Jonathan Ho, David~J. Fleet, and Mohammad Norouzi.
\newblock Photorealistic text-to-image diffusion models with deep language understanding.
\newblock In Alice~H. Oh, Alekh Agarwal, Danielle Belgrave, and Kyunghyun Cho, editors, \emph{Advances in Neural Information Processing Systems}, 2022.
\newblock URL \url{https://openreview.net/forum?id=08Yk-n5l2Al}.

\bibitem[Santambrogio(2015)]{Santambrogio2015}
Filippo Santambrogio.
\newblock \emph{Optimal Transport for Applied Mathematicians: Calculus of Variations, PDEs, and Modeling}, volume~87 of \emph{Progress in Nonlinear Differential Equations and Their Applications}.
\newblock Springer International Publishing, 2015.
\newblock ISBN 978-3-319-20827-5.
\newblock \doi{10.1007/978-3-319-20828-2}.
\newblock URL \url{https://doi.org/10.1007/978-3-319-20828-2}.

\bibitem[Schmidt-Hieber(2020)]{schmidt2020nonparametric}
Johannes Schmidt-Hieber.
\newblock {Nonparametric regression using deep neural networks with ReLU activation function}.
\newblock \emph{The Annals of Statistics}, 48\penalty0 (4):\penalty0 1875 -- 1897, 2020.
\newblock \doi{10.1214/19-AOS1875}.
\newblock URL \url{https://doi.org/10.1214/19-AOS1875}.

\bibitem[Shen et~al.(2022)Shen, Yang, and Zhang]{shen2022optimal}
Zuowei Shen, Haizhao Yang, and Shijun Zhang.
\newblock Optimal approximation rate of relu networks in terms of width and depth.
\newblock \emph{Journal de Math{\'e}matiques Pures et Appliqu{\'e}es}, 157:\penalty0 101--135, 2022.

\bibitem[Singh and Fischer(2024)]{singh2024stochastic}
Saurabh Singh and Ian Fischer.
\newblock Stochastic sampling from deterministic flow models.
\newblock \emph{arXiv preprint arXiv:2410.02217}, 2024.

\bibitem[Sohl-Dickstein et~al.(2015)Sohl-Dickstein, Weiss, Maheswaranathan, and Ganguli]{sohl2015deep}
Jascha Sohl-Dickstein, Eric Weiss, Niru Maheswaranathan, and Surya Ganguli.
\newblock Deep unsupervised learning using nonequilibrium thermodynamics.
\newblock In \emph{International conference on machine learning}, pages 2256--2265. pmlr, 2015.

\bibitem[Song and Ermon(2019)]{song2019generative}
Yang Song and Stefano Ermon.
\newblock Generative modeling by estimating gradients of the data distribution.
\newblock \emph{Advances in Neural Information Processing Systems}, 32, 2019.

\bibitem[Song et~al.(2021)Song, Sohl-Dickstein, Kingma, Kumar, Ermon, and Poole]{song2021scorebased}
Yang Song, Jascha Sohl-Dickstein, Diederik~P Kingma, Abhishek Kumar, Stefano Ermon, and Ben Poole.
\newblock Score-based generative modeling through stochastic differential equations.
\newblock In \emph{International Conference on Learning Representations}, 2021.
\newblock URL \url{https://openreview.net/forum?id=PxTIG12RRHS}.

\bibitem[Tang and Yang(2024)]{pmlr-v238-tang24a}
Rong Tang and Yun Yang.
\newblock Adaptivity of diffusion models to manifold structures.
\newblock In Sanjoy Dasgupta, Stephan Mandt, and Yingzhen Li, editors, \emph{Proceedings of The 27th International Conference on Artificial Intelligence and Statistics}, volume 238 of \emph{Proceedings of Machine Learning Research}, pages 1648--1656. PMLR, 02--04 May 2024.
\newblock URL \url{https://proceedings.mlr.press/v238/tang24a.html}.

\bibitem[Tao(2011)]{tao2011measure}
Terence Tao.
\newblock \emph{An Introduction to Measure Theory}, volume 126 of \emph{Graduate Studies in Mathematics}.
\newblock American Mathematical Society, Providence, RI, 2011.
\newblock ISBN 978-0-8218-6919-2.
\newblock URL \url{https://bookstore.ams.org/gsm-126}.

\bibitem[Vershynin(2026)]{Vershynin_2026}
Roman Vershynin.
\newblock \emph{High-Dimensional Probability: An Introduction with Applications in Data Science}.
\newblock Cambridge Series in Statistical and Probabilistic Mathematics. Cambridge University Press, 2 edition, 2026.

\bibitem[Watson et~al.(2023)Watson, Juergens, Bennett, Trippe, Yim, Eisenach, Ahern, Borst, Ragotte, Milles, et~al.]{watson2023novo}
Joseph~L Watson, David Juergens, Nathaniel~R Bennett, Brian~L Trippe, Jason Yim, Helen~E Eisenach, Woody Ahern, Andrew~J Borst, Robert~J Ragotte, Lukas~F Milles, et~al.
\newblock De novo design of protein structure and function with rfdiffusion.
\newblock \emph{Nature}, 620\penalty0 (7976):\penalty0 1089--1100, 2023.

\bibitem[Weed and Bach(2019)]{weed2019sharp}
Jonathan Weed and Francis Bach.
\newblock Sharp asymptotic and finite-sample rates of convergence of empirical measures in wasserstein distance.
\newblock \emph{Bernoulli}, 25\penalty0 (4A):\penalty0 2620 -- 2648, 2019.
\newblock \doi{10.3150/18-BEJ1065}.
\newblock URL \url{https://doi.org/10.3150/18-BEJ1065}.

\bibitem[Wu and Cai(2026)]{wu2026diffusion}
Jingda Wu and Changxiao Cai.
\newblock Diffusion models are statistically optimal for learning low-dimensional multi-modal distributions.
\newblock In \emph{Forty-third International Conference on Machine Learning}, 2026.
\newblock URL \url{https://openreview.net/forum?id=L5JTAPUdbQ}.

\bibitem[Yarotsky(2017)]{yarotsky2017error}
Dmitry Yarotsky.
\newblock Error bounds for approximations with deep relu networks.
\newblock \emph{Neural Networks}, 94:\penalty0 103--114, 2017.

\bibitem[Zhou and Liu(2025)]{zhou2025an}
Zhengyu Zhou and Weiwei Liu.
\newblock An error analysis of flow matching for deep generative modeling.
\newblock In \emph{Forty-second International Conference on Machine Learning}, 2025.
\newblock URL \url{https://openreview.net/forum?id=vES22INUKm}.

\end{thebibliography}
\end{document}